\documentclass{article}

\usepackage[preprint]{neurips_2026}
\usepackage{math_headers}
\usepackage{subcaption}
\usepackage{placeins}

\title{Interpreting Reinforcement Learning Agents with Susceptibilities}

\newcommand\equalContribution{=}

\author{%
    Chris Elliott\textsuperscript{\equalContribution}\\
    Timaeus
\And
    Einar Urdshals\textsuperscript{\equalContribution}\\
    Timaeus
\And
    David Quarel\\
    Timaeus
\And
    Daniel Murfet\\
    Timaeus
}

\date{\today}

\begin{document}

\maketitle
\makeatletter
\def\@makefnmark{\hbox{\@textsuperscript{\normalfont\@thefnmark}}}
\makeatother
\renewcommand{\thefootnote}{=}
\setcounter{footnote}{0}
\footnotetext{These authors contributed equally to this work.}
\setcounter{footnote}{0}
\renewcommand{\thefootnote}{\arabic{footnote}}

\begin{abstract}
Susceptibilities are a technique for neural network interpretability that studies the response of posterior expectation values of observables to perturbations of the loss. We generalize this construction to the setting of the regret in deep reinforcement learning and investigate the utility of susceptibilities in a simple gridworld model that nevertheless exhibits non-trivial stagewise development. We argue that susceptibilities reveal internal features of the development of the model in parameter space that one cannot detect purely by studying the development of the learned policy. We validate these results with activation-steering, and discuss the framework's extension to RLHF post-training.
%We generalize susceptibilities, a linear-response interpretability technique, from supervised learning to deep reinforcement learning. Applied to RL agents trained on the cheese-in-the-corner gridworld, susceptibilities reveal directional asymmetries in the local regret geometry that are not apparent from the learned policy, distinct algorithmic regimes at successive training stages, spontaneous symmetry breaking under uniform training determined by initialization, and continued parameter-space simplification while the policy is fixed. We validate these results with activation-steering, and discuss the framework's extension to RLHF post-training.
%Susceptibilities are a technique for neural network interpretability that studies the response of posterior expectation values of observables to perturbations of the loss.  We generalize this construction to the setting of the regret in deep reinforcement learning and investigate the utility of susceptibilities in a simple gridworld model that nevertheless exhibits non-trivial stagewise development.  We argue that susceptibilities reveal internal features of the development of the model in parameter space that one cannot detect purely by studying the development of the learned policy.
\end{abstract}

\section{Introduction}

Deep reinforcement learning (RL) agents based on large language models are now widely deployed, particularly for tasks involving coding \citep{massenkoffmccrory2026labor}. Understanding and controlling the behavior of these agents is therefore an important scientific and sociotechnical problem. Unfortunately, our understanding is presently limited: for instance, most efforts in neural network interpretability have been in the setting of supervised learning, with some notable exceptions \citep{McGrath2022Chess, mini2023understanding}. We apply the recently developed interpretability framework of susceptibilities \citep{Susceptibilities, BIF} to study a simple cheese-in-the-corner gridworld deep RL agent, and both validate the applicability of these techniques by comparing to known ground truth and uncover some previously unnoticed aspects of the development of the policy over training.

%Susceptibilities \cite{Susceptibilities, BIF} are a recent interpretability technique drawn from the linear-response framework of statistical mechanics, conceptually adjacent to influence-function methods for data attribution \cite{KohLiang2017}: they measure the response of the posterior expectation value of a chosen observable, localized on a component of the network, to a perturbation of the data distribution or loss.  In parallel work, \cite{RL1} extended singular learning theory to deep reinforcement learning: the generalized posterior on policies obeys an analogue of Watanabe's free energy formula with the regret playing the role of the loss, and the local learning coefficient of the regret governs the stagewise development of policies through Bayesian phase transitions.  The present work connects these two threads: we generalize susceptibilities to the reinforcement learning setting introduced in \cite{RL1} and apply them to the cheese-in-the-corner environment studied there, investigating what susceptibilities reveal about the internal structure of an RL agent beyond what is visible from the development of the learned policy alone, complementing structural-interpretability efforts on deep RL agents \cite{McGrath2022Chess, mini2023understanding} and progress-measure approaches to phase transitions in deep learning more broadly \cite{Nanda2023Grokking}.

Some of the results are summarized in \cref{fig:alpha_0.6_single_run}. To explain, recall that in deep RL the central geometric object is the \emph{regret landscape}, which is the analogue of the loss landscape in supervised learning. The key idea driving the approach to interpretability taken in this paper, following \citep{hoogland2024developmental,murfet2025programs}, is that \emph{internal structure of the learned agent is encoded in the local geometry of the regret landscape}; this was explored in the RL setting by \citet{RL1}.

We formalize this by working with the generalized posterior distribution
\begin{equation}
p(w) \propto \exp(-n\beta G(w)) \phi(w)\,, \qquad \mathbb{E}_{w \sim p(w)}[ \OO ] := \int \OO(w) p(w) dw
\end{equation}
where $G(w)$ is the regret associated to a policy $\pi_w$, where $w \in W$ is the parameter, $\phi(w) > 0$ is the prior, $n$ plays the role of the number of observed trajectories and $\beta > 0$ is the inverse temperature. To probe the internal structure of an agent with policy $\pi_{w^*}$, we study how variations in the initial state distribution $\Lambda \in \Delta(\mc S)$ affect the local geometry of the regret $G$ near $w^*$: the idea being that if some parts of the model are particularly important for acting in trajectories beginning at $s_0 \in \mc S$ then varying the probability of $s_0$ should vary the geometry in those directions in the regret landscape more.

More precisely, a perturbation of the initial state distribution from $\Lambda$ to $\Lambda + h \xi$ deforms the regret function from $G$ to a perturbed regret $G^h$, and through the generalized free-energy formalism of \citet{RL1} this in turn deforms the generalized Bayesian posterior on policies from $p(w)$ to $p(w|h)$.  We measure this deformation by tracking how it shifts the posterior expectation of an observable function on parameter space: given an observable $\mc \OO \colon W \to \RR$, the \emph{susceptibility} associated to $\mc O$ and the perturbation direction $\xi$ is the leading-order response of $\mathbb{E}_{w \sim p(w|h)}[ \OO ]$ to $\xi$.

When the perturbation $\xi_{s_0}$ upweights a particular initial state $s_0$, and the observable $\OO = \OO_C$ is a per-component loss (say for a convolutional layer, see \cref{weight_restricted_observable_appendix}) we obtain a susceptibility that measures a coupling of the initial state $s_0$ and the component $C$
\begin{equation}
\chi_C(\xi_{s_0}) := \frac{1}{n\beta}\frac{\partial}{\partial h} \mathbb{E}_{w \sim p(w|h)}[ \OO_C ]\Big|_{h = 0}\,.
\end{equation}
From this and a list of components $C_1,\ldots,C_H$ we construct the \emph{susceptibility vector}
\begin{equation}
\chi(\xi_{s_0}) := \big( \chi_{C_1}(\xi_{s_0}), \ldots, \chi_{C_H}(\xi_{s_0}) \big) \in \mathbb{R}^H\,.
\end{equation}
For the cheese-in-the-corner agent shown in \cref{fig:alpha_0.6_single_run} the network has $H = 6$ components consisting of three convolutional blocks (Conv 1--3), two fully-connected layers (FF 1, FF 2) and a policy head (\cref{model_appendix}). We project $\chi(\xi_{s_0})$ onto $\mathbb{R}^2$ by averaging the convolution and fully-connected layers separately. The scatter-plots shown are of the resulting projections as $s_0$ ranges over all initial states (or rather, of empirical estimators of these quantities) providing a picture of all these variations in the initial state distribution and how they affect the regret. We see that the susceptibilities respond in interesting ways during the stagewise development.

\begin{figure}[ht]
    \centering
    \includegraphics[width=0.99\linewidth]{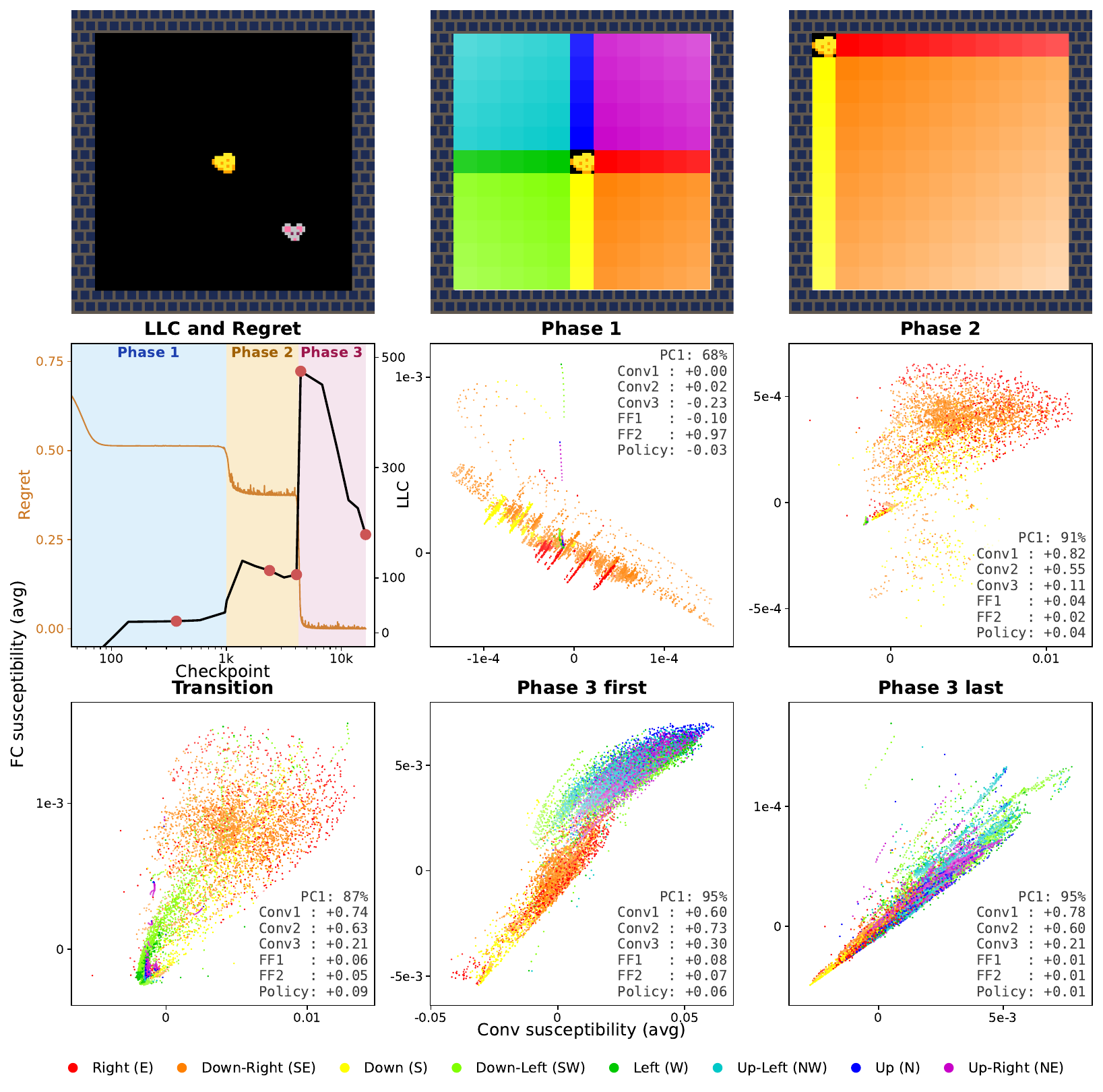}
    \caption{\textbf{Training dynamics for a model trained with $\alpha=0.6$}. \textbf{Top row:} Cheese-in-corner environment; an RL agent (mouse) navigates to cheese (+1 reward). Initial states are colored by mouse position relative to the cheese (two rightmost panels); in a fraction $1-\alpha$ of environments cheese is in the top-left corner (top right panel), so all initial mouse positions are red/orange/yellow. \textbf{Middle and lower row:} Middle left shows the LLC and regret curves; the model passes through 3 phases, visible as LLC spikes and regret drops. Red dots mark the checkpoints whose susceptibility scatter plots are shown, projecting susceptibilities to 2D (average convolution-block vs.\ average fully-connected-block susceptibility). Red/orange/yellow states dominate phases 1--2, with other directions concentrated near the origin. Lower left: transition to phase 3, where other directions (notably green) gain non-zero susceptibilities, before two clusters emerge early in phase 3 (lower middle). Late in phase 3, the clusters merge and the overall susceptibility scale declines. Each susceptibility panel reports the variance explained by PC1 and PC1's cosine similarity to each weight restriction direction.}
    %\caption{A scatter plot of susceptibilities for a model trained with $\alpha = 0.6$, projected to two dimensions (average convolution block susceptibility against average fully connected block susceptibility).  The top left and middle panels represents checkpoints in the middle of phases 1 and 2 respectively.  These phases are dominated by susceptibilities for states receiving reward when the model is moving towards the top left corner (red, orange and yellow points).  The top right panel represents a checkpoint during the phase transition, as the model learns to respond non-trivially to some new states (green).  The lower left panel represents a checkpoint early in phase 3, shortly after the phase transition: the model now behaves approximately optimally in all states and all states are now represented non-trivially in the susceptibility scatter plot. We note that the states that were represented non-trivially in phase two now form a cluster partly separated from the other states. The lower middle panel shows the susceptibilities at the end of training, where we see that the clustering has gotten less pronounced, and the overall scale of the susceptibilities has shrunk significantly. Finally, at the bottom right we see the LLC and Regret curves for this model, with the red dots indicating the checkpoints at which the susceptibilities have been estimated. The variation explained by PC1 and the cosine similarity between the PC1 and each of the weight restriction directions is indicated in each susceptibility panel.}
    \label{fig:alpha_0.6_single_run}
\end{figure}

We offer the following contributions.
\begin{enumerate}
    \item In \cref{susceptibilities_background_section}, we present \textbf{a generalization of susceptibilities to the reinforcement learning setting}.  We define susceptibilities for three natural classes of perturbation of a Markov decision problem (initial state distribution, reward function, and transition dynamics) and provide a fluctuation-dissipation formula allowing for efficient estimation from SGLD samples of the posterior.

    \item In \cref{developmental_section}, we present \textbf{a stagewise empirical account of how the internal structure of a deep RL agent develops during training} in the gridworld environment of \citet{Misgen, RL1}.  Successive stages of training are characterized by qualitatively different susceptibility structures, culminating in an optimal phase during which the parameterization continues to simplify despite the policy itself being essentially fixed.  For uniformly-trained agents we document a parallel story in which analogous structure emerges via spontaneous symmetry breaking.

    \item In \cref{activation_steering_section} and \cref{direction_regret_section}, we \textbf{independently validate the representational structure detected by susceptibility analysis through two complementary observables}: an activation-space probe (activation steering) and a parameter-space probe (direction-conditioned posterior regret).  The directions flagged by susceptibilities are on average more resistant to being steered away from in activation space than non-flagged directions, and they are also the directions of lowest direction-conditioned posterior regret. Perturbing the training distribution biases which directions acquire this distinguished status.

    \item In the remaining subsections of \cref{evidence_section}, we \textbf{show that susceptibilities probe features of deep RL training dynamics that the policy trajectory does not reveal}: a training-distribution intervention establishes a causal role for phase-1-distinguished states in driving the phase 2 to phase 3 transition (\cref{antenna_section}); a skip-versus-visit comparison shows that phase transition history is encoded in end-of-training susceptibilities via transition timing (\cref{phase2_residue_section}); and a coordinated decline of multiple complexity measures during the optimal phase reveals continuing internal reorganization while the policy is approximately fixed (\cref{within_optimal_section}).
\end{enumerate}

The paper is structured as follows.  \Cref{background_section} collects background material on the Markov decision problem setup, local learning coefficients and presents our extension of susceptibilities to the reinforcement learning context.  We then present the cheese-in-the-corner environment and IMPALA architecture studied in the empirical sections.  \Cref{developmental_section} presents our empirical account of how susceptibility structure evolves, and what this tells us about the development of the model, through the phases of training.  Direct evidence for the individual interpretative claims made in this account is collected in \cref{evidence_section}.  \Cref{discussion_section} discusses limitations and open questions; supporting theoretical and experimental details are deferred to the appendices.

\section{Background} \label{background_section}

This section summarizes the background material on which our results rest.  Firstly, \Cref{markov_section} reviews the Markov decision problem setup and the generalized Bayesian posterior on policies introduced in \citet{RL1}.  Next, \Cref{LLC_section} introduces the local learning coefficient and its weight-restricted refinement.  \Cref{susceptibilities_background_section} reviews the susceptibility framework of \citet{Susceptibilities} and presents its natural extension to perturbations of the initial state distribution in the reinforcement learning setting.  Finally, \Cref{prior_results_section} describes the \emph{cheese in the corner} environment and IMPALA-style architecture on which our experiments are based (following the work of \citealp{Misgen}), and recalls the stagewise training behavior of these models established in \citet{RL1}.

\subsection{Markov Decision Problem Setup} \label{markov_section}
We begin by briefly summarizing the notation for finite Markov decision problems that we will use below.  Let $\mc S, \mc O, \mc A$ be finite sets of \emph{states, observations} and \emph{actions} respectively.  Fix a \emph{transition function}\footnote{We adopt the standard restriction on $p$ that the observation depends only on the output state.}
\[p \colon \mc S \times \mc A \to \Delta(\mc S \times \mc O \times \RR)\]
and write $r(s,a)$ for the expected value of the real-valued factor (the \emph{expected reward}).

We fix a parameterized family of \emph{policies}: a real analytic manifold $W$ together with a real analytic function $w \mapsto \pi_w \in \Delta(\mc A)^{\mc O}$.  We also fix an \emph{initial state distribution} $\Lambda \in \Delta(\mc S)$ and a maximum episode length $T_{\mr{max}}$.

Together this data determines a probability distribution on trajectories.  A \emph{trajectory} is a finite sequence
\[\tau = (s_0, o_0, a_1, s_1, o_1, \ldots, a_{T_{\mr{max}}}, s_{T_{\mr{max}}}, o_{T_{\mr{max}}})\]
where $s_i \in \mc S, o_i \in \mc O$ and $a_i \in \mc A$.  Write $\mc T$ for the finite set of all trajectories.  The transition function $p$, the initial state distribution $\Lambda$ and the parameterized family of policies $\{\pi_w\}$ together determine a real analytic family of probability distributions $\{q_w\}$ on $\mc T$ by
\[q_w(\tau) = \Lambda(s_0) p(o_0|s_0) \prod_{i=1}^{T_{\mr{max}}} \pi_w(a_i | o_{i-1}) p(s_i | s_{i-1}, a_i) p(o_i | s_i).\]

Let us now define the expected return and the regret as expectation values with respect to this family of distributions. Fix a discount factor $\gamma \in [0,1]$.  The \emph{return} of a trajectory is $r(\tau) = \sum_{i=1}^{T_{\mr{max}}} \gamma^{i-1} r(s_{i-1}, a_i)$.
The \emph{expected return} is $R(w) = \bb E_{\tau \sim q_w}(r(\tau))$ and the \emph{regret} is
\[G(w) = R_{\mr{max}} - R(w), \text{ where } R_{\mr{max}} = \sup_{w \in W} R(w).\]
Lastly, given a dataset $D_n = \{(w_i, \tau_i)\}_{i=1}^n$ with $\tau_i \sim q_{w_i}$, the \emph{importance-weighted regret estimator} is the unbiased estimator for $G(w)$ defined by
\[G_n(w) = \frac{1}{n} \sum_{i=1}^n \frac{q_w(\tau_i)}{q_{w_i}(\tau_i)} g(\tau_i)\]
where $g(\tau) = V^*(s_0) - r(\tau)$ and $V^*(s_0)$ denotes the optimal expected return from the initial state $s_0$ of $\tau$.  Since $s_0 \sim \Lambda$ under every $q_w$, the expected value $\bb E_{\tau \sim q_w}(V^*(s_0)) = \bb E_{s_0 \sim \Lambda}(V^*(s_0))$ does not depend on $w$, so $\bb E_{\tau \sim q_w}(g(\tau)) = G(w) + c$ where $c \ge 0$ is a constant that vanishes when the model contains an optimal policy, which we assume throughout.  The importance weights depend only on policy ratios and can be computed without knowledge of the transition function.

\subsection{Local Learning Coefficients} \label{LLC_section}

Singular learning theory (SLT) \citep{WatanabeGrey,WatanabeGreen} provides the mathematical framework for studying Bayesian learning in singular statistical models, including deep neural networks.  A central object in SLT is the \emph{local learning coefficient} (LLC) \citep{LLC}, a geometric invariant $\lambda(\wstar) > 0$ associated to each local minimum $\wstar$ of a loss function that measures the effective complexity of the model near $\wstar$.  The LLC governs Bayesian model selection through the \emph{free energy formula}: the free energy in a neighborhood $U$ of $\wstar$ satisfies
\[F_n(U) = nL_n(\wstar) + \lambda(\wstar) \log n + o_P(\log n)\]
where $L_n$ is the empirical loss.  Bayesian model selection is determined by where the posterior concentrates, and posterior mass on a neighborhood $U$ is proportional to $\exp(-F_n(U))$: regions with smaller free energy attract more weight.  The loss term in $F_n(U)$ scales with $n$ while the complexity term scales only as $\log n$, so for small $n$ the posterior prefers simpler solutions (smaller $\lambda$) and as $n$ grows the balance shifts toward lower-loss solutions, with rapid switches in concentration known as \emph{Bayesian phase transitions}.  See \cref{LLC_appendix} and \citet[\S 4]{RL1} for details.

In \citet{RL1} it was shown that the free energy formula generalizes to the reinforcement learning setting, with the regret $G$ playing the role of the loss and the importance-weighted estimator $G_n$ playing the role of the empirical loss.  The LLC of a local minimum $\wstar$ of $G$ therefore controls the concentration of the generalized posterior over policies, and the same complexity-regret tradeoff governs generalized Bayesian learning in deep reinforcement learning.

Given a decomposition $W = V \times C$ of the parameter space into a \emph{component} $C$ (for instance a single layer or convolutional block) and its complement $V$, the \emph{weight-refined LLC} (or \emph{rLLC}) $\lambda_C(\wstar)$ at a local minimum $w^* = (v^*, c^*)  \in V \times C$ is defined analogously but restricting to variations purely in the $C$ factor, holding the parameters in $V$ fixed at the value $v^*$ \citep{Wang}.  The rLLC measures the effective complexity contributed by the component $C$ to the model's representation near $\wstar$.  This construction applies equally in the reinforcement learning setting, and the results of \citet{RL1} imply that one may consider the rLLC as controlling the local complexity of the model at $\wstar$ along the component $C$.

Both the LLC and the rLLC may be estimated from the posterior distribution on $W$.  Fix an inverse temperature $\beta > 0$.  The LLC is recovered from the posterior expectation of the loss via
\[\lambda(\wstar) \approx n\beta \left(\langle L \rangle_{n\beta} - L(\wstar)\right)\]
and the rLLC is recovered similarly using the weight-restricted observable $\phi_C(w) = \delta_C(w)(L(w) - L(\wstar))$, where $\delta_C$ restricts to the subspace $\{v^*\} \times C$ (see \cref{LLC_appendix} for details).  In practice one approximates the posterior expectation by SGLD sampling \citep{LLC}.

\subsection{Susceptibilities} \label{susceptibilities_background_section}

The LLC and rLLC measure the complexity of a model at a particular point in parameter space; we would like to go a step further and study the dependence of this complexity on the structure of the data -- for instance in the reinforcement learning setting on the distribution over initial states visited during training.  \emph{Susceptibilities} \citep{Susceptibilities} address this by measuring the infinitesimal change in the rLLC under a perturbation of such data within the learning problem.  In this section we first review the construction in the supervised learning setting and then present our generalization to a version applicable in reinforcement learning.

\subsubsection{Susceptibilities in Supervised Learning}
We begin by reviewing the theory developed in \citet{Susceptibilities}.

Susceptibilities measure the local \emph{response} of a model, understood to mean the change in the posterior expectation value of an observable quantity, to a perturbation of the truth.  Consider the following distribution learning setting: fix a data space $X$, a true distribution $q \in \Delta(X)$ and a \emph{model}: a family of probability distributions $p_w$ on $X$ parameterized by a manifold $W$.  Consider the optimization problem on $W$ in which one aims to minimize the Kullback--Leibler divergence
\[K(w) = \mr{KL}(q || p_w) = \bb E_{x \sim q} \left( \log \frac{q(x)}{p_w(x)} \right)\]
and let $\wstar \in W$ be a global minimum of $K$.

Given a component $C \sub W$ of the model we may investigate how the rLLC of $C$ changes under perturbations of the truth.  For simplicity let us consider linear mixtures, so let $q'$ be another probability distribution on $X$ and let
\[q^h(x) = (1-h)q(x) + hq'(x)\]
for $0 \le h \le 1$.  This naturally leads to a corresponding deformation of the KL divergence
\[K^h(w) = \mr{KL}(q^h || p_w).\]
Recall the weight-restricted observable $\phi_C(w) = \delta_C(w)(K(w) - K(\wstar))$ from \cref{LLC_section}, whose posterior expectation measures the rLLC.  The \emph{susceptibility} associated to $C$ and the deformation $q^h$ is defined as the empirical covariance
\[\widehat\chi_C = -\widehat{\mr{Cov}}_{n\beta}(\phi_C(w),\, K^1(w) - K(w))\]
where the covariance is taken with respect to the Bayesian posterior at inverse temperature $\beta$ inherited from a Gaussian prior around $\wstar$.  Since $n\beta\langle \phi_C \rangle_{n\beta}$ approximates the rLLC, the susceptibility measures the response of the rLLC of the component $C$ to a linear perturbation of the truth.  Like the LLC it admits an efficient estimator via SGLD sampling; the details are presented in \citet{Susceptibilities}.

\subsubsection{Initial State Susceptibilities in Reinforcement Learning}
Let us now fix a finite Markov decision problem and study the natural generalization of the construction laid out in the previous section.  We now no longer have a ``true'' distribution to perturb; there are nevertheless a number of natural perturbations of the Markov decision problem that we may study instead.  In this paper we will focus on \emph{initial state perturbations}, although in \cref{susc_theory_appendix} we present a more general class of constructions.

So, let $\Lambda \in \Delta(\mc S)$ be the initial state distribution of a Markov decision problem (following the notation introduced in \cref{markov_section}).  We will consider linear perturbations of $\Lambda$.  So, let $\xi$ be a function $\xi \colon \mc S \to \RR$ so that $\sum_{s \in \mc S} \xi(s) = 0$, and suppose that $\xi(s) < 0$ only if $\Lambda(s) > 0$.
\begin{definition}
The linear deformation of $\Lambda$ associated to $\xi$ is the family of functions
\[\Lambda^h(s) = \Lambda(s) + h \xi(s)\]
for real parameters $h \ge 0$.  The function $\Lambda^h$ is a probability distribution for $0 \le h \le h_0$ for some positive $h_0$.
\end{definition}

\begin{example}
Let $s_0 \in \mc S$ be a state.  There is a linear deformation of $\Lambda$ defined for $0 \le h \le 1$ by
\[\Lambda^h_{s_0}(s) = (1-h)\Lambda(s) + h\delta_{s_0}(s) = \Lambda(s) + h(\delta_{s_0}(s) - \Lambda(s))\]
where $\delta_{s_0}$ is the constant distribution concentrated at $s_0$.
\end{example}

We will now define susceptibilities associated to such a linear deformation directly paralleling the construction presented for supervised learning.  Fix a model $W$ parameterizing a space of policies for the Markov decision problem.  Let $G$ denote the regret function, and let $G_n$ be its importance-weighting estimator.  The deformation $\Lambda^h$ of $\Lambda$ naturally determines deformations of the regret and its estimators, which we'll denote by $G^h, G_n^h$.

Let $C \sub W$ be a subspace, thought of as a \emph{component} of the parameter space.  Fix a number $n$ of data points and a value $\beta>0$ for the inverse temperature.  Finally, fix a local minimum $w^*$ of the regret.

\begin{definition}
The \emph{initial state susceptibility estimator} at $w^*$ associated to the linear deformation $\Lambda^h$ and a component $C$ is the empirical covariance
\[\widehat \chi_C(\xi) = -\widehat{\mr{Cov}}_{n\beta}((G_n(w) - G_n(w^*))\delta_C(w), G_n^1(w) - G_n(w)).\]
\end{definition}
As before, we refer to \cref{susc_theory_appendix} for details on how this empirical covariance is determined by sampling.

\subsubsection{Interpretation} \label{interpretation_section}
How, then, should one think about these susceptibilities?  Recall from \cref{LLC_section} that the rLLC $\lambda_C(\wstar)$ measures the effective complexity contributed by a component $C$ to the model's representation near a local minimum $\wstar$.  The susceptibility $\widehat \chi_C(\xi)$ agrees to leading order with the rate of change of the rLLC under the perturbation $\xi$ (see Remark \ref{susc_LLC_remark} in \cref{susc_theory_appendix} for a more precise statement).  In other words:
\begin{itemize}
 \item A \emph{positive} susceptibility $\widehat \chi_C(\xi) > 0$ indicates that perturbing the initial state distribution in the direction $\xi$ \emph{increases} the effective complexity of the model in the component $C$.
 \item A \emph{negative} susceptibility $\widehat \chi_C(\xi) < 0$ indicates that this perturbation \emph{decreases} the effective complexity in $C$.
\end{itemize}
One may therefore view the susceptibility as measuring how significantly different components of the model's internal algorithm are engaged by different classes of state.  If the rLLC of a component responds strongly to concentrating the initial distribution towards a particular state $s$, it suggests that the component plays a more significant role in the model's representation of how to act in states resembling $s$.

\begin{remark}
 There is a more general notion of susceptibilities in which one chooses both a perturbation (as above) and an observable, and measure the response of the posterior expectation value of that observable to the perturbation, or -- closely relatedly -- their posterior covariance.  The construction above specifically uses the weight-restricted regret as the observable. This choice is motivated by the connection to the rLLC.  One could equally well study susceptibilities for other observables -- for instance, functions of the model's activations or action probabilities -- and the covariance description would still apply (see Proposition \ref{susc_covariance_prop}).  We would, however, lose the description that is present for our weight-restricted regret specifically, that the susceptibilities would be closely entwined with the geometry of the regret landscape, and therefore by the free energy formula to the Bayesian learning behavior of the model.
\end{remark}

\subsection{Environment, Model and Prior Results} \label{prior_results_section}

In this section we present the empirical setting to be analyzed in detail: the simplified \emph{cheese in the corner} model studied in \citet{RL1}, following the environment and model considered in \citet{Misgen}.  We also review the empirical results from LLC estimation, in which the LLC exhibits spikes associated to \emph{phase transitions} occurring during the learning process.

\subsubsection{Environment} \label{environment_section}

The environment is an $11 \times 11$ grid, illustrated in top row of \cref{fig:alpha_0.6_single_run}.  One cell contains a goal location (the \emph{cheese}) and another the initial location of the agent (the \emph{mouse}).  The agent may move in any of the four cardinal directions; collisions with walls leave the agent's position unchanged.  The episode ends when the agent reaches the cheese, receiving return $\gamma^{T-1}$ where $T$ is the number of steps taken, or when the episode times out after $T_{\mr{max}}$ steps, receiving return $0$.

The initial state distribution is determined by a \emph{mixing parameter} $\alpha \in [0,1]$.  We define a distribution
\[\Lambda_\alpha = (1-\alpha)\Lambda_{\mr{corner}} + \alpha \Lambda_{\mr{uniform}}\]
on the state space, where under $\Lambda_{\mr{corner}}$ the cheese always spawns in the top-left corner (with the mouse uniformly distributed over all other cells) and under $\Lambda_{\mr{uniform}}$ both the cheese and mouse locations are uniformly distributed over all distinct pairs of cells.  When $\alpha = 0$ the agent only ever encounters the cheese in the corner and strategies that proceed to the top-left corner via a shortest path have zero regret.  As $\alpha$ increases the agent is increasingly strongly incentivized to learn policies that navigate to the cheese when it is located anywhere in the grid.  We discuss the implementation of this environment precisely in \cref{environment_appendix}.

\subsubsection{Model}

The model follows the IMPALA-style architecture used in \citet{Misgen} and \citet{RL1}; see \cref{model_appendix} for full details.  It consists of a convolutional encoder followed by a feedforward network.  The encoder is a sequence of three convolutional blocks, each consisting of a convolution, max-pooling, and two residual blocks, with channel dimensions $3 \to 16 \to 32 \to 32$.  The output is flattened and passed through two feedforward layers (FF 1 and FF 2, the latter replacing the LSTM used in the original IMPALA architecture \citep{IMPALA}) to produce a $256$-dimensional representation.  A final linear layer produces logits over the four actions.

For susceptibility estimation we define six components within parameter space.  Each component consists of the weights belonging to one of the following modules; the full parameter space is the direct sum of these components.
\begin{itemize}
 \item \emph{Conv 1--3}: The three convolutional blocks of the encoder.
 \item \emph{FF 1}: The linear layer mapping the flattened convolutional output to the observation embedding.
 \item \emph{FF 2}: The feedforward layer mapping the observation embedding to the pre-policy representation.
 \item \emph{Policy}: The policy layer producing action logits.
\end{itemize}
We estimate susceptibilities with respect to each of these six weight restrictions independently, providing a decomposition of the model's response to perturbations across the architectural components of the model.

\subsubsection{Phases and LLCs}

It was observed in \citet{RL1} that models trained on the simplified cheese-in-the-corner environment with intermediate values of $\alpha$ exhibit \emph{stagewise development}: training proceeds through a sequence of phases in which the regret is approximately constant separated by rapid transitions.  The following three phases are typical for models trained with $\alpha = 0.6$ and $\gamma = 0.98$ (the hyperparameters used in this paper).
\begin{itemize}
 \item \emph{Phase 1}: The agent moves up with probability $0.5$ and left with probability $0.5$ in all states.
 \item \emph{Phase 2}: The agent moves deterministically towards the top-left corner, choosing a shortest path that passes through the cheese when possible.
 \item \emph{Phase 3}: The agent moves deterministically towards the cheese along a shortest path in all states.
\end{itemize}

The LLC estimator tracks this stagewise development.  Phase transitions are accompanied by rapid increases in the LLC, and later phases correspond to larger LLC values, consistent with the theoretical prediction that Bayesian phase transitions should proceed from simpler to more complex policies.

We additionally observed that within phases 2 and 3 the LLC estimate exhibits a \emph{decline} over the course of training.  During this period the policy itself does not change appreciably -- the regret remains approximately constant -- but the representation of the policy in parameter space continues to evolve.  This within-phase decline in complexity suggests that SGD progressively simplifies the parameterization of a fixed policy even as the regret remains approximately constant.  Our susceptibility analysis below provides more concrete evidence for this hypothesis.

\section{Susceptibilities and the Developmental Process} \label{developmental_section}
We will now give a detailed phase-by-phase developmental account of the structural development of initial state susceptibilities within the cheese-in-the-corner environment.  We explain what structure in the model's parameterization these susceptibilities reveal, structure that is not detectable from the development of the learned policy alone.  In the following \cref{evidence_section} we present statistical evidence that supports the empirical claim we make here; inline references for each claim are presented to the corresponding evidence sections.

\subsection{Models and Data}

The empirical results in this section draw on three classes of cheese-in-the-corner models, each trained with vanilla REINFORCE (see \cref{model_appendix}) under a different choice of mixing parameter $\alpha$ and discount factor $\gamma$.
\begin{itemize}
    \item $\alpha = 0.6$, $\gamma = 0.98$: the models studied in \citet{RL1}, on which we develop our main account of phase 1 to phase 3 development.  These hyperparameters were chosen to consistently show three phase structural development.
    \item $\alpha = 0.5$, $\gamma = 0.99$: hyperparameters under which there are substantial fractions of training runs \emph{both} that exhibit three phase development, \emph{and} that skip directly from phase 1 to phase 3.  This allows for a direct comparison between models that did and did not exhibit phase 2 structure (\cref{phase2_residue_section}).
    \item $\alpha = 1$, $\gamma = 0.98$: uniform training, under which there is no incentive to visit intermediate phases and so the model proceeds directly to optimal policies.  The absence of any bias toward particular cardinal directions in the training distribution makes this setting a clean testbed for the developmental component of the phase 3 asymmetry.
\end{itemize}

\subsection{Three-Stage Development at \texorpdfstring{$\alpha < 1$}{alpha < 1}} \label{three_stage_section}
When visualizing susceptibilities, we have found it useful to color states based on the relative position of the mouse to the cheese, see top row of \cref{fig:alpha_0.6_single_run}. This is classified into 4 cardinal direction and 4 quadrants, and the direction states indicates the location of the mouse relative to the cheese: \emph{Right} corresponds to states where the mouse starts directly to the right of the cheese, \emph{Down-Right} corresponds to the mouse starting strictly below and strictly to the right of the cheese etc. All 8 relative positions form a partition of all possible states.

\begin{figure}
    \centering
    \includegraphics[width=0.99\linewidth]{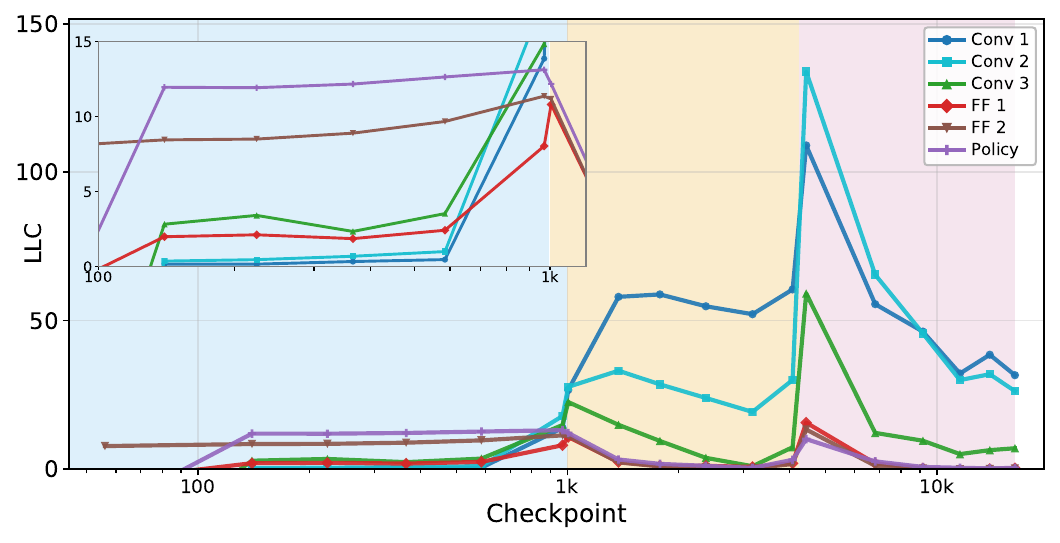}
    \caption{\textbf{Weight-restricted LLCs for the individual layers of the model.} We see that in phase 1 (blue background), while the policy is "blind", the LLC is dominated by the two last layers.  Then, as the model enters phase 2 (beige background) and learns to "see", the Conv layers activate and start to dominate the LLC. We note the LLC of all layers have a peak as the model enters phase 3 (magenta background).}
    \label{fig:single_wr_llc}
\end{figure}
In \cref{fig:alpha_0.6_single_run}, we see an example of the susceptibilities for phase 1 and phase 2, as well as susceptibilities at the transition between phase 2 and 3, for the start and the end of phase 3.
 The weight-restricted LLCs from \cref{fig:single_wr_llc} and susceptibility structure together motivate a stagewise developmental account, which we develop in the subsections below. Throughout this subsection, aggregate statistics are reported over a fixed population of ten $\alpha = 0.6$ training runs for which we estimated susceptibilities at all six weight restrictions across a wide range of checkpoints spanning phases 1--3 (elsewhere in the paper we consider larger populations of seeds sampled at the FF 2 weight restriction only and at a smaller set of checkpoints). We refer to this group throughout as the ``ten-run $\alpha = 0.6$ population''.
For the regret and LLC estimates for the full model, see the middle left panel of \cref{fig:alpha_0.6_single_run} and for individual components, see \cref{fig:single_wr_llc}.

\begin{remark}
The overall scale of the susceptibility data changes considerably between the phases, and phase transitions are accompanied by changes in scale as the global and weight-restricted LLCs jump.  We observe that the scale expands by roughly an order of magnitude entering phase 2, expands by a further factor of $\sim 4$ entering phase 3, and contracts by a factor of $\sim 5.5$ over the course of phase 3.  To quantify this we measure the standard deviation of the full 6D susceptibility vector for each checkpoint in the ten run $\alpha = 0.6$ population, reported in \cref{tab:scale_statistics}.

\begin{table}[h]
\centering
\caption{Susceptibility scale across phases, summarized as the mean $\pm$ standard deviation of the per-panel susceptibility standard deviation across the ten-run $\alpha = 0.6$ population.}
\label{tab:scale_statistics}
\begin{tabular}{lr}
\toprule
Phase & Mean $\pm$ Std \\
\midrule
1 & $3.5 \times 10^{-4} \pm 1.7 \times 10^{-4}$ \\
2 & $3.1 \times 10^{-3} \pm 6.8 \times 10^{-4}$ \\
3 start & $1.25 \times 10^{-2} \pm 2.7 \times 10^{-3}$ \\
3 end & $2.2 \times 10^{-3} \pm 1.4 \times 10^{-3}$ \\
\bottomrule
\end{tabular}
\end{table}
\end{remark}

\subsubsection{Phase 1: a ``Blind'' Policy.}
From the weight restricted LLCs, i.e. the LLCs of the individual layers, we see that during phase 1, the LLC is dominated by the last two layers.  This is expected, as phase 1 is a ``blind'' policy, and therefore is insensitive to perturbations of the conv-layers.  We see that the susceptibility scatter plot is dominated by Right, Down-Right and Down, with the other directions being concentrated near the origin -- these are the only states for which the model has the potential to receive reward in this phase.  In the pooled phase 1 susceptibility matrix the convolutional components account for only $1.2\%$ of the total variance.

We observe that, nevertheless, the susceptibilities demonstrate that the model already demonstrates some latent geometric ``understanding'' of the environment even within phase 1, where the policy is independent of the input state.  We see that the susceptibilities differentiate the  Right, Down-Right and Down states from the other states.  In particular the susceptibilities are structurally different for those states in which the phase 1 policy can reach the cheese and those states in which it cannot.  This distinction only begins to affect the policy at the subsequent transition to phase 2.

Looking at the phase 1 susceptibilities in \cref{extra_plots_appendix}, shown in more detail in \cref{fig:antenna_states}, we consistently see streaks in the Up-Right and Down-Left directions, each separated from the rest of the Up-Right and Down-Left susceptibilities. The streaks correspond to states where the mouse starts on the top row and the cheese is placed just below the top left corner, and states where the mouse starts on the left wall, and the cheese is just to the right of the top left corner.  We hypothesize that these streak states play a causal role in the phase 2 to phase 3 transition; see the right panel of \cref{fig:antenna_states} for the intervention experiment supporting this claim, with additional details in \cref{antenna_section}.

\begin{figure}
\centering
    \includegraphics[width=0.99\linewidth]{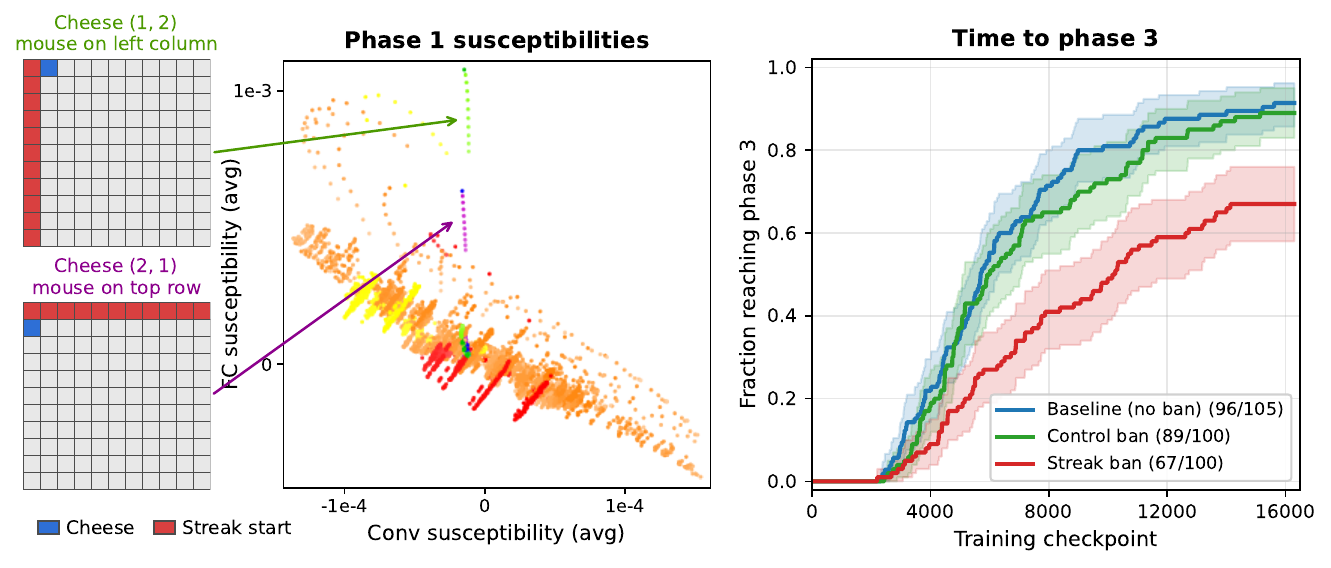}
    \caption{\textbf{``Streaks'' in phase 1 susceptibilities.}  In the left panels we see the states corresponding to the streaks, where the cheese is placed to the right of the top left corner and the mouse is located along the left column (top) where the cheese is placed below the top left corner and the mouse is located along the top row (bottom). The arrows are pointing to the linear pattern in the susceptibility plot in the middle panel, which is a zoomed in version of phase 1 of \cref{fig:alpha_0.6_single_run} with larger markers. This linear pattern is made up of the states to the left, and are what we refer to as streaks. To the right, we show the impact on the time to reach phase 3 of removing the streaks from the training distribution (red line), together with the baseline training distribution (blue) and a control baseline where a collection of states that don't separate into streaks have been removed (green). The shaded region shows the 95\% confidence interval. For more details, see \cref{antenna_section}.}
    \label{fig:antenna_states}
\end{figure}

\subsubsection{Phase 2: Learning to See.}
As the model enters phase 2 we see the LLC of the convolutional blocks shoot up and start to dominate. The conv-blocks dominate the behavior starting from phase 2, as seen in the weight refined LLCs.  This is also seen by PC1 starting in phase 2 explaining $87\%$ to $97\%$ of the variation in the susceptibilities, and PC1 leaning mostly on the conv directions. This transition reflects the development of spatial representations in the convolutional blocks: phase 2 requires processing the relative position of the mouse and the cheese, though only in a simple way: avoiding moving too far in the up or left directions.  In the susceptibility scatter the same three directions (Right, Down-Right, Down) continue to dominate the spread -- these are the states where the corner-seeking policy already obtains near-optimal reward, so perturbations in their direction most actively probe the regret landscape -- while the remaining initial state susceptibilities for which no reward is received stay concentrated near the origin.

\subsubsection{Phase 3 Onset: Mixture of Algorithms} \label{branching_logic_section}
As the model enters phase 3, the other directions explode into a second cloud, indicating the formation of a new algorithm.
The remaining directions: Down-Left, Left, Up-Left, Up, Up-Right (those states where the phase 2 corner-seeking policy received zero reward) now appear with large positive susceptibility and separated from the warm-color directions which remain closer to and below the origin.  PC1 remains Conv-dominated in every run, and above $95\%$ in eight of the ten, indicating that the new structure is again primarily located in the convolutional blocks.
For all these checkpoints, though less so at the end of training, we see that the susceptibilities for Right, Down-Right and Down differ from the other susceptibilities.

An analogous clustering phenomenon is observed in the $\alpha = 1$ models; we defer the $\alpha = 1$ case to \cref{alpha1_section}.  For $\alpha < 1$, susceptibilities associated to initial states where the mouse is located below and to the right of the cheese -- in which phase 2 policies receive optimal reward -- have minimal susceptibilities within all of the top three principal components with respect to the susceptibility matrix at the beginning of phase 3.  In other words, although the policy itself is symmetrical -- it obtains approximately optimal reward independent of the relative position of the mouse and cheese -- its representation in parameter space is \emph{asymmetrical}.

Following the interpretation in \cref{interpretation_section} this asymmetry in susceptibilities suggests that the implementation of the optimal policy under the deep neural network architecture is asymmetrical from the point of view of the geometry of the loss landscape.  That is, increasing the representation of states in the training distribution asymmetrically affects the refined LLCs for each component.  There is a distinguished class of states associated to a subset of relative positions of the mouse and cheese whose increased presence \emph{reduces} the refined LLCs (simplifies the model), while increasing the presence of generic states \emph{increases} the refined LLCs (increases the complexity of the model).

\begin{remark}
One natural mechanism for this kind of asymmetry is \emph{branching logic}: an algorithm that routes different classes of input through subroutines of differing complexity.  In its simplest form, consider a parameterized function $F_w \colon X \to Y$ that factors as
\[F_w(x) = \begin{cases} F^1_{w_1}(x) & x \in X_1 \\ F^2_{w_2}(x) & x \in X_2 \end{cases}\]
for a partition $X = X_1 \sqcup X_2$ and a corresponding decomposition $w = (w_1, w_2)$ of the parameters.  Even if $F^1$ and $F^2$ implement the same function, the geometry of the parameterizations $w_1 \mapsto F^1_{w_1}$ and $w_2 \mapsto F^2_{w_2}$ near a common optimum may differ, and this difference would be detected by the rLLC for the components spanning the two factors in parameter space.  We would anticipate that the component associated to the simpler subroutine would have a lower refined LLC.  Susceptibilities would then reflect this asymmetry, since concentrating the initial distribution towards inputs in $X_1$ would probe the geometry of $F^1$ more intensely than that of $F^2$.

Of course our deep neural network does not implement branching logic in such a simple sense.  Nevertheless, a network may implement an \emph{approximate} version of this structure, in which the dependence of the output on certain weights is much stronger for some classes of input than for others.  The susceptibility analysis detects exactly this kind of approximate conditional structure.
\end{remark}

\begin{example}[The case of $\alpha < 1$]
When $\alpha < 1$ the training distribution $\Lambda_\alpha$ is biased towards states in which the cheese is in the top-left corner.  The model first learns phase 2 -- a policy that navigates towards the corner -- and only later transitions to phase 3 -- a policy that navigates towards the cheese regardless of its location.  Crucially, phase 2 policies are already optimal for the subset of states in which the mouse is below and to the right of the cheese, since moving towards the corner is also moving towards the cheese.  We therefore expect the transition from phase 2 to phase 3 to primarily involve developing new computational structure for the \emph{remaining} states, while the representation of how to act in the ``easy'' states -- those already handled correctly by phase 2 -- can remain comparatively simple.  The lower susceptibilities associated to these states match this prediction.  It leads us to hypothesize that the model retains a simpler internal ``subroutine'' that addresses these states, inherited from phase 2, while deploying a more complex subroutine for states requiring genuinely new behavior.  This also provides a more concrete interpretation of the within-phase decline in the LLC noted in the previous section.  Our hypothesis would lead us to predict that learning favors the progressive simplification of the representation of an optimal policy by eliminating the residual asymmetrical representation learned in phase 2. For visualizations of the susceptibilities for different runs with $\alpha=0.6$, see \cref{fig:alpha_0.6_all_susceptibilities}.
\end{example}

Two independent lines of evidence support the functional reality of the asymmetric phase 3 representation: activation steering experiments (\cref{activation_steering_section}) and direction-conditioned posterior regret expectations (\cref{direction_regret_section}).  A further analysis of the effect of the duration spent in phase 2 during training (\cref{phase2_residue_section}) shows that the timing of the phase transitions is clearly detectable in the susceptibilities at the end of training.  The effect is, however, mediated primarily by the duration spent in phase 3 rather than the duration spent in the intermediate phase 2.

\subsubsection{Within Phase 3: Progressive Simplification}
Over the course of phase 3 the two clouds start merging as the LLC decreases towards the end of training.  The overall susceptibility scale shrinks by approximately an order of magnitude, with the per-seed ratio of phase-3-start to phase-3-end standard deviations averaging 7.1 across the ten-run $\alpha = 0.6$ population. Since the policy does not substantively change during phase 3, this reflects a simplification of the parameterization of an approximately fixed policy rather than any change in behavior.  One might naturally hypothesize that, during phase 3, SGD progressively simplifies the representation by erasing the asymmetric structure developed at the transition from phase 2 to phase 3: the distinction discussed in \cref{branching_logic_section} between the ``easy'' states (handled by a simpler subroutine established at the phase 2 to phase 3 transition) and the ``hard'' states (handled by the more complex subroutine developed at the phase 3 transition) is gradually eliminated.  We develop this simplification in detail, alongside geometric evidence from the Hessian trace, in \cref{within_optimal_section}.

\subsection{The \texorpdfstring{$\alpha = 1$}{alpha = 1} Case: Spontaneous Symmetry Breaking} \label{alpha1_section}
\begin{figure}
    \centering
    \includegraphics[width=0.99\linewidth]{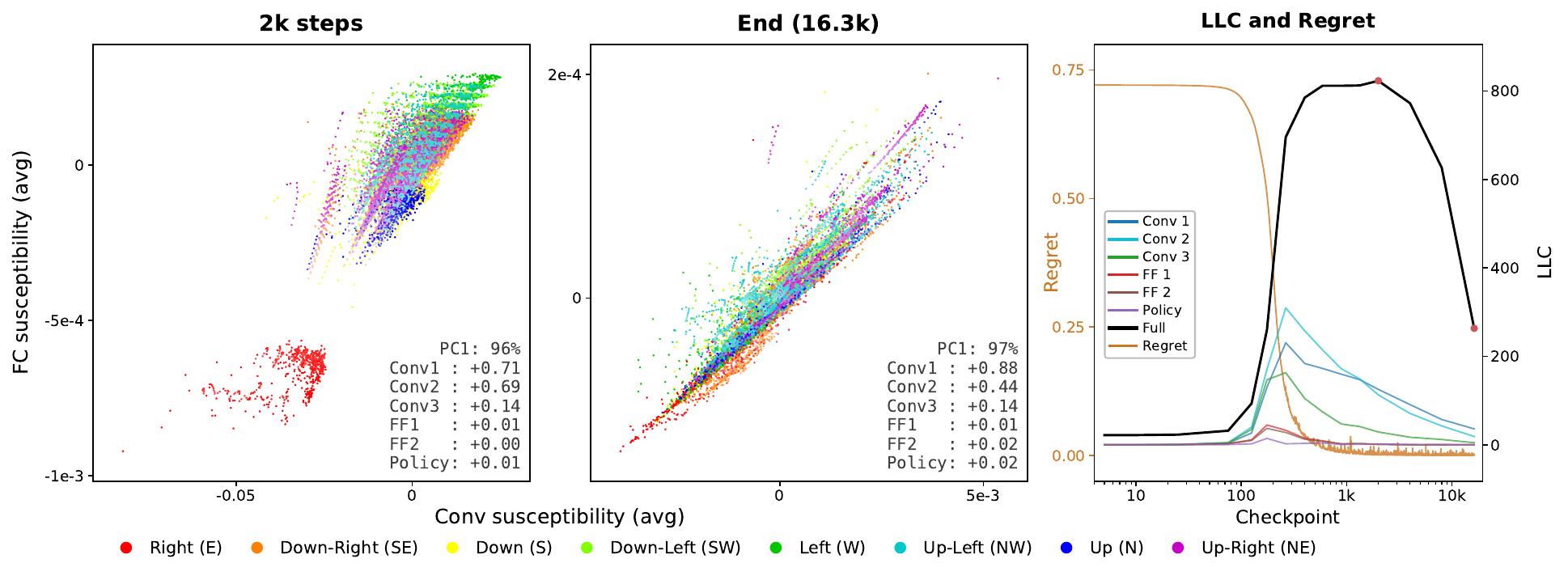}
    \caption{Susceptibilities and LLC estimator compared to regret for a single training run with $\alpha=1$ and $\gamma=0.98$. The red dots on the LLC estimator curve correspond to the checkpoints at which the susceptibilities have been evaluated. The variation explained by PC1 and the cosine similarity between the PC1 and each of the weight restriction directions is indicated at the upper left corner of each susceptibility panel.}
    \label{fig:alpha_1.0_single_run}
\end{figure}
In \cref{fig:alpha_1.0_single_run}, we see the susceptibilities and weight restricted LLCs of a model trained with $\alpha=1$ at checkpoint 2k and at the end of training. At checkpoint 2k we see that the Right states form a cluster different from the rest of the directions, suggesting that the network handles these initial states differently. As the LLC drops towards the end of training, we see this special direction increasingly folding into the other susceptibilities, being treated less as a special direction. Just like in phase 3 of the $\alpha=0.6$ case, the LLC and susceptibilities are dominated by the Conv layers, and they decrease in a similar fashion throughout phase 3.

This asymmetry -- in which one or more cardinal directions form clusters with distinct and more negative susceptibilities in at least the three first principal directions in observable space $\RR^6$ -- is repeated across the majority of independent seeds for $\alpha = 1$.  79.2\% of seeds have 1--2 \emph{distinguished} directions, defined by comparing the susceptibilities for the FF 2 layer at training step 2000: a direction is distinguished if the 99th percentile of its set of susceptibilities is lower than the 1st percentile of the set of all diagonal state susceptibilities.

The distribution across cardinal directions and pairs of directions is given in \cref{tab:alpha1_direction_freq}.  The marginal frequencies are broadly balanced (Right 21.8\%, Down 19.8\%, Left 29.7\%, Up 29.7\%), with a mild preference for Left and Up.

\begin{table}[h]
\centering
\caption{Frequencies of distinguished cardinal directions at training step 2000 across 101 $\alpha = 1$ seeds.  Diagonal entries give the frequency with which a single direction is distinguished; off-diagonal entries give the frequency with which a pair of directions is simultaneously distinguished.  The 20.8\% of seeds with zero or three-or-more distinguished directions do not appear in the table.}
\label{tab:alpha1_direction_freq}
%\begin{tabular}{lrrrr}
%\toprule
% & Right & Down & Left & Up \\
%\midrule
%Right & $12.9\%$ & $5\%$ & $0\%$ & $4\%$ \\
%Down  &          & $8.9\%$ & $5\%$ & $1\%$ \\
%Left  &          &          & $17.8\%$ & $6.9\%$ \\
%Up    &          &          &          & $17.8\%$ \\
%\bottomrule
%\end{tabular}
\begin{tabular}{lrrrr}
\toprule
 & Right & Down & Left & Up \\
\midrule
Right & $12.9\%$ $[6.9, 19.8]$ & $5.0\%$ $[1.0, 9.9]$ & $0.0\%$ & $4.0\%$ $[1.0, 7.9]$ \\
Down &  & $8.9\%$ $[4.0, 14.9]$ & $5.0\%$ $[1.0, 9.9]$ & $1.0\%$ $[0.0, 3.0]$ \\
Left &  &  & $17.8\%$ $[10.9, 25.7]$ & $6.9\%$ $[2.0, 11.9]$ \\
Up &  &  &  & $17.8\%$ $[10.9, 25.7]$ \\
\bottomrule
\end{tabular}
\end{table}

\begin{remark}
    These results are consistent across a variety of different processes by which one may detect clustering.  For further discussion see \cref{cluster_criteria_appendix}.
\end{remark}

As in the $\alpha < 1$ case, we find that the asymmetry in susceptibilities corresponds to a functional asymmetry in the model's internal computation.  The activation steering and training-distribution intervention experiments presented in \cref{activation_steering_section} provide direct evidence for this in the $\alpha = 1$ setting.

\subsubsection{Spontaneous Symmetry Breaking}
When $\alpha = 1$ the training distribution is uniform and there is no bias towards any region of the state space.  The symmetry group of the environment acts on the state space, and the optimal region in policy space is preserved by this symmetry.  Nevertheless, the susceptibility estimates frequently exhibit clustering associated to one or two cardinal directions, and this choice varies across random seeds.  This is a form of \emph{spontaneous symmetry breaking}: the optimization process selects, from among the many equivalent optimal parameterizations, one that treats certain directions more simply than others.  The susceptibility analysis detects which directions were selected.  Since there is no training-history explanation analogous to the $\alpha < 1$ case, the origin of the asymmetry would necessarily lie in the interaction between the random initialization and the dynamics of policy gradient learning. For visualizations of the susceptibilities for different runs with $\alpha=1$, see \cref{fig:alpha_1.0_all_susceptibilities}.

\subsubsection{Sensitivity to Initialization} \label{init_sensitivity_section}
Since susceptibility clustering varies across independent training runs, it is natural to ask what determines which direction is selected in a given run.  In our environment the symmetry breaking is most strongly determined by the initialization of the model parameters rather than the trajectory randomness introduced by mini-batch sampling and policy rollouts. To check this we may fix the initialization and vary the random seed controlling trajectory sampling during training.  We defer the experimental details and supporting evidence to \cref{fixed_init_appendix}.

\section{Empirical Evidence} \label{evidence_section}
We will now provide the specific evidence for the claims made in \cref{three_stage_section,alpha1_section} about the internal structure of the model represented by our susceptibility data.  Each subsection contains experimental and statistical evidence for one or more of these claims.

\subsection{Methodological Conventions} \label{methods_conventions_section}

We first record two conventions that are used throughout the evidence below.

\subsubsection{Phase Classification}
We detect models that did and did not visit phase 2 (and measure the number of checkpoints spent in the phase) by measuring an $L^2$-distance from the appropriate region in policy space.  That is, we classify a checkpoint as visiting phase $P$ if the $L^2$ distance from the policy at that checkpoint to the phase-$P$ subspace of policy space falls below $0.15$ times the diameter of policy space.  For more details and a discussion of this threshold see \cite[Section 6.4.1, Appendix H]{RL1}.

\subsubsection{On- and Off-Distribution Susceptibility Estimation} \label{on_off_remark}
Given a parameter $w^* \in W$ one can study susceptibility estimates associated to \emph{any} loss function, not necessarily the same as the loss that was used to find this parameter.  The susceptibility estimator is a sensible quantity for any loss function $G'$ for which $w^*$ is approximately a local minimum.  For example, in our cheese-in-the-corner environment, if $w^*$ is approximately optimal for the regret $G$ associated to an initial state distribution $\Lambda_\alpha$ as in \cref{environment_section}, it will also be approximately optimal for the regret $G'$ associated to any other $\Lambda_{\alpha'}$ as long as $\alpha, \alpha' > 0$.  We may, for example, use the universal choice $\alpha'=1$. The analogous choice is known to affect LLC estimation in the cheese-in-the-corner environment \cite[Remark 6.2, Appendix F]{RL1}, and we find that it also affects susceptibility estimation.  Unless otherwise specified we use on-distribution susceptibility estimates, and we discuss the dependence on this distinction in more detail in \cref{on_off_appendix}.

\subsubsection{Distinguished Directions} \label{distinguished_directions_section}
Several of the analyses below partition the directions in observable space into \emph{distinguished} and non-distinguished classes based on the susceptibility data.  We use a percentile criterion on the FF 2 layer susceptibilities, with a slightly different form for $\alpha = 1$ and $\alpha < 1$ models reflecting the cluster structures of \cref{branching_logic_section,alpha1_section}.  For $\alpha = 1$ models, where the natural cluster comprises cardinal directions only, a cardinal direction $D_c$ is distinguished if the 99th percentile of $\widehat \chi_{\mr{FF 2}}(D_c)$ lies below the 1st percentile of $\widehat \chi_{\mr{FF 2}}(D_d)$ pooled over the diagonal directions $D_d$ (the strict ``P99/P1'' criterion).  For $\alpha < 1$ models, where the natural cluster typically includes diagonals (e.g.\ the $\{\mr{Right}, \mr{Down\text{-}Right}, \mr{Down}\}$ cluster of \cref{branching_logic_section}), the strict P99/P1 criterion fires on too few seeds; we instead rank all eight directions by median FF 2 susceptibility and find the largest low-median group whose pooled 95th percentile lies below the 5th percentile of the remaining directions, allowing arbitrary bipartitions (the generalized ``P95/P5'' criterion).  Alternative criteria (including mean-separation and gap-based variants) yield qualitatively similar partitions; see \cref{cluster_criteria_appendix} for details.

\subsubsection{Confidence intervals}
\label{confidence_intervals_section}
All confidence intervals (CI) in this paper, denoted as $\text{mean}$ $[\text{lower CI}, \text{upper CI}]$, are 95\% percentile bootstrap intervals over independent training runs estimated with B = 10,000 resamples.

\subsection{Phase 1 Streak Intervention} \label{antenna_section}

We test the claim of \cref{three_stage_section} that the phase 1 ``streak'' states play a causal role in the phase 2 to phase 3 transition.

In phase 1 the streak states are, in most runs, the only states outside the Down-Right cluster (see \cref{fig:antenna_states}) with non-negligible susceptibilities; the rest of these two direction classes have susceptibilities concentrated near the origin. When we estimate susceptibilities in phase 1 across our population of ten runs with $\alpha = 0.6$ we find streak state susceptibilities with a mean $\pm$ std $L^2$ norm of $1.4\times 10^{-3} \pm 7.7\times 10^{-4}$ compared to $3.6\times 10^{-4}\pm3.1\times10^{-4}$ for the non-streak states outside the Down-Right cluster. If the model's parameterization handles direction classes coherently -- consistent with the phase 3 susceptibility clusters of \cref{branching_logic_section} -- then the streak susceptibilities in phase 1 are the main channel through which states that don't receive reward influence the parameterization.  Removing this channel should therefore impair the activation of these direction classes at the phase 2 to phase 3 transition, and thereby impair the model's ability to reach phase 3 at all.

We test this conjecture as follows. We train models with the same $\alpha=0.6$ training distribution, but with the states corresponding to the phase 1 Up-Right and Down-Left streaks removed. The impact of this can be seen in \cref{fig:antenna_states}, right panel. For example, we note a reduction in training runs reaching phase 3 by checkpoint 5000 from 35.2\% to 18.0\%, and a reduction from 91.4\% to 67.0\% at the end of training. We show these reach fractions and confidence intervals in \cref{tab:phase3-reach}. As a control we train models with the mirrored states removed (i.e. mouse bottom row, cheese just above bottom right corner and mouse right wall, cheese just left of bottom right corner), resulting in no significant difference from the baseline.

\begin{table}[t]
\centering
\caption{Fraction of seeds that have reached phase 3 by checkpoint
\(T\), with 95\% bootstrap confidence intervals (10{,}000 resamples
over seeds).}
\label{tab:phase3-reach}
\begin{tabular}{lccc}
\toprule
Condition & \(n\) & \(T = 5{,}000\) & \(T = 16{,}270\) (end of training) \\
\midrule
Baseline (no ban)  & 105 & 35.2\,[26.7,\,44.8] & 91.4\,[85.7,\,96.2] \\
Control ban        & 100 & 37.0\,[28.0,\,47.0] & 89.0\,[83.0,\,95.0] \\
Streak ban         & 100 & 18.0\,[11.0,\,26.0] & 67.0\,[58.0,\,76.0] \\
\bottomrule
\end{tabular}
\end{table}
\subsection{Functional Evidence for Cluster Structure} \label{activation_steering_section}

We present two related lines of evidence for the claim that the phase 3 susceptibility clusters of \cref{branching_logic_section,alpha1_section} correspond to a functional asymmetry in the model's internal computation.  The first is a training-distribution intervention that biases the emergence of distinguished cluster directions in $\alpha = 1$ models.  The second is an activation steering experiment, presented first for the $\alpha = 1$ case and then extended to $\alpha < 1$ models.  We also report a partial-predictivity finding for $\alpha = 0.5$: the activation steering asymmetry is stronger than the susceptibility analysis alone would identify.

\subsubsection{Training-Distribution Intervention: Biasing the Distinguished Direction}
We replace the $\alpha = 1$ training distribution with the mixture
\[\Lambda^{\mr{R}}_\eta = \eta \Lambda_{\mr{uniform}} + (1 - \eta) \Lambda_{\mr{Right}},\]
where $\Lambda_{\mr{Right}}$ is the uniform distribution over initial states with the mouse cardinally right of the cheese.  Right becomes the unique distinguished direction in 60/100 seeds for $\eta=0.8$ and 77/100 seeds for $\eta=0.7$, compared to 13/101 seeds when training with $\alpha=1$.  This demonstrates that the training distribution can bias which direction becomes distinguished.  Note that we here compute the susceptibilities using the $\alpha'=1$ uniform distribution to rule out the asymmetry being induced by the SGLD sampling process.

\subsubsection{Activation Steering for $\alpha = 1$}
To independently test the hypothesis that susceptibility clustering corresponds to different algorithms, we construct an activation steering experiment, following the activation-addition / contrastive-activation-addition line of work \citep{Turner2023ActAdd, Rimsky2024CAA} and its application to maze-solving policy networks \citep{mini2023understanding}.  The motivation is that if the agent has learned two different algorithms depending on the cardinal direction, they are likely to have different robustness to steering. Specifically, we would expect that the more specialized algorithm should be more robust to steering.

We place the mouse and the cheese on our open grid, and record the activations $\mathbf{a}^\ell$ at the output of layer $\ell$. Then we mirror the cheese around the mouse, and record the mirrored activations $\mathbf{a}_M^\ell$. We then find the smallest scale $s$ so that adding $s(\mathbf{a}^\ell_M-\mathbf{a}^\ell)$ to the activations $\mathbf{a}^\ell$ causes the highest logit not to be the logit corresponding to walking towards the cheese. We then average $s$ over every possible cheese and mouse location for each cardinal direction, restricting to locations where the model already acts optimally and where mirroring does not bring the cheese outside the grid or to a tile with a wall. Doing this, we obtain an $\overline{s}_\text{min}^{\ell,{D_c}}$ for each layer $\ell$ and cardinal direction $D_c$. We then construct a test statistic $X_i^\ell = \text{mean}_{D_c\in \text{distinguished}_i}(\overline{s}_{\text{min},i}^{\ell,D_c})-\text{mean}_{D_c\notin \text{distinguished}_i}(\overline{s}_{\text{min},i}^{\ell,D_c})$, where $i$ denotes index over independent training runs, and the distinguished directions are defined by being distinguished in the susceptibility plots. We have found that the distinguished directions most of the time are visible in all layers, with the first and the last layers being the least reliable. To identify distinguished directions, we use the FF 2 layer, defining a cardinal direction $D_c$ as being distinguished if the 99th percentile of $\widehat \chi_\ell(D_c)$ lies below the 1st percentile of $\widehat \chi_\ell(D_d)$, where $D_d$ denotes the diagonal directions. We train $101$ independent seeds, of which $80$ have 1--2 distinguished directions and $79$ yield valid measurements for all directions. Distinguished directions being more robust to steering translates to $X_i>0$.

Under the hypothesis that the distinguished directions as shown by the susceptibilities are unrelated to $\overline{s}_\text{min}^{\ell,D_c}$, $\bb E_{i} \left(X_i^\ell\right) = 0$, and we perform a one-sample two-sided $t$-test and reject the hypothesis that $\bb E_{i} \left(X_i^\ell\right) = 0$ with p-values shown in \cref{tab:ttest-steering}.  We see that starting at Conv 3, $X_i$ is positive (distinguished directions are harder to steer), and the hypothesis that $\bb E_{i} \left(X_i^\ell\right) = 0$ can be strongly rejected. If we use that we expect the distinguished directions to be more robust to steering and instead perform a one-sided t-test with $\bb E_{i} \left(X_i^\ell\right) > 0$ as the alternative hypothesis, we get even stronger $p$-values.

\begin{table}[t]
  \centering
  \caption{One-sample $t$-test of steering threshold for distinguished vs.\ non-distinguished cardinal directions ($n=79$ seeds).}
  \label{tab:ttest-steering}
  \begin{tabular}{lcccc}
    \toprule
    Layer & $t$ & $p$ (two-sided) & mean($X_i$) & $X_i > 0$ \\
    \midrule
    \texttt{Conv 1} & $-2.62$ & $1.1 \times 10^{-2}$ & -0.025 [-0.043, -0.006] & 37\% [27, 48] \\
    \texttt{Conv 2} & $\phantom{-}0.89$ & $3.7 \times 10^{-1}$ & +0.009 [-0.011, +0.029] & 54\% [43, 66] \\
    \texttt{Conv 3} & $\phantom{-}6.11$ & $3.7 \times 10^{-8}$ & +0.073 [+0.050, +0.096] & 68\% [58, 78] \\
    \texttt{FF 1}     & $\phantom{-}6.79$ & $2.0 \times 10^{-9}$ & +0.080 [+0.057, +0.102] & 72\% [62, 81] \\
    \texttt{FF 2}    & $\phantom{-}6.67$ & $3.3 \times 10^{-9}$ & +0.076 [+0.054, +0.098] & 76\% [66, 85] \\
    \bottomrule
  \end{tabular}
\end{table}

Activation steering effect sizes are larger for the right-mixture distinguished directions than for the spontaneous $\alpha=1$ case, with $\eta=0.8$ producing a mean-to-std ratio of 2.27 for the seeds where Right has been induced as the only distinguished direction, compared to 0.57 for the $\alpha=1$ models with single distinguished directions.

\subsubsection{Activation Steering for $\alpha < 1$}

We extend the activation steering experiment to $\alpha < 1$ models.  We identify the distinguished cardinal directions from the phase 3 on-distribution susceptibilities (see \cref{cluster_criteria_appendix} for analysis of the choice of the cluster-detection procedure) and compute the test statistic $X_i^\ell$ at the FF 2 layer.

For $\alpha = 0.5$ models the susceptibility-distinguished directions have significantly higher steering resistance than the non-distinguished directions, both at phase 3 start ($\mr{mean}(X) = +0.12$ $[0.10, 0.15]$, $t = 9.3$, $n = 66$) and at phase 3 end ($\mr{mean}(X) = +0.25$ $[0.21, 0.29]$, $t = 12.8$, $n = 41$). In both cases a one-sample two-sided $t$-test rejects the null hypothesis $\bb E(X_i^{\mr{FF 2}}) = 0$ at $p < 10^{-12}$.  

In addition to the binary distinguished/non-distinguished test, we compute the per-seed Spearman rank correlation between the four cardinal directions' FF 2 susceptibility ranks and their steering threshold ranks.  The mean correlation is $\ol \rho = +0.45$ $[+0.37, +0.53]$ for $\alpha = 0.5$ models at phase 3 start and $\ol \rho = +0.49$ $[+0.40, +0.58]$  at phase 3 end, and $\ol \rho = +0.34$ $[+0.24, +0.44]$ for $\alpha = 0.6$ models at phase 3 start, decaying to $\ol \rho \sim +0.15$ $[+0.03, +0.26]$ at phase 3 end.  In each condition a one-sample two-sided $t$-test on the per-seed distribution of $\rho$ rejects the null hypothesis $\bb E(\rho) = 0$ at $p < 10^{-6}$, with the exception of the $\alpha = 0.6$ condition at phase 3 end, where the test gives $p = 0.015$.  The full rank analysis is presented in \cref{rank_correlation_appendix}. 

\subsubsection{Susceptibilities are not Fully Predictive for Steering Asymmetry}

For the $\alpha = 0.5$ models we make the following complicating finding: the cardinal directions Right and Down have higher steering resistance than Left and Up in nearly all seeds, including those seeds for which the susceptibility analysis does not identify Right or Down as distinguished. Among 104 matched $\alpha = 0.5$ seeds partitioned by whether the susceptibility criterion flags $\{\mr{Right}, \mr{Down}\}$, the fixed $\{\mr{Right}, \mr{Down}\}$-vs-$\{\mr{Left}, \mr{Up}\}$ test statistic is significantly positive in both subgroups (susceptibility flag $\{\mr{Right}, \mr{Down}\}$ distinguished: $\mr{mean}(X) = +0.25$ $[0.21, 0.29]$, $t = 12.82$, $n=41$; does not: $\mr{mean}(X) = +0.17$ $[0.14, 0.20]$, $t = 10.74$, $n=63$).  In both subgroups a one-sample two-sided $t$-test rejects the null hypothesis $\bb E(X_i^{\mr{FF 2}}) = 0$ at $p < 10^{-15}$.

The functional asymmetry between $\{\mr{Right}, \mr{Down}\}$ and $\{\mr{Left}, \mr{Up}\}$ is therefore only partly captured by the susceptibility analysis.  The mechanism for the persistence of this asymmetry beyond the susceptibility-flagged seeds is not clear; we discuss it in \cref{limitations_section}. 

\subsection{Phase 3 Susceptibilities Carry Phase Transition History} \label{phase2_residue_section}

We test whether phase 3 susceptibilities carry information about whether and for how long the model passed through phase 2 during training.  We exploit the $\alpha = 0.5, \gamma = 0.99$ population in which substantial fractions of training runs both visit and skip phase 2, transitioning directly from phase 1 to phase 3 in the latter case.

For each seed $i$ we compute the per-seed Down-Right discrepancy
\[\Delta_i = \ol \chi_{\mr{DR},i} - \tfrac{1}{3}\bigl(\ol \chi_{\mr{DL},i} + \ol \chi_{\mr{UL},i} + \ol \chi_{\mr{UR},i}\bigr),\]
where $\ol \chi_{D_d,i}$ is the mean on-distribution FF 2 susceptibility over states in diagonal direction class $D_d$ in seed $i$.  We measure $\Delta_i$ at the end of phase 3 and ask how it depends on features of the phase transition history, in particular the presence or absence of phase 2.  

Welch's two-sample $t$-test rejects the null hypothesis that the populations that visited and skipped phase 2 have equal expected value for $\Delta$ at the end of training ($p = 1.63 \times 10^{-4}$).  Treating the duration spent in phase 2 as a continuous covariate, the per-seed Spearman rank correlation between visit duration and $\Delta_i$ at end of training is $\rho = -0.367 $ $[-0.533, -0.184]$ ($p = 1.3 \times 10^{-4}$, n=104): seeds that visited phase 2 for longer have a more-distinguished Down-Right susceptibility cluster at end of training.  By comparison, the corresponding correlation at the first phase 3 checkpoint per seed is much weaker: $\rho = -0.21$ $[-0.400, -0.003]$ ($p = 3.7 \times 10^{-2}$, $n = 104$).

We attribute the presence of this distinction at the end of training primarily to the timing of the transition into phase 3 (we will discuss susceptibility decay within phase 3 in \cref{within_optimal_section} below). Models that enter phase 3 later in training have spent fewer steps in phase 3 by the end of training and have therefore experienced less of the decay. We test this by repeating the $t$ and Spearman tests described above at a checkpoint 6000 gradient steps after entering phase 3 in each model.  The $t$-test is no longer significant ($t = -1.46$, $p = 0.15$, $n = 104$), and the Spearman correlation between phase 2 duration and $\Delta_i$ is reduced to $\rho = -0.22$ $[-0.40, -0.02]$ ($p = 2.7 \times 10^{-2}$).

We return to the question of whether any phase-2-specific structure persists once phase 3 duration is held fixed, and what this says about the mixture-of-algorithms hypothesis presented in \cref{branching_logic_section}, in \cref{limitations_section}.

\begin{remark}
The same qualitative pattern holds when susceptibilities are evaluated off-distribution at $\alpha' = 1$: the Spearman rank correlation between phase 2 duration and $\Delta_i$ at end of training remains negative and significant ($\rho = -0.485$ $[-0.635, -0.309]$, $p = 1.9 \times 10^{-7}$, $n = 104$).
\end{remark}

\begin{remark} \label{steering_skip_visit_remark}
The naive residue reading of \cref{branching_logic_section} would predict that seeds visiting phase 2 retain more of the $\{\mr{Right}, \mr{Down}\}$-vs-$\{\mr{Left}, \mr{Up}\}$ steering asymmetry $X_i$ of \cref{activation_steering_section} than seeds skipping it.  We find the opposite: at the first phase 3 checkpoint the skipped population has $\mr{mean}(X) = +0.193$ $[+0.161, +0.226]$ ($n = 25$) versus $+0.111$ $[+0.084, +0.137]$ for visited ($n = 79$), Welch $t = -3.78$, $p = 3.77 \times 10^{-4}$; the effect strengthens at end of training (ckpt $16270$) to $+0.297$ $[+0.248, +0.343]$ ($n = 25$) versus $+0.176$ $[+0.149, +0.202]$ ($n = 79$), $t = -4.23$, $p = 1.36 \times 10^{-4}$.  We return to the question of the sign in \cref{limitations_section}.
\end{remark}

\subsection{Direction-Conditioned Posterior Regret} \label{direction_regret_section}

We present a parameter-space counterpart to the activation-space measurement of \cref{activation_steering_section}, providing a further line of evidence for the functional asymmetry claimed in \cref{branching_logic_section}.  For each cardinal direction $D_c \in \{\mr{Right}, \mr{Left}, \mr{Up}, \mr{Down}\}$, let $\Lambda_{D_c}$ denote the uniform distribution on the set of initial states in which the mouse is cardinally aligned with the cheese in direction $D_c$, and let $G_{D_c}$ denote the expected regret associated to the initial state distribution $\Lambda_{D_c}$, as in \cref{markov_section}.  We consider the posterior expectation
\[\bb E_{\mr{post}}(G_{D_c}) = \frac{1}{T} \sum_{j=1}^{T} G_{D_c}(y_j'),\]
where $y_1', \ldots, y_T'$ are SGLD samples from the full-model posterior associated to the regret $G$ (and not to the direction-restricted regret $G_{D_c}$, for which SGLD sampling would collapse onto a degenerate policy that always moves in direction $D_c$).  Define the asymmetry statistic
\[A(\wstar) = \mr{std}_{D_c} \bb E_{\mr{post}}(G_{D_c}),\]
the standard deviation across cardinal directions.  $A$ is a first-moment measure of phase 3 asymmetry, complementary to the second-moment signal provided by susceptibilities.

On a per-seed level, we find a strong correspondence between $\ol \chi_{D_c}$ and $\bb E_{\mr{post}}(G_{D_c})$ early in phase 3, with $\ol \rho = +0.77$ $[+0.67, +0.86]$ ($n=20$) for $\alpha = 1$ and $\ol \rho = +0.87$ $[+0.81, +0.92]$ ($n=30$) for $\alpha = 0.6$.  Both have $p < 10^{-11}$, being additional evidence for the directions in susceptibilities corresponding to real asymmetries in the model. We cover correlations between $A(w^*)$ and other measures below in \cref{within_optimal_section}.

\subsection{Coordinated Simplification Within Phase 3} \label{within_optimal_section}

We test the claim of \cref{three_stage_section} that during phase 3 the parameterization progressively simplifies, by exhibiting a coordinated decline across four quantities that probe the regret landscape from distinct angles:
\begin{enumerate}
    \item The \emph{full-model LLC estimate} $\widehat \lambda(\wstar)$, defined in \cref{LLC_section}.
    \item The \emph{susceptibility cluster centroid distance}: the distance between the centroids of the two susceptibility clusters identified at phase 3 onset in \cref{branching_logic_section}, normalized by the pooled standard deviation of the susceptibilities within clusters.  This directly tracks the asymmetric branching hypothesized in \cref{branching_logic_section}.
    \item The \emph{Hessian trace} $\mr{tr}(\nabla^2 G(\wstar))$ measures the mean curvature of the regret landscape at the approximate local minimum.  We estimate it using Hutchinson's method: for random vectors $v$ with $\bb E[vv^T] = I$ one has $\mr{tr}(\nabla^2 G(\wstar)) = \bb E_v[v^T \nabla^2 G(\wstar) v]$, and each sample $v^T \nabla^2 G(\wstar) v$ may be computed without calculating $\nabla^2 G(\wstar)$ explicitly.  See \cref{hessian_appendix} for further details and estimation hyperparameter selection.
    \item The \emph{direction-conditioned posterior regret asymmetry} $A(\wstar) = \mr{std}_{D_c} \bb E_{\mr{post}}(G_{D_c})$ introduced in \cref{direction_regret_section}.
\end{enumerate}
These quantities fall into two pairs.  The LLC estimate and the Hessian trace are scalar complexity measures at the local minimum; their joint decline reflects an overall simplification of the parameterization.  The cluster centroid distance and the asymmetry statistic $A(\wstar)$ depend on a choice of partition of state space; their joint decline reflects a simultaneous rebalancing of the parameterization towards greater symmetry across direction classes.

The coordinated decline is present across the $\alpha = 0.6$, $\alpha = 0.5$ and $\alpha = 1$ hyperparameter regimes.  The $\alpha = 1$ case in particular is worth noting in the context of \cref{branching_logic_section}.  The fact that the simplification is still measurable across all complexity measures tells us that we are not only measuring the simplification of structures developed during intermediate phases, but also the reduction of the spontaneous asymmetry inherited from initialization as discussed in \cref{init_sensitivity_section}.

In \cref{fig:phase3_metrics} we show how the four metrics, together with the raw cluster distance, evolve over phase 3 for a single representative training run with $\alpha = 0.6$.  All metrics decline between the start and end of phase 3, though most are not monotonically decreasing.  To evaluate the degree to which the metrics correlate with one another and with the checkpoint index, we perform a Spearman rank correlation analysis between each pair of metrics across the ten-run $\alpha = 0.6$ population of \cref{three_stage_section} (data for these runs is collected in \cref{extra_plots_appendix}), reporting the resulting $\ol \rho$ values in \cref{tab:phase3_metrics}.  As expected from the figure, the LLC, raw cluster distance, and checkpoint index are all strongly correlated with each other ($|\ol \rho| > 0.9$).  The Hessian trace and the normalized cluster distance have noisier correlations with the other metrics ($|\ol \rho| \approx 0.5$), and the correlation between the Hessian trace and the normalized cluster distance is not significant. The direction-conditioned posterior regret asymmetry lands somewhere in-between, with significant $|\ol\rho|$ in the 0.5-0.8 range, with the exception of the correlation with the Hessian trace, which is insignificant. This is consistent with our hypothesis that the algorithm is being simplified throughout phase 3 by erasing the branching logic, with the Hessian trace being the most noisy measures of this simplification, leading to non-significant correlations with other noisy measures like normalized cluster distance and direction-conditioned posterior regret asymmetry. 

\begin{figure}
    \centering
    \includegraphics[width=0.99\linewidth]{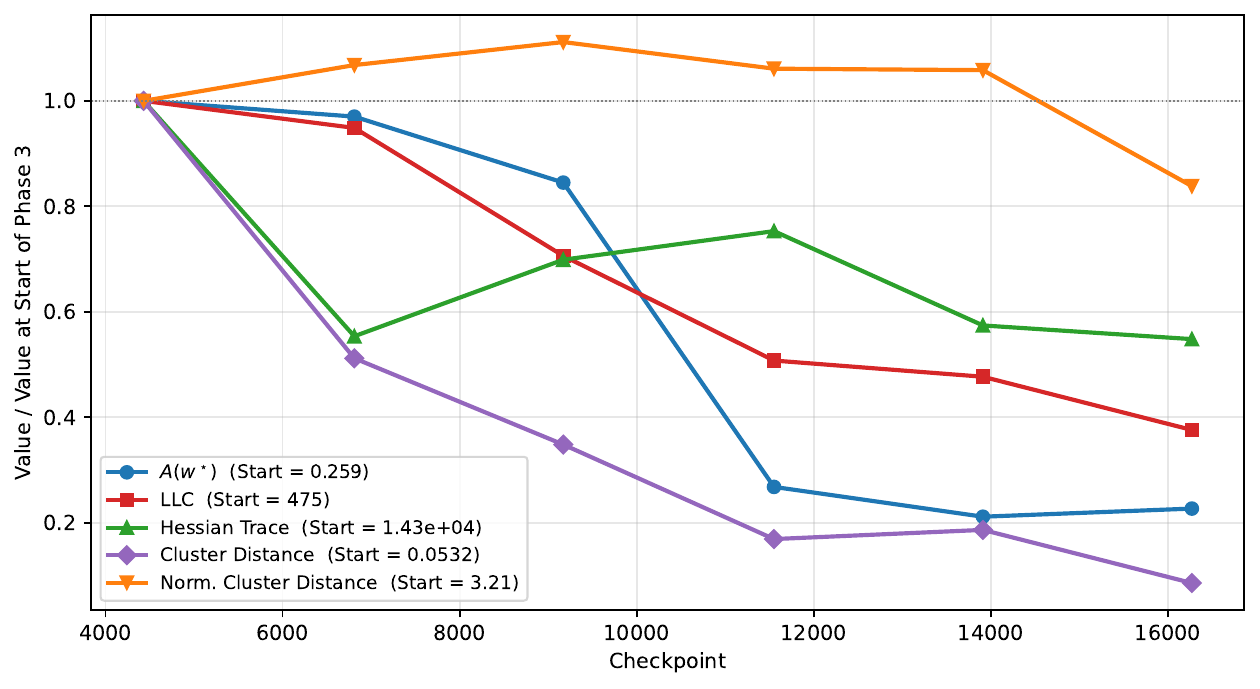}
    \caption{Comparison of the behavior of the four metrics and the unnormalized cluster distance over the course of phase 3 for the representative model trained with $\alpha = 0.6$ depicted in \cref{fig:alpha_0.6_single_run}.}
    \label{fig:phase3_metrics}
\end{figure}

\begin{table}[h]
    \centering
\resizebox{\textwidth}{!}{%\
\begin{tabular}{l|ccccc}
\toprule
 & $\widehat \lambda$ & $\widehat{\mr{Tr}(H)}$ & $\delta$ & $\delta/\sigma$ & $A(w^\star)$ \\
\midrule
Checkpoint & $-0.99\,[-1.00, -0.97]$ & $-0.49\,[-0.74, -0.19]$ & $-0.96\,[-0.98, -0.94]$ & $-0.52\,[-0.75, -0.22]$ & $-0.79\,[-0.89, -0.66]$ \\
$\widehat \lambda$ &  & $+0.49\,[+0.17, +0.75]$ & $+0.95\,[+0.93, +0.98]$ & $+0.54\,[+0.25, +0.77]$ & $+0.77\,[+0.64, +0.87]$ \\
$\widehat{\mr{Tr}(H)}$ &  &  & $+0.44\,[+0.14, +0.70]$ & $+0.23\,[-0.08, +0.52]$ & $+0.27\,[-0.05, +0.57]$ \\
$\delta$ &  &  &  & $+0.59\,[+0.32, +0.80]$ & $+0.75\,[+0.60, +0.87]$ \\
$\delta/\sigma$ &  &  &  &  & $+0.51\,[+0.35, +0.68]$ \\
\bottomrule
\end{tabular}%
}
    \caption{Mean Spearman rank correlation $\ol \rho$ between pairs of quantities varying within phase 3, computed across ten $\alpha = 0.6$ models.  The four quantities are $\widehat \lambda$ (the LLC estimate), $\widehat{\mr{Tr}(H)}$ (the Hessian trace estimate), $\delta$ (the susceptibility cluster centroid distance of \cref{branching_logic_section}), and $\delta/\sigma$ ($\delta$ normalized by the susceptibility standard deviation) and $A(w^*)$ (the direction-conditioned posterior regret asymmetry).  Each cell reports the mean correlation followed by the bracketed $95\%$ bootstrap confidence interval.  The first row gives the correlation of each quantity with checkpoint index.}
    \label{tab:phase3_metrics}
\end{table}

\begin{remark}
This reduction-of-asymmetry story is complicated by an independent observation: the declining complexity measures reported above, including those that specifically track parameterization asymmetry, do not predict a corresponding decrease in the \emph{activation steering asymmetry} of \cref{activation_steering_section}.  Three observations converge on this tension:
\begin{enumerate}
    \item The susceptibility-based criterion of \cref{activation_steering_section} for detecting activation-steering-distinguished directions becomes unreliable by the end of phase 3, coincident with the decline of the cluster centroid distance.  In the $\alpha = 1$ models, for example, the criterion identifies distinguished directions in only $2$ of $101$ seeds at the end of training compared with $80$ of $101$ seeds at the reference checkpoint early in phase 3.
    \item The activation steering asymmetry of \cref{activation_steering_section} nevertheless persists: directions distinguished early in phase 3 continue to exhibit higher steering resistance than the remaining directions at the end of training.
    \item Cross-checkpoint steering predictions are stronger than same-checkpoint ones: using directions distinguished by the susceptibility criterion early in phase 3 to predict steering resistance at the end of phase 3 yields a larger effect size than using the directions distinguished at the end of phase 3 itself.  For the $\alpha = 1$ models at the FF 2 layer, the cross-checkpoint test gives $\mr{mean}(X)=+0.127$ $[+0.099, +0.157]$, $t = 8.44$ ($n = 80$), compared with $\mr{mean}(X)=+0.076$ $[+0.054, +0.098]$, $t = 6.67$ ($n = 79$) for the same-checkpoint test early in phase 3 and $\mr{mean}(X)=+0.024$ $[-0.005, +0.054]$, $t = 1.57$ ($n = 67$, non-significant) for the same-checkpoint test at the end of phase 3 (using the percentile cluster detection criterion with a threshold relaxed from P99/P1 to P64/P36 to retain enough seeds, see \cref{cluster_criteria_appendix}). 
\end{enumerate}
Two of these measures of parameterization asymmetry -- the susceptibility cluster structure and the direction-conditioned posterior regret -- decline over phase 3, but the activation-steering asymmetry does not.  This suggests that \emph{parameterization asymmetry} is a more multi-dimensional feature than the susceptibility analysis alone would indicate: activation steering probes a component of it that does not decrease within the optimal phase.  We discuss this tension further in \cref{limitations_section}.
\end{remark}

\section{Discussion} \label{discussion_section}

In this section we step back and discuss what the susceptibility framework does and does not let us conclude, both about the toy environment that has been the subject of our analysis and in other more general contexts.  We separate methodological limitations of susceptibility-based interpretability that we expect to apply beyond the present setting (\cref{limitations_general_section}) from observations specific to the environment under consideration in our experiments that our analysis left unresolved (\cref{limitations_section}); we then list outstanding questions we did not pursue (\cref{outstanding_questions_section}), and close with implications for applications to interpretability for RLHF (\cref{rlhf_section}).

\subsection{Limitations of our Approach} \label{limitations_general_section}
First, let us discuss some general limitations associated with susceptibility estimation as an interpretability tool in reinforcement learning.
\begin{enumerate}
    \item \emph{Applicability of generalized Bayesian inference.}  The approach developed in \citet{RL1} presumes that policy development is mediated by generalized Bayesian inference for the regret function.  In practice one does not learn in this way: the training trajectories are in reality generated by reinforcement learning algorithms such as -- in our case -- policy gradient methods.  One should then ask whether the theoretical underpinnings of susceptibilities are applicable.  This is an inherited limitation: SLT-based interpretability in supervised learning faces the same Bayes-vs-SGD gap (see e.g.\ \citep{HitchcockHoogland} for a discussion), with an empirical track record that we extend rather than resolve here.  For some discussion of this point we refer to \citep[Section 7.3]{RL1}.
    \item \emph{Assumptions on the environment.}  Even if one accepts the generalized Bayesian context as informative the key results of singular learning theory rely on a technical condition on the transition functions of the Markov model.  One requires that no reachable trajectory obtains a return exceeding the optimal expected return from its initial state (\cref{key_assumption_appendix}).  This holds when the transition function is deterministic and the model contains an optimal policy, as is the case in our environment here.  Nonetheless reinforcement learning problems violating this condition -- which includes most stochastic environments of practical interest -- fall outside the present framework and would require a different theoretical treatment that we do not provide.
    \item \emph{Generalization beyond toy environments.}  This leads to a natural objection.  All of our experiments are on the cheese-in-the-corner environment of \citet{RL1}, chosen to exhibit a clear stagewise phase development structure.  Whether the developmental phenomena and the cluster structure we observe generalize to richer RL settings (for instance those with continuous or much larger state spaces, partial observability) is an empirical question we don't address.  Some of the methodological choices we have made, for instance the ability to evaluate regret exactly rather than estimate it through on- or off-policy sampling, will not apply in more complex settings.
    \item \emph{Choice of observables and perturbations.}  A susceptibility measures the response of the posterior expectation of a specified observable (e.g. loss on a component) to a specified family of distributional perturbations.  Our interpretability claims are made relative to the family of initial state likelihood perturbations chosen here; a different choice, for instance of expected rewards, would probe a different aspect of parameter geometry.  Even the choice to study per-initial-state perturbations was only possible robustly in our environment because the number of states is small enough that we may consider all such perturbations. This will not be true more generally; one would need, for instance, to consider samples from a choice of distribution over initial states.
    \item \emph{Sensitivity of SGLD-based estimation.}  The susceptibility estimator inherits the hyperparameter sensitivities of SGLD-based LLC estimation.  We use the hyperparameter regime developed using the approach of \cite[Appendix A]{RL1}, optimized for well-behaved LLC estimation, but we have not systematically explored the robustness of our cluster identifications to these choices.
    \item \emph{Computational cost.}  Estimating a susceptibility matrix at a single checkpoint requires independent SGLD samples for each perturbation; populations of seeds and checkpoints multiply this cost.  The dominant per-sample cost in the RL setting is the regret estimation rather than the SGLD update itself, which contrasts with susceptibility estimation in language models, where covariances are taken against the per-token loss and each sample is essentially a forward pass.  Recent work in this latter setting has scaled susceptibility estimation to LLMs at the billion-parameter scale \citep{SpectroscopyAtScale}; whether comparable scaling is achievable for RL depends on how regret -- or a tractable surrogate -- is estimated.  We discuss this in \cref{rlhf_section} below.
\end{enumerate}

\subsection{Unexplained Observations} \label{limitations_section}

In this section we discuss specific findings in the cheese-in-the-corner environment that our analysis left unresolved, including situations in which the susceptibility analysis and other evidence are in tension, and features in other measurements that susceptibilities did not detect.
\begin{enumerate}
    \item \emph{Residue of phase 2.}  The end-of-training skip-versus-visit susceptibility difference (\cref{phase2_residue_section}) is consistent with phase 3 duration being the proximate driver, with no separable phase-2-specific signal that we can detect.  This complicates the mixture-of-algorithms reading of \cref{branching_logic_section}: if the asymmetry visible at phase 3 onset reflected a genuinely persistent ``simpler subroutine'' inherited from phase 2, we would expect some signature beyond the timing-mediated effect we observe.  Two readings are compatible with the data: that the algorithmic asymmetry is real but is rapidly erased within phase 3, or that the cluster structure at phase 3 onset is a transient feature of a freshly-completed phase transition rather than a persistent structural property.  Our measurements do not distinguish them.
    
    Activation steering on the same population (\cref{steering_skip_visit_remark}) gives a finding with sign opposite to the naive residue prediction: $\{\mr{Right}, \mr{Down}\}$ states are \emph{harder} to steer in seeds that skipped phase 2 than in those that visited.
    
    \item \emph{Direction asymmetry beyond susceptibilities.}  The $\{\mr{Right}, \mr{Down}\}$ vs $\{\mr{Left}, \mr{Up}\}$ activation-steering asymmetry in $\alpha = 0.5$ seeds is retained even in seeds where the susceptibility criterion does not flag $\{\mr{R}, \mr{D}\}$ as distinguished (\cref{activation_steering_section}).  The susceptibility analysis only partially captures whatever feature of parameterization the steering test is detecting.
    
    \item \emph{Phase 3 progression.}  Our data does not distinguish three candidate accounts of phase 3 development: complexity equalization between internal subroutines, reorganization of the parameterization without complexity change, and proportional decline of all measures below the detection threshold.  Relatedly, the four within-phase-3 quantities of \cref{within_optimal_section} all decline but they are not tightly coupled, and we do not have a framework that predicts which of them should covary.
\end{enumerate}
Several of these questions point in the same direction: toward the nature of the parameterization asymmetry probed by activation steering but not resolved by the susceptibility analysis.  In \cref{within_optimal_section} we posited that the asymmetric representation of the model's behavior in different classes of states comprises multiple independent dimensions that do not always covary, but our techniques are not sufficient to test this supposition.

\subsection{Outstanding Questions} \label{outstanding_questions_section}
We now present questions that we feel are natural but that we have not addressed yet in the current work.
\begin{enumerate}
    \item \emph{Other structural features in susceptibilities.}  The susceptibility data across phases contain structure we did not pursue systematically: for instance we frequently see susceptibilities forming linear ``streaks'' like those of \cref{antenna_section} but associated to other cheese locations: there is a structural relationship between the similarity of states and the susceptibilities whose features we did not address.  We also don't address the meaning of the left and right singular vectors of susceptibility matrices.  For instance we consistently see over 90\% of the variance explained by the first principal component, but the significance of the first right singular vector in state space remains unclear.
    \item \emph{Relationship to existing interpretability tools.}  How do the parameter-space clusters identified by susceptibility analysis relate to features surfaced by activation-space tools such as sparse autoencoders \citep{Bricken2023Monosemanticity, Huben2024SAE, Templeton2024Scaling} (after the sparse-dictionary approach of \citet{Arora2018Polysemy}) or linear probes?
    \item \emph{Patterning interventions.}  Following the patterning paradigm of \citet{Patterning}, susceptibilities can be used to design data-side interventions intended to reshape training.  We explored this informally for cheese-in-the-corner by computing a phase 3 cluster-separation direction in observable space, applying the susceptibility pseudo-inverse to obtain the corresponding optimal perturbation of the initial-state distribution, and training models with the perturbed distribution, but did not obtain clear enough results to pursue more systematically in the present work.  Whether this style of intervention can systematically reshape phase 3 development either in our environment or more generally remains an open question.
\end{enumerate}

\subsection{Susceptibilities for RLHF} \label{rlhf_section}
Susceptibilities have been used for LLC interpretability using techniques developed in the supervised setting \citep{Susceptibilities, Embryology, Spectroscopy} and scaled to the billion-parameter regime \citep{SpectroscopyAtScale}.  The approach taken therein studies cluster structure in weight-restriction space interpretable in terms of features of the data distribution.  RLHF post-training of large language models is a contextual-bandit reinforcement learning problem -- a prompt is sampled, the policy emits a response, a reward model returns a scalar signal used to update the policy -- and the susceptibility framework we develop in this paper applies to the optimization trajectory of the parameterized policy under perturbations of either the prompt distribution or the per-prompt rewards output by the model.

The alignment relevance follows from the loss-complexity tradeoff that drives internal model selection in singular learning theory, the implications of which for alignment are discussed in \citet{YAWYE} and \citet[\S 7.1]{RL1}.

In a regime where many generalizations are compatible with the training data, the (generalized) Bayesian posterior may favor a simpler, less accurate solution over a more accurate but more complex one.  Reward models in RLHF are trained on preference datasets that are small relative to pre-training corpora, which is precisely the regime in which this tradeoff is most salient.  The tendency of reward models to latch onto simple features such as response length and list formatting fits this account, as do higher-level behaviors such as sycophancy and reward hacking.  At the simplest level, susceptibility analysis applied to a reward model offers the same kind of cluster-based interpretability we have used here: a diagnostic for which features of preference data the optimization process treats in a unified way.  Going further, the patterning paradigm of \citet{Patterning} uses susceptibility-identified directions as targets for data-side interventions; transferring this to RLHF turns the question into one of which observables of the reward model one should steer.  Designing such observables is an important parallel ingredient that we leave to future work.

In the language-model setting, susceptibilities are estimated as covariances against the per-token loss $\ell_{xy}(w) = -\log p(y | x, w)$ -- where $x,y$ are input and output variables respectively -- costing a forward pass per SGLD sample.  On the other hand, in the RL setting they are covariances against the regret $G(w)$, whose estimation requires policy rollouts and substantially increases the per-sample cost.  Scaling RLHF susceptibilities will therefore require cheaper regret surrogates -- for instance learned value functions or per-token-loss-like reformulations -- together with empirical decisions about the on-distribution versus off-distribution choice analogous to the one we examine in \cref{on_off_appendix}.  

\subsection{Conclusions} \label{conclusions_section}

We have seen empirically that susceptibilities reveal developmentally dependent structures in the parameterization of deep RL agents (asymmetric parameterizations, earlier-phase-dependent structure, and structural development within the optimal phase) that are invisible from analysis of the policy itself.  The independent agreement of multiple measurement tools (including activation-steering experiments and direction-conditioned posterior regrets) with the structures identified by susceptibilities establishes that this parameter-space structure is functionally meaningful and not an artifact of the estimation process.

\begin{ack}
This project is funded by the Advanced Research + Invention Agency (ARIA).
\end{ack}

\bibliographystyle{abbrvnat}
\bibliography{RL.bib}

@misc{RL1,
      title={Stagewise Reinforcement Learning and the Geometry of the Regret Landscape},
      author={Chris Elliott and Einar Urdshals and David Quarel and Matthew Farrugia-Roberts and Daniel Murfet},
      year={2026},
      eprint={2601.07524},
      archivePrefix={arXiv},
      primaryClass={cs.LG},
      url={https://arxiv.org/abs/2601.07524},
}

@misc{massenkoffmccrory2026labor,
 author = {Maxim Massenkoff and Peter McCrory},
 title = {Labor market impacts of AI: A new measure and early evidence},
 date = {2026-03-05},
 year = {2026},
 url = {https://www.anthropic.com/research/labor-market-impacts},
}

@article{Dayan,
  title={Improving Generalization for Temporal Difference Learning: The Successor Representation},
  author={Dayan, Peter},
  journal={Neural Computation},
  volume={5},
  number={4},
  pages={613--624},
  year={1993},
  publisher={MIT Press}
}

@misc{mini2023understanding,
      title={Understanding and Controlling a Maze-Solving Policy Network},
      author={Ulisse Mini and Peli Grietzer and Mrinank Sharma and Austin Meek and Monte MacDiarmid and Alexander Matt Turner},
      year={2023},
      eprint={2310.08043},
      archivePrefix={arXiv},
      primaryClass={cs.LG},
      url={https://arxiv.org/abs/2310.08043},
}

@book{WatanabeGrey,
  title={Algebraic geometry and statistical learning theory},
  author={Watanabe, Sumio},
  volume={25},
  year={2009},
  series={Cambridge Monographs on Applied and Computational Mathematics},
  publisher={Cambridge University Press}
}

@book{WatanabeGreen,
  title={Mathematical theory of {B}ayesian statistics},
  author={Watanabe, Sumio},
  year={2018},
  publisher={Chapman and Hall/CRC}
}

@article{WatanabeWBIC,
  title={A widely applicable {B}ayesian information criterion},
  author={Watanabe, Sumio},
  journal={The Journal of Machine Learning Research},
  volume={14},
  number={1},
  pages={867--897},
  year={2013},
  publisher={JMLR. org}
}

@inproceedings{LLC,
title={The Local Learning Coefficient: A Singularity-Aware Complexity Measure},
author={Edmund Lau and Zach Furman and George Wang and Daniel Murfet and Susan Wei},
booktitle={The 28th International Conference on Artificial Intelligence and Statistics},
year={2025},
url={https://openreview.net/forum?id=1av51ZlsuL}
}

@inproceedings{Wang,
title={Differentiation and Specialization of Attention Heads via the Refined Local Learning Coefficient},
author={George Wang and Jesse Hoogland and Stan van Wingerden and Zach Furman and Daniel Murfet},
booktitle={The Thirteenth International Conference on Learning Representations},
year={2025},
url={https://openreview.net/forum?id=SUc1UOWndp}
}

@inproceedings{Misgen,
  title={Mitigating Goal Misgeneralization via Minimax Regret},
  year = {2025},
  author={Abdel Sadek, Karim
          and Farrugia-Roberts, Matthew
          and Erlebach, Hannah
          and de Witt, Christian Schroeder
          and Krueger, David
          and Anwar, Usman
          and Dennis, Michael D},
  booktitle={Reinforcement Learning Conference}
}

@inproceedings{WellingTeh,
  title={Bayesian learning via stochastic gradient {L}angevin dynamics},
  author={Welling, Max and Teh, Yee W},
  booktitle={Proceedings of the 28th international conference on machine learning (ICML-11)},
  pages={681--688},
  year={2011}
}

@article{HitchcockHoogland,
  title={From Global to Local: A Scalable Benchmark for Local Posterior Sampling},
  author={Hitchcock, Rohan and Hoogland, Jesse},
 journal = {arXiv preprint 2507.21449},
  year={2025}
}

@misc{Susceptibilities,
      title={Structural Inference: Interpreting Small Language Models with Susceptibilities},
      author={Garrett Baker and George Wang and Jesse Hoogland and Daniel Murfet},
      year={2025},
      eprint={2504.18274},
      archivePrefix={arXiv},
      primaryClass={cs.LG},
      url={https://arxiv.org/abs/2504.18274},
}

@misc{Embryology,
  title         = {Embryology of a Language Model},
  author        = {George Wang and Garrett Baker and Andrew Gordon and Daniel Murfet},
  year          = {2025},
  eprint        = {2508.00331},
  archivePrefix = {arXiv},
  primaryClass  = {cs.LG},
  url           = {https://arxiv.org/abs/2508.00331},
}

@misc{Spectroscopy,
  title         = {Towards Spectroscopy: Susceptibility Clusters in Language Models},
  author        = {Andrew Gordon and Garrett Baker and George Wang and William Snell and Stan van Wingerden and Daniel Murfet},
  year          = {2026},
  eprint        = {2601.12703},
  archivePrefix = {arXiv},
  primaryClass  = {cs.LG},
  url           = {https://arxiv.org/abs/2601.12703},
}

@misc{SpectroscopyAtScale,
  title         = {Spectroscopy at Scale: Finding Interpretable Structure in Pythia-1.4B},
  author        = {Daniel Murfet and Andrew Gordon and Max Adam and George Wang and Jesse Hoogland and Garrett Baker and William Snell and Stan van Wingerden and Adam Newgas and Billy Snikkers and Rohan Hitchcock},
  year          = {2026},
  howpublished  = {\url{https://timaeus.co/research/2026-04-21-spectroscopy-main}},
  url           = {https://timaeus.co/research/2026-04-21-spectroscopy-main},
}

@misc{Patterning,
  title         = {Patterning is Dual to Interpretability: Shaping Neural Network Development with Susceptibilities},
  author        = {George Wang and Daniel Murfet},
  year          = {2026},
  eprint        = {2601.13548},
  archivePrefix = {arXiv},
  primaryClass  = {cs.LG},
  url           = {https://arxiv.org/abs/2601.13548},
}

@misc{YAWYE,
      title={You Are What You Eat--{AI} Alignment Requires Understanding How Data Shapes Structure and Generalisation},
      author={Simon Pepin Lehalleur and Jesse Hoogland and Matthew Farrugia-Roberts and Susan Wei and Alexander Gietelink Oldenziel and George Wang and Liam Carroll and Daniel Murfet},
      year={2025},
      eprint={2502.05475},
      archivePrefix={arXiv},
      primaryClass={cs.LG},
      url={https://arxiv.org/abs/2502.05475},
}

@article{BIF,
  title={Bayesian Influence Functions for {H}essian--Free Data Attribution},
  author={Philipp Alexander Kreer and Wilson Wu and Maxwell Adam and Zach Furman and Jesse Hoogland},
  journal={ArXiv},
  year={2025},
  volume={abs/2509.26544},
  url={https://api.semanticscholar.org/CorpusID:281681466}
}

@misc{IMPALA,
      title={IMPALA: Scalable Distributed Deep-RL with Importance Weighted Actor-Learner Architectures}, 
      author={Lasse Espeholt and Hubert Soyer and Remi Munos and Karen Simonyan and Volodymir Mnih and Tom Ward and Yotam Doron and Vlad Firoiu and Tim Harley and Iain Dunning and Shane Legg and Koray Kavukcuoglu},
      year={2018},
      eprint={1802.01561},
      archivePrefix={arXiv},
      primaryClass={cs.LG},
      url={https://arxiv.org/abs/1802.01561}, 
}

@article{
hoogland2024developmental,
title={Loss Landscape Degeneracy and Stagewise Development in Transformers},
author={Jesse Hoogland and George Wang and Matthew Farrugia-Roberts and Liam Carroll and Susan Wei and Daniel Murfet},
journal={Transactions on Machine Learning Research},
issn={2835-8856},
year={2025},
url={https://openreview.net/forum?id=45qJyBG8Oj},
note={}
}

@article{Li_Chen_Carlson_Carin_2016,
    title={Preconditioned Stochastic Gradient Langevin Dynamics for Deep Neural Networks},
    volume={30},
    url={https://ojs.aaai.org/index.php/AAAI/article/view/10200},
    DOI={10.1609/aaai.v30i1.10200},
    number={1},
    journal={Proceedings of the AAAI Conference on Artificial Intelligence},
    author={Li, Chunyuan and Chen, Changyou and Carlson, David and Carin, Lawrence},
    year={2016},
    month={Feb.}
}

@misc{Turner2023ActAdd,
  title         = {Steering Language Models With Activation Engineering},
  author        = {Alexander Matt Turner and Lisa Thiergart and Gavin Leech and David Udell and Juan J. Vazquez and Ulisse Mini and Monte MacDiarmid},
  year          = {2023},
  eprint        = {2308.10248},
  archivePrefix = {arXiv},
  primaryClass  = {cs.CL},
  url           = {https://arxiv.org/abs/2308.10248},
}

@inproceedings{Rimsky2024CAA,
  title     = {Steering {Llama} 2 via Contrastive Activation Addition},
  author    = {Rimsky, Nina and Gabrieli, Nick and Schulz, Julian and Tong, Meg and Hubinger, Evan and Turner, Alexander},
  booktitle = {Proceedings of the 62nd Annual Meeting of the Association for Computational Linguistics (Volume 1: Long Papers)},
  year      = {2024},
  address   = {Bangkok, Thailand},
  publisher = {Association for Computational Linguistics},
  pages     = {15504--15522},
  doi       = {10.18653/v1/2024.acl-long.828},
  url       = {https://aclanthology.org/2024.acl-long.828/},
  eprint    = {2312.06681},
  archivePrefix = {arXiv},
  primaryClass  = {cs.CL},
}

@misc{murfet2025programs,
  title         = {Programs as Singularities},
  author        = {Daniel Murfet and Will Troiani},
  year          = {2025},
  eprint        = {2504.08075},
  archivePrefix = {arXiv},
  primaryClass  = {cs.LO},
  url           = {https://arxiv.org/abs/2504.08075},
}

@article{McGrath2022Chess,
  title     = {Acquisition of chess knowledge in {AlphaZero}},
  author    = {McGrath, Thomas and Kapishnikov, Andrei and Toma{\v{s}}ev, Nenad and Pearce, Adam and Wattenberg, Martin and Hassabis, Demis and Kim, Been and Paquet, Ulrich and Kramnik, Vladimir},
  journal   = {Proceedings of the National Academy of Sciences},
  volume    = {119},
  number    = {47},
  pages     = {e2206625119},
  year      = {2022},
  publisher = {National Academy of Sciences},
  doi       = {10.1073/pnas.2206625119},
  url       = {https://www.pnas.org/doi/10.1073/pnas.2206625119},
  eprint    = {2111.09259},
  archivePrefix = {arXiv},
  primaryClass  = {cs.LG},
}

@misc{Bricken2023Monosemanticity,
  title        = {Towards Monosemanticity: Decomposing Language Models With Dictionary Learning},
  author       = {Trenton Bricken and Adly Templeton and Joshua Batson and Brian Chen and Adam Jermyn and Tom Conerly and Nick Turner and Cem Anil and Carson Denison and Amanda Askell and Robert Lasenby and Yifan Wu and Shauna Kravec and Nicholas Schiefer and Tim Maxwell and Nicholas Joseph and Zac Hatfield-Dodds and Alex Tamkin and Karina Nguyen and Brayden McLean and Josiah E. Burke and Tristan Hume and Shan Carter and Tom Henighan and Christopher Olah},
  year         = {2023},
  howpublished = {\url{https://transformer-circuits.pub/2023/monosemantic-features/index.html}},
  note         = {Transformer Circuits Thread},
}

@inproceedings{Huben2024SAE,
  title     = {Sparse Autoencoders Find Highly Interpretable Features in Language Models},
  author    = {Huben, Robert and Cunningham, Hoagy and Smith, Logan Riggs and Ewart, Aidan and Sharkey, Lee},
  booktitle = {The Twelfth International Conference on Learning Representations},
  year      = {2024},
  url       = {https://openreview.net/forum?id=F76bwRSLeK},
  eprint    = {2309.08600},
  archivePrefix = {arXiv},
  primaryClass  = {cs.LG},
}

@misc{Templeton2024Scaling,
  title        = {Scaling Monosemanticity: Extracting Interpretable Features from {Claude} 3 {Sonnet}},
  author       = {Adly Templeton and Tom Conerly and Jonathan Marcus and Jack Lindsey and Trenton Bricken and Brian Chen and Adam Pearce and Craig Citro and Emmanuel Ameisen and Andy Jones and Hoagy Cunningham and Nicholas L. Turner and Callum McDougall and Monte MacDiarmid and C. Daniel Freeman and Theodore R. Sumers and Edward Rees and Joshua Batson and Adam Jermyn and Shan Carter and Chris Olah and Tom Henighan},
  year         = {2024},
  howpublished = {\url{https://transformer-circuits.pub/2024/scaling-monosemanticity/index.html}},
  note         = {Transformer Circuits Thread},
}

@article{Arora2018Polysemy,
  title     = {Linear Algebraic Structure of Word Senses, with Applications to Polysemy},
  author    = {Arora, Sanjeev and Li, Yuanzhi and Liang, Yingyu and Ma, Tengyu and Risteski, Andrej},
  journal   = {Transactions of the Association for Computational Linguistics},
  volume    = {6},
  pages     = {483--495},
  year      = {2018},
  publisher = {MIT Press},
  doi       = {10.1162/tacl_a_00034},
  url       = {https://aclanthology.org/Q18-1034/},
  eprint    = {1601.03764},
  archivePrefix = {arXiv},
  primaryClass  = {cs.CL},
}

@article{Williams1992REINFORCE,
  title     = {Simple Statistical Gradient-Following Algorithms for Connectionist Reinforcement Learning},
  author    = {Williams, Ronald J.},
  journal   = {Machine Learning},
  volume    = {8},
  number    = {3--4},
  pages     = {229--256},
  year      = {1992},
}

@inproceedings{Kornblith2019CKA,
  title     = {Similarity of Neural Network Representations Revisited},
  author    = {Kornblith, Simon and Norouzi, Mohammad and Lee, Honglak and Hinton, Geoffrey},
  booktitle = {Proceedings of the 36th International Conference on Machine Learning},
  series    = {Proceedings of Machine Learning Research},
  volume    = {97},
  pages     = {3519--3529},
  year      = {2019},
}

\clearpage

\appendix

\section{The Local Learning Coefficient in Reinforcement Learning} \label{LLC_appendix}
In this appendix we review the local learning coefficient and its role in understanding the policy landscape of deep reinforcement learning models.  We follow the treatment in \citet{RL1}, which extends the singular learning theory of \citet{WatanabeGrey, WatanabeGreen} to the generalized Bayesian inference setting appropriate for reinforcement learning.

\subsection{The Generalized Posterior and Free Energy} \label{posterior_appendix}

Fix a finite Markov decision problem with notation as in \cref{markov_section}.  Given a prior distribution $\phi$ on $W$ and an inverse temperature $\beta > 0$, the \emph{generalized tempered posterior} is the probability distribution on $W$ defined by
\[\mu_{n,\beta}(U) = \frac{Z_{n,\beta}(U)}{Z_{n,\beta}(W)}, \qquad Z_{n,\beta}(U) = \int_U \exp(-n\beta G_n(w))\phi(w)\,\d w\]
for measurable subsets $U \sub W$.  We call $Z_{n,\beta}(U)$ the \emph{evidence} and $F_{n,\beta}(U) = -\log Z_{n,\beta}(U)$ the \emph{free energy}.

The key results require a fundamental condition on the Markov decision problem.

\begin{assumption} \label{key_assumption_appendix}
Assume that $g(\tau) \ge 0$ for every reachable trajectory $\tau$, that is, every trajectory that occurs with positive probability under some policy $\pi_w$, $w \in W$.
\end{assumption}

This holds, for example, whenever the transition function is deterministic and the model contains an optimal policy; see \citet[Assumption 4.1 and Proposition D.24]{RL1}.

\subsection{Main Results} \label{main_results_appendix}

The following extends Watanabe's free energy formula and widely applicable Bayesian information criterion \citep{WatanabeGrey, WatanabeWBIC} to the reinforcement learning setting.  For proofs we refer to \citep[Appendix D]{RL1}.

\begin{theorem}[\citealp{RL1}] \label{main_theorem_appendix}
Consider a Markov decision problem satisfying \cref{key_assumption_appendix}, and let the prior $\phi$ be positive and analytic on $W$.  Let $\wstar$ be a global minimum of $G$, or a local minimum of $G$ satisfying the relative finite variance condition of \citet[Assumption 4.2]{RL1}, let $U$ be a sufficiently small open neighborhood of $\wstar$, and let $\lambda = \mr{rlct}_U(G - G(\wstar))$ and $m = \mr{rlcm}_U(G - G(\wstar))$ denote the real log canonical threshold and multiplicity of $G - G(\wstar)$ on $U$.
\begin{enumerate}
 \item \emph{Free energy asymptotics.}  The evidence satisfies
 \[Z_{n,\beta}(U) = e^{-n\beta G_n(\wstar)}\bigl(C(\beta, \phi|_U)\, n^{-\lambda}(\log n)^{m-1} + o_{\mr P}(n^{-\lambda}(\log n)^{m-1})\bigr)\]
 where $C(\beta, \phi|_U)$ is constant in $n$ when $\wstar$ is a global minimum, and in general a positive random variable that does not depend on $n$.
 \item \emph{WBIC.}  Let $\beta \colon \bb N \to \RR_{>0}$ satisfy $\beta(n) \to 0$ and $n\beta(n) \to \infty$ as $n \to \infty$, for instance $\beta(n) = \beta_0/\log n$.  Then
 \[n\beta\, \bb E_{n,\beta}(G_n(w) - G_n(\wstar)\,|\, U) \to \lambda\]
 in probability as $n \to \infty$, where $\bb E_{n,\beta}(-|U)$ denotes the expectation with respect to $\mu_{n,\beta}$ restricted to $U$.
\end{enumerate}
\end{theorem}

\subsubsection{The Local Learning Coefficient}

Given a local minimum $\wstar$ of the regret, the \emph{local learning coefficient} (LLC) is $\lambda(\wstar) = \lim_{\eps \to 0} \mr{rlct}_{B_\eps(\wstar)}(G - G(\wstar))$.  Part (2) of \cref{main_theorem_appendix} provides the key asymptotic relationship
\begin{equation} \label{LLC_identity}
\lambda(\wstar) \sim  n\beta\left(\bb E_{n,\beta}(G_n(w)|_U) - G_n(\wstar)\right)
\end{equation}
as $n \to \infty$, for $U$ a sufficiently small neighborhood of $\wstar$.  We use this fact to estimate the LLC.

We may generalize this to study asymptotic volumes in \emph{subspaces} of the parameter space.  Given a decomposition $W = V \times C$ of the parameter space into a \emph{component} $C$ and its complement $V$, the \emph{weight-refined LLC} (rLLC) at $\wstar = (v^*, c^*)$ is defined analogously but restricting to variations in $C$ alone.  Concretely, define the observable
\begin{equation} \label{weight_restricted_observable_appendix}
\OO_C(v,c) = \delta(v - v^*)(G(w) - G(\wstar)).
\end{equation}
Then $n\beta \langle \OO_C \rangle_{n,\beta}$ computes $\lambda_C(\wstar)$ to leading order by the same argument as \eqref{LLC_identity} applied to the restriction of $G$ to the slice $\{v^*\} \times C$.

\subsubsection{Stagewise Development}

The free energy formula (part (1) of \cref{main_theorem_appendix}) implies that the concentration of the posterior in a neighborhood $U$ of a local minimum $\wstar$ is controlled to leading order by two terms:
\[F_{n,\beta}(U) = n\beta\, G_n(\wstar) + \lambda \log n + o_{\mr P}(\log n).\]
Comparing two local minima $w_1, w_2$ with regrets $G(w_1) > G(w_2)$ and LLCs $\lambda_1 < \lambda_2$, the difference in free energies is
\[F_{n,\beta}(U_1) - F_{n,\beta}(U_2) = n\beta\, \delta G + \delta\lambda \log n + o_{\mr P}(\log n)\]
where $\delta G = G(w_1) - G(w_2) > 0$ and $\delta \lambda = \lambda_1 - \lambda_2 < 0$.  For small $n$ the $\delta\lambda \log n$ term dominates and the simpler solution $w_1$ is preferred; for large $n$ the $n\beta\,\delta G$ term dominates and the lower-regret solution $w_2$ is preferred.  The crossover defines a critical data set size $n^*$ at which the posterior shifts concentration, a \emph{Bayesian phase transition} \cite[\S 7.6]{WatanabeGrey}.

This predicts that deep RL should proceed stagewise from simple, high-regret policies to complex, low-regret policies.  This prediction was verified empirically in \citet{RL1} in the cheese-in-the-corner environment studied in this paper; see \cref{interpretation_section} for a summary.

\subsection{LLC Estimation via SGLD} \label{SGLD_estimation_appendix}

In practice we estimate the LLC and rLLC using stochastic gradient Langevin dynamics \citep[SGLD;][]{WellingTeh} to approximately sample from the posterior, following the approach developed in \citet{LLC} and applied to the RL setting in \citet[\S E]{RL1}.

Fix a local minimum $\wstar$ and choose a Gaussian prior $\phi_{\sigma^2}$ centered at $\wstar$ with variance $\sigma^2$.  The SGLD chain $y_0, y_1, \ldots, y_T$ is initialized at $y_0 = \wstar$ and updated by
\[y_{j+1} = y_j - \frac{\eps}{2}\,\nabla_w\left(n\beta\, G(w) + \frac{1}{2\sigma^2}\|w - \wstar\|_2^2\right)\bigg|_{w=y_j} + \sqrt{\eps}\,\eta_j\]
where $\eps > 0$ is the step size and $\eta_j \sim \mc N(0, I)$.  In practice the gradient $\nabla_w G$ is replaced by a stochastic estimate over a minibatch of trajectories.  After a burn-in period of $B$ steps the \emph{SGLD LLC estimator} is
\[\widehat\lambda(\wstar) = \frac{n\beta}{T-B}\sum_{j=B+1}^T (G(y_j) - G(\wstar)).\]
The weight-refined estimator $\widehat\lambda_C(\wstar)$ is defined identically but restricting the SGLD chain to variations in the component $C$ (holding the parameters in $V$ fixed at $v^*$).  We refer to \cite[\S E]{RL1} for a discussion of the hyperparameters and convergence properties of this estimator.

\section{Theoretical Aspects of Susceptibilities and Susceptibility Estimation} \label{susc_theory_appendix}
In this appendix we will discuss the theoretical foundations of susceptibilities and their estimators in more detail, presenting a more general and abstract perspective that encompasses the presentation in the main text.

\subsection{Perturbations of the Markov Problem}

The regret function $G$ of a Markov decision problem depends on three pieces of data beyond the parameterized family of policies: the initial state distribution $\Lambda$, the reward function $r$ and the transition function $p$.  A perturbation of any of these determines a deformation of $G$, and correspondingly a susceptibility.  We describe the three natural classes of perturbation as tangent vectors to the appropriate function spaces.

\begin{enumerate}
 \item \emph{Initial state perturbations}: choose a tangent vector $\xi \in T_\Lambda \Delta(\mc S)$ to the space of probability distributions on $\mc S$ at the point $\Lambda$.

 \item \emph{Reward perturbations}: choose a tangent vector $\rho \in T_r C^\omega(\mc S \times \mc A)$ to the space of reward functions at $r$.

 \item \emph{Transition function perturbations}: choose a tangent vector $\eta \in T_p C^\omega(\mc S \times \mc A, \Delta(\mc S \times \mc O \times \RR))$ to the space of transition functions at $p$.
\end{enumerate}

In each case the tangent vector determines, via the associated flow of smooth functions, a one-parameter family of deformed regret functions $G^h$ for $h$ in a neighborhood of zero, and correspondingly a family $G_n^h$ of importance-weighted estimators.

\begin{remark} \label{finite_tangent_remark}
 In the finite Markov problem setting, these tangent spaces are finite-dimensional real vector spaces and the tangent vector formulation reduces to ordinary calculus.  For instance, $T_\Lambda \Delta(\mc S)$ is the subspace of $\RR^{|\mc S|}$ consisting of functions summing to zero, and a tangent vector $\xi \in T_\Lambda \Delta(\mc S)$ satisfying $\xi(s) < 0$ only if $\Lambda(s) > 0$ determines a family of probability distributions $\Lambda^h(s) = \Lambda(s) + h\xi(s)$ for $0 \le h \le h_0$.  This recovers the definition of linear deformations given in the main text.  Likewise $T_r C^\omega(\mc S \times \mc A) \iso \RR^{|\mc S| \times |\mc A|}$, and a tangent vector $\rho$ determines a family of reward functions $r^h(s,a) = r(s,a) + h\rho(s,a)$.
\end{remark}

\begin{remark}
 The expected return $R(w)$ is linear in the initial state distribution $\Lambda$ and in the reward function $r$, but depends polynomially on the transition function $p$.  With the baseline $V^*(s_0)$ the regret $G$ is linear in $\Lambda$ and, up to a constant that does not depend on $w$, linear in $r$.  As a result, perturbations of the first two types lead to the simplest analysis.  Furthermore one may be interested in studying examples in which one does not have access to the transition function, making perturbations of the third type harder to study empirically.
\end{remark}

\subsection{Susceptibilities}
We now define susceptibilities in the reinforcement learning context, generalizing the construction reviewed in the main text for supervised learning.  Fix $n\beta > 0$ and a prior distribution $\phi$ on $W$.  Consider the annealed posterior
\[\mu_{n\beta}(w) = \frac{1}{Z_{n\beta}} \exp(-n\beta G(w)) \phi(w)\]
where $Z_{n\beta} = \int_W \exp(-n\beta G(w)) \phi(w) \d w$.  Given a measurable subset $U \sub W$ and a function $\OO \colon W \to \RR$ (an \emph{observable}) we write
\[\langle \OO \rangle_{n\beta, U} = \frac{1}{Z_{n\beta}(U)} \int_U \OO(w) \exp(-n\beta G(w)) \phi(w) \d w\]
for its posterior expectation value restricted to $U$.

\begin{definition} \label{susceptibility_def}
Let $\xi$ denote one of the perturbations described in the previous subsection, with associated deformed regret $G^h$.  Fix an observable $\OO$ and a measurable subset $U \sub W$.  The \emph{susceptibility} associated to $\xi$, $\OO$ and $U$ is
\[\chi_{\OO,\xi,U} = \frac{1}{n\beta} \nabla_\xi \langle \OO \rangle_{n\beta, U}\]
where $\nabla_\xi$ denotes the directional derivative in the direction of the deformation $\xi$.
\end{definition}

\begin{remark}
More generally one may replace $\OO$ by a tempered distribution, i.e. an element of the continuous linear dual to the space of smooth functions on $W$.  In practice we will consider observables that arise as the product of an analytic function by the delta distribution on a linear subspace of a vector space; this is indeed distributional but easy to define by restriction of the domain of integration without needing the full machinery of distributions.
\end{remark}

The key result that makes susceptibilities practically computable is the following covariance formula, which generalizes \cite[Lemma 2.1]{Susceptibilities} to the reinforcement learning setting.

\begin{prop} \label{susc_covariance_prop}
For an affine deformation $G^h = G + h(G^1 - G)$ associated to $\xi$, if the observable $\OO$ does not depend on the perturbation parameter $h$ then
\[\chi_{\OO,\xi,U} = \mr{Cov}_{n\beta, U}(\OO,\, G - G^1)\]
where the covariance is taken with respect to the posterior $\mu_{n\beta}$ restricted to $U$.  More generally, if $\OO = \OO_h$ depends on $h$ then
\[\chi_{\OO,\xi,U} = \mr{Cov}_{n\beta, U}(\OO,\, G - G^1) + \frac{1}{n\beta}\langle \nabla_\xi \OO \rangle_{n\beta, U}.\]
In particular the dependence of $\OO$ on $h$ contributes only at lower order in $n\beta$.  The initial-state and reward deformations of \cref{finite_tangent_remark} are affine in $h$, so this form covers all cases used in the present paper.
\end{prop}

\begin{proof}
Write $Z^h = Z^h_{n\beta}(U) = \int_U \exp(-n\beta G^h(w)) \phi(w) \d w$.  Then
\begin{align*}
 n\beta \cdot \chi_{\OO,\xi,U} &= \frac{\dd}{\dd h}\bigg|_{h=0} \langle \OO_h \rangle_{n\beta,U} \\
 &= \frac{\dd}{\dd h}\bigg|_{h=0} \frac{1}{Z^h} \int_U \OO_h(w) \exp(-n\beta G^h(w)) \phi(w) \d w \\
 &= \frac{\dd}{\dd h}\bigg|_{h=0} \frac{1}{Z^h} \cdot Z^0 \cdot \langle \OO_h \rangle_{n\beta,U} + \langle \nabla_\xi \OO \rangle_{n\beta,U} + n\beta \langle \OO \cdot (G - G^1) \rangle_{n\beta,U}.
\end{align*}
For the first term, differentiating $1/Z^h$ at $h=0$ gives
\[\frac{\dd}{\dd h}\bigg|_{h=0} \frac{1}{Z^h} = -\frac{1}{(Z^0)^2} \int_U \left(-n\beta \frac{\dd}{\dd h}\bigg|_{h=0} G^h(w)\right) \exp(-n\beta G(w)) \phi(w) \d w = -\frac{n\beta}{Z^0} \langle G - G^1 \rangle_{n\beta,U}.\]
Combining these expressions we obtain
\[n\beta \cdot \chi_{\OO,\xi,U} = -n\beta \langle \OO \rangle_{n\beta,U} \langle G - G^1 \rangle_{n\beta,U} + n\beta \langle \OO \cdot (G - G^1) \rangle_{n\beta,U} + \langle \nabla_\xi \OO \rangle_{n\beta,U}\]
and dividing through by $n\beta$ yields the result.
\end{proof}

\begin{remark}
 The content of Proposition \ref{susc_covariance_prop} is that susceptibilities can be estimated using samples from the \emph{unperturbed} posterior alone.  One does not need to resample from $\mu^h_{n\beta}$ for each value of $h$; it suffices to evaluate $G^1(w)$ at the sampled points.  This is the same simplification that appears in the supervised learning case \citep{Susceptibilities}.
\end{remark}

\subsection{Observables and the Connection to Refined LLCs}

The choice of observable $\OO$ determines what aspect of the model's internal structure the susceptibility probes.  The primary choice of observable that we will use, following \citet{Susceptibilities}, is a \emph{weight-restricted loss}.

Decompose the parameter space as $W = V \times C$, where $C$ represents the parameters of a component (for instance a single layer or a convolutional block) and $V$ represents the remaining parameters.  Fix a local minimum $\wstar = (v^*, c^*) \in W$ of the regret.  Define the observable
\begin{equation} \label{weight_restricted_observable}
\OO_C(v,c) = \delta(v - v^*)(G(w) - G(\wstar)).
\end{equation}
The delta function restricts to variations in the component $C$ alone, and the subtraction of $G(\wstar)$ centers the observable at the local minimum.

\begin{remark} \label{susc_LLC_remark}
 The posterior expectation $n\beta \langle \OO_C \rangle_{n\beta}$ computes the weight-refined LLC associated to the component $C$ to leading order, by comparison with the WBIC (see \cref{LLC_appendix}).  Informally, therefore, the susceptibility $\chi_{\OO_C, \xi}$ measures the rate of change of the weight-refined LLC with respect to the perturbation $\xi$.  More precisely, the LLC $\lambda_C$ is a rational number -- an invariant of the singularity type of $(G - G(\wstar))|_C$ -- and therefore does not admit continuous deformations.  What varies continuously is the \emph{effective} weight-refined LLC at finite $n\beta$, defined as
 \[\lambda_C^{n\beta} = \frac{\dd}{\dd \log(n\beta)} F^C_{n,\beta}(U)\]
 where $F^C_{n,\beta}(U) = -\log Z_{n,\beta}(U)$ is the free energy restricted to the component $C$.  The susceptibility is $(n\beta)^{-2}$ times the response of the effective rLLC with the observable held fixed: $\nabla_\xi \lambda_C^{n\beta} = (n\beta)^2 \chi_{\OO_C,\xi} + n\beta \langle \nabla_\xi \OO_C \rangle_{n\beta}$.  Since $\lambda_C^{n\beta} \to \lambda_C$ as $n\beta \to \infty$ but $\lambda_C$ is locally constant in $\xi$, the limits $n\beta \to \infty$ and $\nabla_\xi$ do not in general commute.  We therefore regard the susceptibility as probing the finite-$n\beta$ geometry of the regret landscape rather than its asymptotic singularity type.
\end{remark}

\begin{remark}
This interpretation assumes that $\wstar$ is a critical point of $G^h$ not only at $h=0$ but also for sufficiently small $h>0$, so that the (refined) LLC continues to make sense after a macroscopic perturbation. For an initial state deformation the perturbed regret decomposes as $G^h(w) = \sum_{s \in \mc S} \Lambda^h(s) G_s(w)$, where $G_s(w) = V^*(s) - \bb E(r(\tau) \,|\, s_0 = s)$ is the regret from the initial state $s$ under the policy $\pi_w$; each $G_s$ is non-negative.  If $\wstar$ is a global minimum of $G$ then $G(\wstar) = 0$, so $G_s(\wstar) = 0$ for every $s$ in the support of $\Lambda$, and hence $G^h(\wstar) = 0$ is the minimum value of $G^h$ whenever $\xi$ is positive only on the support of $\Lambda$; the condition therefore holds for initial state deformations of this kind, in particular for the full-support distributions $\Lambda_\alpha$ with $\alpha > 0$ used in this paper.  One should not expect it to hold generically, however, especially for transition function or reward perturbations (which would not locally preserve the critical locus even in the simple example presented here).  The definition of susceptibilities makes sense without any assumption but their geometric interpretation would require more care for perturbations that move the critical locus.
\end{remark}

This connection provides an interpretation of the sign of the susceptibility.  Suppose $\xi$ is an initial state perturbation in the direction of a state $s$.
\begin{itemize}
 \item A \emph{positive} susceptibility $\chi_{\OO_C,\xi} > 0$ means that concentrating the initial state distribution towards $s$ \emph{increases} the effective complexity of the model in the component $C$.
 \item A \emph{negative} susceptibility $\chi_{\OO_C,\xi} < 0$ means that concentrating the initial state distribution towards $s$ \emph{decreases} the effective complexity of the model in $C$.
\end{itemize}
One may therefore think of the collection of susceptibilities associated to different states as describing which parts of the model's internal algorithm are used more or less intensively for different classes of state.

\begin{remark}[The successor representation] \label{successor_remark}
 There is an alternative organizing principle that unifies reward and initial state susceptibilities.  Recall that the \emph{successor representation} \citep{Dayan} of a policy $\pi_w$ with discount factor $\gamma$ is the matrix $M(w) = (\mr{id}_{\mc S} - \gamma P(w))^{-1}$ where $P(w)$ is the state-transition matrix under $\pi_w$.  The expected return, which is truncated at $T_{\mr{max}}$ steps, decomposes as $R(w) = \Lambda^T M_{T_{\mr{max}}}(w) r_w$, where $M_{T_{\mr{max}}}(w) = \sum_{t=0}^{T_{\mr{max}} - 1} (\gamma P(w))^t$ converges to $M(w)$ as $T_{\mr{max}} \to \infty$ and $r_w(s) = \sum_a \pi_w(a \,|\, s)\, r(s,a)$ is the policy-averaged reward.  Since $r_w$ depends on $w$, the common object behind the two perturbations is the initial-state-conditioned state-action occupancy
 \[D_w(s_0; s, a) = \sum_{t < T_{\mr{max}}} \gamma^t\, P_w(s_t = s, a_t = a \,|\, s_0),\]
 in terms of which $R(w) = \sum_{s_0} \Lambda(s_0) \sum_{s,a} D_w(s_0; s, a)\, r(s,a)$, so both the dependence on $\Lambda$ and the dependence on $r$ factor through $D_w$.  One could therefore study a single matrix-valued \emph{successor susceptibility} $\mr{Cov}_{n\beta}(\OO, D_w)$, from which the reward and initial state susceptibilities are recovered as linear combinations of entries.  While we will not use this perspective directly in the present paper, it suggests that the two classes of perturbation contain complementary information about the same underlying structure.
\end{remark}

\subsection{Estimation} \label{susc_estimation_subsection}

We now describe how the susceptibilities defined above may be estimated in practice.  Fix a local minimum $\wstar$ of the regret, a component $C \sub W$ with $W = V \times C$ and a perturbation $\xi$ with associated deformed regret $G^1$.  We follow the approach to LLC estimation developed in \citet{LLC} and applied to the reinforcement learning setting in \citet[\S E]{RL1}; we refer the reader to those references for a more detailed discussion.

\subsubsection{SGLD Sampling}
We use the SGLD sampling procedure described in \cref{SGLD_estimation_appendix} to approximately sample from the tempered posterior $\mu_{n,\beta}$ associated to a Gaussian prior centered at $\wstar$ with variance $\sigma^2$.  Let $y_1, \ldots, y_T$ denote the output of an SGLD chain with the weight restriction to the component $C$ (that is, only the parameters in $C$ vary while those in $V$ are held fixed at $v^*$).  Let $y_1', \ldots, y_T'$ denote a second SGLD chain sampling from the full (unrestricted) posterior.

\subsubsection{The Susceptibility Estimator}
By Proposition \ref{susc_covariance_prop} the susceptibility is, to leading order, the covariance of $\OO_C$ and $G - G^1$.  We estimate this covariance empirically as follows.

\begin{definition} \label{susc_estimator_def}
The \emph{susceptibility estimator} associated to the component $C$ and perturbation $\xi$ is
\begin{align*}
\widehat \chi_C(\xi) &= \frac{1}{T}\sum_{j=1}^T (G_n(y_j) - G_n(\wstar))(G_n(y_j) - G_n^1(y_j)) \\
&\qquad - \frac{1}{T^2}\left(\sum_{j=1}^T (G_n(y_j) - G_n(\wstar))\right) \left(\sum_{j=1}^T \bigl(G_n(y_j') - G_n^1(y_j')\bigr) \right)
\end{align*}
where $G_n, G_n^1$ are the importance-weighted regret estimators associated to the unperturbed and perturbed Markov problems respectively.
\end{definition}

Note that the first factor in the second term uses weight-restricted samples $y_j$ while the second factor uses unrestricted samples $y_j'$; this reflects the structure of the observable $\OO_C$, which involves a delta function in the $V$ directions.  More precisely, writing $Z_C$ for the partition function of the posterior restricted to the slice $\{v^*\} \times C$ and $\bb E_C$ for its expectation, the delta observable has posterior expectation $\langle \OO_C \rangle_{n\beta,U} = (Z_C/Z)\, \bb E_C(G - G(\wstar))$, and likewise $\langle \OO_C (G - G^1) \rangle_{n\beta,U} = (Z_C/Z)\, \bb E_C((G - G(\wstar))(G - G^1))$.  The estimator above therefore computes $(Z/Z_C)\, \mr{Cov}_{n\beta,U}(\OO_C, G - G^1)$: the covariance of Proposition \ref{susc_covariance_prop} divided by the slice mass $Z_C/Z = \langle \delta(v - v^*) \rangle_{n\beta,U}$.  This factor depends on the component $C$ but not on the perturbation $\xi$, so it rescales all susceptibilities of a component by one constant; comparisons across components are of these rescaled quantities.

\begin{remark}[Exact regret computation] \label{exact_regret_remark}
In some finite Markov decision problems, including the gridworld environment studied in this paper, the expected regret $G(w)$ can be computed analytically for any given policy $\pi_w$ (since the state space is finite and the transition function is known).  In cases such as this where the computation is tractable one may use the exact values of $G(y_j)$, $G(\wstar)$ and $G^1(y_j)$ in the estimator above, rather than importance-weighted estimates.  This reduces the variance of the susceptibility estimator since the only source of randomness is the SGLD sampling itself.
\end{remark}

\begin{remark}[Reward susceptibilities]
 The estimator in Definition \ref{susc_estimator_def} applies equally well to reward perturbations.  Given a reward perturbation $\rho$ one computes $G^1(w)$ (or its estimator $G_n^1$) using the deformed reward function $r^1(s,a) = r(s,a) + \rho(s,a)$ in place of $r(s,a)$.  Since the regret is linear in the reward function up to a constant that does not depend on $w$, $G^1(w) - G(w)$ equals a $w$-independent constant minus the expected value of $\rho$ with respect to the trajectory distribution $q_w$; only the latter enters the covariance, and it can be computed analytically in the finite state space setting.
\end{remark}

\subsection{On-Distribution vs Off-Distribution Susceptibility Estimation} \label{on_off_appendix}

Susceptibilities are fundamentally measures of how the geometry of the regret landscape responds to linear perturbations of the Markov problem.  They are defined locally near any point $w^* \in W$, though the interpretation we give them is only meaningful near points $w^*$ that are (at least approximate) local minima of the regret.  As such, susceptibilities and their estimators as in \cref{susc_estimator_def}, may be evaluated using any choice of regret $G'$, not necessarily the regret function $G$ used to find the point $w^*$ under training.  We refer to estimates using $G' = G$ as \emph{on-distribution} and to estimates using a different $G'$ as \emph{off-distribution}.  In our empirical results we only use the off-distribution choice for the regret $G_1$ associated to the uniform initial-state distribution $\Lambda_1$.

In our environment, optimal policies are optimal independent of the mixing parameter, so for $\alpha, \alpha' \in (0, 1]$ a parameter $w^*$ that is approximately a global minimum of $G_\alpha$ is also approximately a global minimum of $G_{\alpha'}$.  Therefore off-distribution susceptibility estimation is interpretable for converged (phase 3) parameter values.  Earlier in training, $w^*$ may only be an approximate local minimum of $G_\alpha$ and need not be an approximate local minimum of $G_{\alpha'}$ for $\alpha' \neq \alpha$, so interpretations of off-distribution susceptibility estimates outside phase 3 are questionable, and such estimates are not used in this paper.

\section{Further Experimental Details}

\subsection{The Environment} \label{environment_appendix}

The environment is a simplified version of the \emph{Cheese in the Corner} gridworld described in \citet{Misgen} and \citet{RL1}.  The full grid is $13 \times 13$, with a border of walls leaving an $11 \times 11$ interior of navigable cells.  The cheese and mouse each occupy one of these $121$ cells, giving $121 \times 120 = 14{,}520$ possible states; since the state space is finite and transitions are deterministic, the regret $G(w)$ can be computed exactly for any policy $\pi_w$ (cf.\ Remark \ref{exact_regret_remark}).  Episodes terminate upon reaching the cheese or after $T_{\mr{max}} = 128$ steps.

\subsection{The Model} \label{model_appendix}
The model used in this paper is an adaptation of the IMPALA convolutional network introduced in \citet{IMPALA} and applied in \citet{Misgen}; the architecture is identical to that used in \citet{RL1}.

The states are presented as a $13 \times 13 \times 3$ tensor.  The three channels encode walls, the mouse position, and the cheese position respectively.  This image is passed through a convolutional encoder consisting of three blocks, where each block applies a $3\times 3$ convolution followed by $3\times 3$ max-pooling with stride $2$ and two pre-activation residual blocks (each: ReLU $\to$ $3\times 3$ Conv $\to$ ReLU $\to$ $3\times 3$ Conv, with a skip connection).  The channel dimensions are $3 \to 16 \to 32 \to 32$, and the spatial dimensions reduce as $13 \to 7 \to 4 \to 2$ through the three pooling stages.  After a final ReLU activation function the output is flattened to a $128$-dimensional vector and passed through two feedforward layers: FF 1 maps to a $256$-dimensional observation embedding, and FF 2 (replacing the LSTM of the original IMPALA architecture) maps this to a $256$-dimensional representation, each followed by a ReLU activation.  A final linear layer produces logits over the four actions.

The six weight-restriction components used for susceptibility estimation partition the parameter space as follows.

\begin{table}[h]
\centering
\begin{tabular}{llr}
\toprule
\textbf{Component} & \textbf{Module} & \textbf{Parameters} \\
\midrule
Conv 1 & \texttt{blocks.0.*} & 9{,}728 \\
Conv 2 & \texttt{blocks.1.*} & 41{,}632 \\
Conv 3 & \texttt{blocks.2.*} & 46{,}240 \\
FF 1   & \texttt{fc.*}       & 33{,}024 \\
FF 2   & \texttt{rnn.*}      & 65{,}792 \\
Policy & \texttt{policy.*}   & 1{,}028 \\
\midrule
\multicolumn{2}{l}{Total} & 197{,}444 \\
\bottomrule
\end{tabular}
\caption{Architectural components and parameter counts. }
\label{tab:components}
\end{table}

Training uses vanilla REINFORCE \citep{Williams1992REINFORCE} with a constant zero baseline function.  The optimizer is Adam with a constant learning rate of $5 \times 10^{-5}$.  At each gradient step the model collects rollouts of length $64$ across $9{,}600$ parallel environments, giving $614{,}400$ environment steps per update.  Each run trains for $10$B environment steps on a single NVIDIA RTX-5090, taking roughly $4$ hours.  Further training hyperparameters (discount factor, mixing parameter) are specified in the main text; see \cite[Appendix C]{RL1} for additional implementation details.

\subsection{Refined LLC and Susceptibility Estimation}

Our refined LLC and initial-state susceptibility estimates are computed by samples generated using preconditioned SGLD \citep{Li_Chen_Carlson_Carin_2016} to estimate expectation values for the tempered generalized posterior described in \cref{SGLD_estimation_appendix}.  The hyperparameters used throughout the paper are those of \cref{tab:sgld_hyperparameters}.

\begin{table}[h]
\centering
\caption{SGLD hyperparameters used for all rLLC and susceptibility estimates reported in the main text.}
\label{tab:sgld_hyperparameters}
\resizebox{\textwidth}{!}{%
\begin{tabular}{ll}
\toprule
Hyperparameter & Value \\
\midrule
SGLD optimizer                                              & RMSProp-preconditioned SGLD \\
RMSProp burn-in steps                                       & $20$ \\
Learning rate                                               & $1 \times 10^{-6}$ \\
Localization strength                                       & $200$ (Gaussian prior with $\sigma^2 = 1/200$) \\
Inverse temperature $n\beta$                                & $1000$ \\
Levels per SGLD gradient step                               & $4800$ \\
Gradient accumulation per chunk                             & $6$ \\
SGLD steps between draws                                    & $4$ (giving $6000$ steps per chain total) \\
Draws per chain                                             & $1500$ (first $500$ discarded as burn-in) \\
Chains per (weight restriction, checkpoint, model)          & $5$ (independent SGLD seeds $31$--$35$) \\
\bottomrule
\end{tabular}%
}
\end{table}

These choices match the LLC estimation hyperparameters of \cite[Appendix A]{RL1} and we refer to that discussion for the rationale behind these choices.\footnote{There are exceptions to the last row: the $\alpha = 0.5$ checkpoints at $16{,}270$ use five chains, and the checkpoints relevant for the direction-conditioned chains of \cref{direction_regret_section} have $25$ or $30$ chains each.}

For each weight-restricted estimator we run an independent SGLD chain with parameters outside the weight restriction frozen at $\wstar$.  The six weight restrictions used are those of \cref{model_appendix} plus an unrestricted (``full'') chain that samples over the entire parameter space.

For models trained at $\alpha < 1$, susceptibility and rLLC estimates are computed using both $\alpha' = \alpha$ (on-distribution) and $\alpha' = 1$ (off-distribution).  The distinction concerns the initial-state distribution used to generate rollouts during SGLD gradient estimation.  For $\alpha = 1$ models only the $\alpha' = 1$ evaluation is reported.  The theoretical caveats around off-distribution evaluation are discussed in \cref{on_off_appendix}.

For the direction-conditioned posterior regret experiments of \cref{direction_regret_section} we use the same per-chain hyperparameters, but using an independent set of 5 chains for each direction, ensuring the estimates are independent from one another.

\subsection{Activation Steering Methodology} \label{activation_steering_appendix}

We supplement the description of the activation-steering experiment in \cref{activation_steering_section} with additional procedural details.

We measure the steering threshold scale independently for each of the five initial IMPALA blocks.  The final policy-head output is not steered, since adding a steering vector directly to the action logits would not probe the internal (activation space) computational structure.

For the purpose of identifying distinguished directions (\cref{distinguished_directions_section}) we use the FF 2 layer.  The susceptibility-distinguished directions are typically visible in every layer of the model; FF 2 is the deepest internal layer and the layer at which the distinguished-direction susceptibilities separate most cleanly in practice.

For a cardinal direction $D_c$ we enumerate all (cheese, mouse) configurations with the mouse cardinally aligned with the cheese in direction $D_c$, giving $605$ states per direction.  A configuration is kept for the final analysis if (i) the unsteered model takes the cheese-directed action at baseline, and (ii) the mirror image of the cheese across the mouse lies in the interior of the grid (not outside the grid or on a wall cell).  

To find the minimum steering scale $s_{\min}$ at which the model's argmax action changes from the unsteered baseline action, we first run a coarse geometric scan over the candidate scales $\{0.1, 0.2, 0.5, 1.0, 2.0, 5.0, 10.0, 20.0, 50.0, 100.0\}$ to bracket the first scale at which the action flips, and then binary-search within the bracket to a tolerance of $0.01$ in $s$.  If the action does not flip at the maximum scale of $100$, the configuration is recorded as a no-flip and excluded from the mean.  The coarse scan is used before the bisection to protect against the occasional non-monotonic response in which the flipped action changes identity as the scale increases.

\subsection{Hessian Trace Estimation} \label{hessian_appendix}

We estimate the Hessian trace $\mr{tr}(\nabla^2 G(\wstar))$ of the regret at a given checkpoint $\wstar$ using Hutchinson's method.  We sample $N$ random vectors $v_1, \ldots, v_N$ from a Rademacher distribution, i.e. each component is uniformly sampled from $\{-1,1\}$.  We then estimate $v_i^T \nabla^2 G(\wstar) v_i$ via the Hessian-vector product (HVP) identity
\[v^T \nabla^2 G(\wstar) v = v^T(\nabla(\nabla G(w^*)^Tv)).\]
This may be estimated for a given $v$ using the REINFORCE method for estimating the regret gradient along with double backpropagation (Hessian-vector product, HVP). One then computes the sample mean over the samples $v_i$. The HVP is accumulated per minibatch ($4{,}096$ flattened timesteps) so that only one batch's autograd graph is held in memory at a time.

\subsubsection{Hyperparameter Convergence}

We determined the number of Hutchinson samples $N$ and environment rollouts (``levels'', each a distinct initial-state configuration) required for a stable estimate by a convergence study on a representative model (an $\alpha = 0.6$ model at its first phase 3 checkpoint); see \cref{tab:hessian_convergence}.  The mean stabilizes around $N = 1000$ Hutchinson samples with a sample-mean coefficient of variation SEM/mean of roughly $3\%$ (here $\mr{SEM} = \mr{std}/\sqrt{N}$ is the standard error of the mean for the Hutchinson estimator).  The estimator has already approximately converged at $300$ levels and does not continue to improve thereafter.

\begin{table}[h]
\centering
\caption{Convergence of the Hutchinson Hessian trace estimator with respect to the number of samples (left block, $300$ rollout levels) and number of rollout levels (right block, $200$ samples).  Test model: an $\alpha = 0.6$ seed at its first phase 3 checkpoint.  $\mr{SEM}/\mr{mean}$ is the sample-mean coefficient of variation.}
\label{tab:hessian_convergence}
\begin{tabular}{rrr|rrr}
\toprule
\multicolumn{3}{c|}{Hutchinson samples ($300$ levels)} & \multicolumn{3}{c}{Levels ($200$ samples)} \\
$N$ & Mean & $\mr{SEM}/\mr{mean}$ & Levels & Mean & $\mr{SEM}/\mr{mean}$ \\
\midrule
$50$   & $18{,}627$ & $0.142$ & $300$  & $27{,}843$ & $0.081$ \\
$200$  & $23{,}688$ & $0.076$ & $600$  & $28{,}741$ & $0.084$ \\
$500$  & $26{,}166$ & $0.051$ & $1200$ & $29{,}559$ & $0.082$ \\
$1000$ & $26{,}838$ & $0.036$ & $2400$ & $23{,}419$ & $0.087$ \\
$1500$ & $27{,}193$ & $0.031$ & $4800$ & $28{,}014$ & $0.087$ \\
$2000$ & $27{,}430$ & $0.026$ \\
\bottomrule
\end{tabular}
\end{table}

Based on this study we use $N=1000$ Hutchinson samples and $300$ rollout levels.

\subsubsection{Within-Phase-3 Trajectory by Hyperparameter Regime}

\Cref{tab:hessian_summary} summarizes the ratio of the Hessian trace estimate between the end and start of phase 3 across three hyperparameter regimes, alongside the corresponding LLC ratios for comparison.  ``Down'' counts models for which the ratio is below one, i.e. where the Hessian trace estimator declines.  \Cref{tab:hessian_LLC_correlation} summarizes the Spearman correlation between the LLC estimate ratio and Hessian trace estimate ratio.  We do not observe any significant correlation: we fail to reject the null hypothesis of no monotone relationship at $p=0.05$ in all hyperparameter regimes.

\begin{table}[h]
\centering
\caption{Within phase 3 last/first checkpoint ratios for the LLC estimator and the Hessian trace estimator, across the three hyperparameter regimes of \cref{three_stage_section,alpha1_section}.   For $\alpha = 0.6$ and $\alpha'=1$ the Hessian trace estimator is available for a subset of $20$ models.}
\label{tab:hessian_summary}
\begin{tabular}{lrr c rr c rr}
\toprule
Group & \multicolumn{2}{c}{LLC ($\alpha' = \alpha$)} & & \multicolumn{2}{c}{Hess ($\alpha' = \alpha$)} & & \multicolumn{2}{c}{Hess ($\alpha' = 1$)} \\
\cmidrule(lr){2-3}\cmidrule(lr){5-6}\cmidrule(lr){8-9}
 & Down & Median & & Down & Median & & Down & Median \\
\midrule
$\alpha = 0.5$ ($n = 30$) & $30/30$ & $0.52$ & & $16/30$ & $0.87$ & & $19/30$ & $0.90$ \\
$\alpha = 0.6$ ($n = 30$) & $29/30$ & $0.36$ & & $19/30$ & $0.75$ & & $6/20$  & $1.10$ \\
$\alpha = 1$   ($n = 28$) & $28/28$ & $0.26$ & & $28/28$ & $0.43$ & & ---      & ---     \\
\bottomrule
\end{tabular}

\end{table}

\begin{table}[h]
\centering
\caption{Per-group Spearman rank correlation between the last/first LLC ratio and the last/first Hessian trace ratio, with associated $p$-values.  Models whose Hessian ratio exceeds $5$ are excluded as outliers.  The $\alpha = 0.6$ row uses the twenty seed models for which both ratios are available.}
\label{tab:hessian_LLC_correlation}
\begin{tabular}{lrr c rr}
\toprule
Group & \multicolumn{2}{c}{Hess ($\alpha' = \alpha$)} & & \multicolumn{2}{c}{Hess ($\alpha' = 1$)} \\
\cmidrule(lr){2-3}\cmidrule(lr){5-6}
 & $\rho$ & $p$ & & $\rho$ & $p$ \\
\midrule
$\alpha = 0.5$ ($n = 30$) & $+0.04$ & $0.82$ & & $+0.18$ & $0.36$ \\
$\alpha = 0.6$ ($n = 20$) & $-0.26$ & $0.28$ & & $-0.32$ & $0.18$ \\
$\alpha = 1$   ($n = 28$) & $+0.07$ & $0.73$ & & ---      & ---     \\
\bottomrule
\end{tabular}
\end{table}

\subsection{Compute Cost}
The training and SGLD sampling for this paper was performed on RTX-5090 GPUs.  Training a model cost around 4 GPU-hours, whereas computing a single SGLD chain cost around 0.9 GPU hours for the full model, with weight-restricted SGLD sampling for the later layers costing only 0.2 GPU hours.  We estimate that the SGLD chains used in this paper had an average cost of around 0.5 GPU hours/SGLD chain. We estimate that training all the models used for this paper cost on the order of 3000 GPU hours, whereas the total cost of the SGLD sampling for the data in this paper is estimated to be around 9000 GPU hours. The compute spent on other experiments like activation steering and Hessian Trace estimation is much less. In addition to this, we spent compute on exploratory experiments that didn't make it into this paper.

\section{Distinguished Direction Criteria} \label{cluster_criteria_appendix}

In \cref{branching_logic_section,alpha1_section} we referred to \emph{clusters} in susceptibility space associated to the relative position of the agent and the goal.  There are many ways one can straightforwardly define and detect clusters.  In this appendix we compare a range of choices to verify that different criteria give qualitative similar results as we vary their definition.

We will compare three families of criteria for partitioning either the four cardinal directions or the eight cardinal and diagonal directions into a \emph{distinguished} set and its complement:

\begin{enumerate}
    \item \emph{Percentile-based} (P99/P1 and P95/P5): rank cardinal directions by median susceptibility; find the largest low cluster whose pooled high-percentile lies below the remaining directions' pooled low-percentile.  P99/P1 (99th percentile of the distinguished set lies below 1st percentile of its complement) is the strict criterion used in the main text; P95/P5 (95th percentile of the distinguished set lies below 5th percentile of its complement) is a more relaxed counterpart.
    \item \emph{Mean-separation} (stddev\_$s$): rank directions by their median; a candidate low cluster qualifies as distinguished if the separation score difference of means between the clusters exceeds the threshold $s \widehat \sigma_{\mr{high}}$ where $\widehat \sigma_{\mr{high}}$ is the empirical standard deviation of the high clusters.  We sweep over the set $s \in \{1.0, 1.5, 2.0, 2.5, 3.0, 3.5, 4.0\}$.
    \item \emph{Gap-based} (gap\_$r$): sort the direction medians, locate the largest gap between consecutive medians, and split there.  We include a guard parameter $r$, requiring that the largest gap must exceed $r$ times the median gap; otherwise no split is made.  We sweep over the set $r \in \{0, 2, 4, 8, 16, 32\}$.
\end{enumerate}

Each criterion is evaluated on the FF 2 layer susceptibilities, on-distribution at $\alpha' = \alpha$ for each population. 

\subsection{Cardinal Direction Clustering}
We first check that the qualitative conclusion of \cref{activation_steering_section} -- that susceptibility-distinguished directions are systematically harder to steer than non-distinguished ones -- does not depend on the specific cluster criterion chosen.  For each (criterion, population) combination we compute the activation-steering $t$-statistic of \cref{activation_steering_section} and check whether it remains significant.  We report the three populations where $\alpha = 0.5, 0.6$ and 1 separately.  The data is presented in two tables.  \Cref{tab:cluster_crit_early} reports the $t$-statistic and seed count $n$ at the first phase 3 checkpoint by population. \Cref{tab:cluster_crit_cross} reports the same quantities for the cross-checkpoint test, i.e. where we use the same clusters and attempt activation steering late in phase three.

\begin{remark} \label{rem:P99_alpha_less_than_1}
The strict percentile-based criterion P99/P1 is, in practice, insufficiently sensitive to detect meaningful clusters when $\alpha < 1$.  For example, for $\alpha = 0.6$ in a representative subset of $20$ $\alpha = 0.6$ seeds, the criterion P99/P1 fires on only $2$ seeds for off-distribution evaluation $\alpha' = 1$ and on none for on-distribution evaluation $\alpha' = \alpha$.
\end{remark}

\begin{table}[h]
\centering
\caption{Activation-steering $t$-statistics and seed counts at the first phase 3 checkpoint on the FF 2 layer, broken out by training-distribution population.  Each cell reports the one-sample two-sided $t$-statistic of \cref{activation_steering_section} together with $n$, the number of seeds in which the criterion identifies at least one distinguished direction.  The symbol ``--'' indicates that the criterion is not applicable to the population as in \cref{rem:P99_alpha_less_than_1} or that fewer than $5$ seeds supplied a distinguished direction.}
\label{tab:cluster_crit_early}
\begin{tabular}{l rr c rr c rr}
\toprule
Criterion & \multicolumn{2}{c}{$\alpha = 0.5$} & & \multicolumn{2}{c}{$\alpha = 0.6$} & & \multicolumn{2}{c}{$\alpha = 1$} \\
\cmidrule(lr){2-3}\cmidrule(lr){5-6}\cmidrule(lr){8-9}
 & $t$ & $n$ & & $t$ & $n$ & & $t$ & $n$ \\
\midrule
P99/P1 & -- & -- & & -- & -- & & $6.7$ & $64$ \\
P95/P5 & $9.3$ & $66$ & & $2.6$ & $29$ & & $8.1$ & $90$ \\
stddev\_1.0 & $7.2$ & $97$ & & $7.2$ & $63$ & & $10.8$ & $85$ \\
stddev\_1.5 & $7.8$ & $99$ & & $7.3$ & $64$ & & $11.7$ & $89$ \\
stddev\_2.0 & $7.8$ & $96$ & & $6.3$ & $64$ & & $11.1$ & $94$ \\
stddev\_2.5 & $7.9$ & $89$ & & $6.4$ & $63$ & & $9.4$ & $92$ \\
stddev\_3.0 & $8.1$ & $86$ & & $7.0$ & $60$ & & $10.1$ & $86$ \\
stddev\_3.5 & $9.1$ & $85$ & & $6.6$ & $53$ & & $7.6$ & $74$ \\
stddev\_4.0 & $8.3$ & $78$ & & $3.6$ & $37$ & & $8.0$ & $65$ \\
gap\_0 & $10.6$ & $104$ & & $5.6$ & $67$ & & $6.8$ & $99$ \\
gap\_2 & $10.6$ & $104$ & & $5.6$ & $67$ & & $6.8$ & $98$ \\
gap\_4 & $10.4$ & $103$ & & $5.3$ & $64$ & & $6.8$ & $86$ \\
gap\_8 & $8.0$ & $75$ & & $4.0$ & $37$ & & $7.3$ & $56$ \\
gap\_16 & $3.2$ & $21$ & & $1.1$ & $8$ & & $3.6$ & $19$ \\
gap\_32 & -- & -- & & -- & -- & & $4.2$ & $5$ \\
\bottomrule
\end{tabular}
\end{table}

\begin{table}[h]
\centering
\caption{Activation-steering $t$-statistics and seed counts for the cross-checkpoint test on the FF 2 layer: distinguished directions identified at the first phase 3 checkpoint, used to predict steering resistance at end of training.  Notation is the same as for \cref{tab:cluster_crit_early}.}
\label{tab:cluster_crit_cross}
\begin{tabular}{l rr c rr c rr}
\toprule
Criterion & \multicolumn{2}{c}{$\alpha = 0.5$} & & \multicolumn{2}{c}{$\alpha = 0.6$} & & \multicolumn{2}{c}{$\alpha = 1$} \\
\cmidrule(lr){2-3}\cmidrule(lr){5-6}\cmidrule(lr){8-9}
 & $t$ & $n$ & & $t$ & $n$ & & $t$ & $n$ \\
\midrule
P99/P1 & -- & -- & & -- & -- & & $7.7$ & $65$ \\
P95/P5 & $13.7$ & $66$ & & $2.6$ & $29$ & & $9.2$ & $90$ \\
stddev\_1.0 & $12.9$ & $97$ & & $12.6$ & $63$ & & $12.2$ & $85$ \\
stddev\_1.5 & $13.8$ & $99$ & & $12.9$ & $64$ & & $12.6$ & $89$ \\
stddev\_2.0 & $14.6$ & $96$ & & $12.7$ & $64$ & & $13.0$ & $94$ \\
stddev\_2.5 & $14.9$ & $89$ & & $12.1$ & $63$ & & $12.2$ & $92$ \\
stddev\_3.0 & $14.9$ & $86$ & & $11.6$ & $60$ & & $10.6$ & $86$ \\
stddev\_3.5 & $15.0$ & $85$ & & $10.2$ & $53$ & & $9.6$ & $74$ \\
stddev\_4.0 & $15.7$ & $78$ & & $6.4$ & $37$ & & $8.9$ & $65$ \\
gap\_0 & $15.4$ & $104$ & & $6.9$ & $67$ & & $9.2$ & $100$ \\
gap\_2 & $15.4$ & $104$ & & $6.9$ & $67$ & & $9.2$ & $99$ \\
gap\_4 & $15.3$ & $103$ & & $6.6$ & $64$ & & $9.2$ & $87$ \\
gap\_8 & $12.7$ & $75$ & & $4.7$ & $37$ & & $8.0$ & $56$ \\
gap\_16 & $5.2$ & $21$ & & $1.2$ & $8$ & & $5.8$ & $19$ \\
gap\_32 & -- & -- & & -- & -- & & $5.3$ & $5$ \\
\bottomrule
\end{tabular}
\end{table}

The $t$-statistic remains significant across all three populations, all three criterion families, and a wide range of thresholds within each family.  The exceptions are the strictest criteria (high-threshold gap and the strictest stddev variants) for $\alpha = 0.6$, which detect clusters on so few seeds that the test loses power.  Cross-checkpoint $t$-statistics are uniformly larger than same-checkpoint ones, reinforcing the observation from \cref{within_optimal_section} that susceptibility labels at the first phase 3 checkpoint predict end-of-training steering resistance better than end-of-training labels do themselves.  Among criteria with sufficient seed counts no single one dominates: the stddev family with thresholds $1.5$--$2.0$ achieves the largest $t$-statistics in most populations, but the percentile and gap families also produce significant results across the same range of populations.  We retain P99/P1 and the generalized P95/P5 of \cref{distinguished_directions_section} as the main-text default criteria for their interpretability rather than for their statistical power.

When we instead use susceptibilities at the end of training for the $\alpha = 0.6$ population to define clusters we find that no criterion produces a statistically significant prediction using same-checkpoint off-distribution ($\alpha' = 1$) susceptibilities (all $|t| \le 1.6$), while on-distribution the gap criteria give $t$ between $3.7$ and $4.2$, confirming that observation from \cref{within_optimal_section} that the susceptibility signal weakens by the end of the optimal phase is robust to different clustering criteria.

\subsection{Cardinal and Diagonal Direction Clustering}

Several of our cluster-based observations at $\alpha < 1$ (including discussion in \cref{branching_logic_section} and \cref{within_optimal_section}) use the P95/P5 criterion to define clusters in the set of eight cardinal and diagonal directions.  The activation-steering test of \cref{activation_steering_section} is, however, only defined for cardinally-aligned mouse-cheese pairs: mirroring the cheese across the mouse is well-posed for a cardinal alignment but requires additional choices for diagonals.  When computing the steering test statistic $X_i$ from an eight-direction-distinguished cluster we use only its cardinal subset.  We will check, therefore, whether the conclusions \cref{activation_steering_section} would change had we used a cardinal-only cluster criterion throughout instead of the eight-direction generalization.

\Cref{tab:cluster_crit_4_vs_8} reports activation-steering $t$-statistics for a few representative criteria computed under both cardinal-only (four-direction) and cardinal-plus-diagonal (eight-direction) clustering on the $\alpha = 0.5$ population (in which the natural clusters frequently mix cardinal and diagonal directions).  The $t$-statistics are comparable across criterion families.  The qualitative conclusion of \cref{activation_steering_section} -- that susceptibility-distinguished directions are systematically harder to steer -- is unchanged.  We retain the eight-direction criterion in the main text because it is essential to the cluster description at $\alpha < 1$ flagged at the start of this subsection (i.e. distinguished clusters including the Down-Right class); the steering analysis of \cref{activation_steering_section} yields the same qualitative conclusion either way.

\begin{table}[h]
\centering
\caption{End of phase 3 activation-steering $t$-statistics and seed counts for some representative cluster-detection criteria for $\alpha = 0.5$, comparing clusters in cardinal directions only (four-direction) and cardinal and diagonal directions.}
\label{tab:cluster_crit_4_vs_8}
\begin{tabular}{l rr c rr}
\toprule
Criterion & \multicolumn{2}{c}{Four-direction} & & \multicolumn{2}{c}{Eight-direction} \\
\cmidrule(lr){2-3}\cmidrule(lr){5-6}
 & $t$ & $n$ & & $t$ & $n$ \\
\midrule
P95/P5 & $13.7$ & $29$ & & $12.8$ & $41$ \\
stddev\_1.5 & $6.3$ & $74$ & & $3.9$ & $39$ \\
stddev\_2.0 & $6.2$ & $65$ & & $5.3$ & $38$ \\
stddev\_3.0 & $6.6$ & $55$ & & $4.8$ & $41$ \\
gap\_0 & $15.9$ & $104$ & & $16.1$ & $101$ \\
\bottomrule
\end{tabular}
\end{table}

\begin{remark}
The cluster-detection criteria above are all evaluated on-distribution, i.e. for $\alpha' = \alpha$.  For $\alpha < 1$ models the cluster of \cref{branching_logic_section} is not consistently surfaced when susceptibilities are evaluated off-distribution at $\alpha' = 1$, even at converged phase 3 checkpoints; we therefore reserve off-distribution evaluation for the comparisons in \cref{regression_appendix,phase2_residue_section}, where the analysis does not depend on cluster detection.
\end{remark}

\section{Symmetry Breaking and Initialization} \label{fixed_init_appendix}
In this appendix we describe the evidence behind the claim in \cref{init_sensitivity_section} that the distinguished directions in the $\alpha = 1$ susceptibility clusters are determined primarily by the initialization of the model parameters.

We selected two $\alpha = 1$ reference models, one with a strongly distinguished Right direction and one with a strongly distinguished Left direction.  For each reference we initialized nine additional training runs from the reference's initial parameters for a total of ten, varying only the seed used for trajectory sampling.  Susceptibilities estimated at checkpoint 2000 in each of the ten resulting models show the same distinguished direction as their reference: Right in nine of the ten runs derived from the first initialization, and Left in all ten runs derived from the second. Indeed the susceptibilities are structurally very similar.  We test this by computing the similarities between the Gram matrices of the ten models with each initialization seed via centered kernel alignment \citep[CKA;][]{Kornblith2019CKA}.  The mean CKA similarity of models with the same initialization seed but different trajectory sampling seeds is 0.971, whereas the mean CKA similarity of models with different initialization seeds is 0.024.  See \cref{fig:shared_init} for an illustration of this observation.

\begin{figure}[h]
    \centering
    \includegraphics[width=0.99\linewidth]{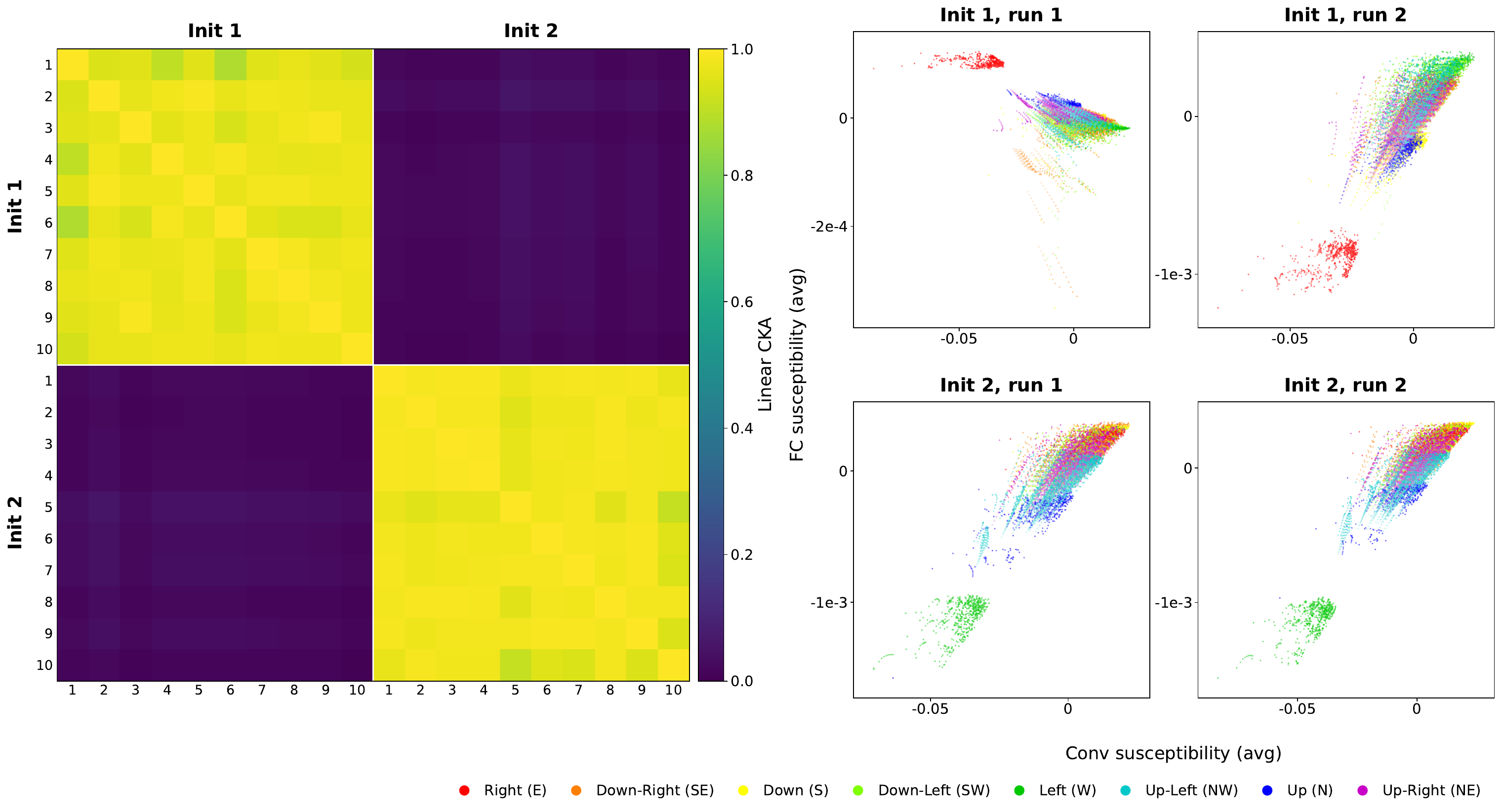}
    \caption{Similarity between susceptibilities of pairs of runs across two initialization seeds.  Susceptibilities are evaluated at checkpoint 2000 in each case and similarities are computed by CKA applied to their Gram matrices.  On the left we see the CKA similarities: the populations with the two distinct initialization seeds are labeled "Init 1" and "Init 2" respectively.  On the right we see scatter plots of a 2d projection of the susceptibilities for two independent runs with initialization seed 1 (above) and two independent runs with initialization seed 2 (below).}
    \label{fig:shared_init}
\end{figure}

\begin{remark}
 Note that we are not claiming that this is a generic feature of RL training: the small state space ($\sim 14{,}500$ states) and large effective batch size ($\sim 6 \times 10^5$ environment steps per gradient update) in our environment leave little room for trajectory randomness to accumulate, and we would expect the balance between initialization and trajectory randomness to differ in environments where the stochasticity of individual rollouts contributes more to each gradient step.
\end{remark}

\section{Supporting Statistical Analysis} \label{supporting_appendix}
This appendix provides supplementary statistical data in support of the following claims in \cref{evidence_section}.  We first discuss the activation steering analysis of \cref{activation_steering_section}.
\begin{itemize}
    \item \cref{steering_full_data_appendix} reports the full set 
    of $t$-test conditions for models trained with $\alpha = 0.5$.
    \item \cref{rank_correlation_appendix} extends the binary analysis of distinguished vs non-distinguished directions to a continuous correlation analysis.
    \item \cref{regression_appendix} investigates the relative contribution of susceptibilities and the binary classification associated to those directions that do and do not receive reward in phase 2.
\end{itemize}
Lastly we provide supplementary evidence for \cref{phase2_residue_section}. 
\begin{itemize}
    \item \cref{phase2_residue_appendix} verifies that the distinguished discrepancy associated to the Down-Right cluster does not replicate for other diagonal directions for which we don't expect the same asymmetry to have been introduced by intermediate phases.
\end{itemize}

\subsection{Activation Steering Additional Data} \label{steering_full_data_appendix}

We present the $t$-statistics for activation steering strength following \cref{activation_steering_section} compared to average susceptibility differences between clusters, as one varies the estimation process in several ways.  We will see that the qualitative result -- distinguished directions being significantly harder to steer -- is not specific to those conditions.  This data is presented in the $\alpha = 0.5$ case for the FF 2 layer as usual; the results for $\alpha = 0.6$ are similar.

Specifically, \cref{tab:al05_steering_full} presents  activation-steering $t$-test results for the following options.  We may define clusters using the on-distribution susceptibilities $\alpha' = 0.5$ or the off-distribution susceptibilities $\alpha' = 1$ either computed early or late in phase 3, though we find no clustering for $\alpha'=1$ late in phase 3 with the exception 4 training runs (and therefore exclude it from \cref{tab:al05_steering_full}).  We may also use the na\"ive binary partition defining $\{\mr{Right}, \mr{Down}\}$ as a cluster.  In each case we may attempt activation steering either early or late in phase 3 (though we do not attempt to steer early in phase 3 using susceptibilities computed late in phase 3).  This gives a total of seven possibilities. 

\begin{table}[h]
\centering
\caption{$\alpha = 0.5$ activation steering $t$-tests at the FF 2 layer.  $X_i^{\mr{FF 2}}$ is the per-seed test statistic defined in \cref{activation_steering_section}: the mean steering threshold for distinguished directions minus the mean steering threshold for non-distinguished directions.  Here and below ``phase 3 start'' refers to the first checkpoint of phase 3 for the seed in question, and ``phase 3 end'' refers to the last training checkpoint.}
\label{tab:al05_steering_full}
\resizebox{\textwidth}{!}{%
\begin{tabular}{lrrrrr}
\toprule
Condition & $n$ & $\mr{mean}(X)$ & $\mr{std}(X)$ & $t$ & $p$ (two-sided) \\
\midrule
Phase-3-start labels ($\alpha' = 1$), steered at phase 3 start & 41 & $+0.141$ $[+0.108, +0.172]$ & $0.107$ & $8.47$ & $1.9 \times 10^{-10}$ \\
Phase-3-start labels ($\alpha' = 0.5$), steered at phase 3 start & 66 & $+0.124$ $[+0.098, +0.151]$ & $0.108$ & $9.34$ & $1.3 \times 10^{-13}$ \\
Phase-3-start labels ($\alpha' = 1$), steered at phase 3 end & 41 & $+0.239$ $[+0.194, +0.280]$ & $0.143$ & $10.67$ & $2.8 \times 10^{-13}$ \\
Phase-3-start labels ($\alpha' = 0.5$), steered at phase 3 end & 66 & $+0.214$ $[+0.184, +0.245]$ & $0.127$ & $13.69$ & $8.1 \times 10^{-21}$ \\
Phase-3-end labels ($\alpha' = 0.5$), steered at phase 3 end & 41 & $+0.254$ $[+0.213, +0.290]$ & $0.127$ & $12.82$ & $9.4 \times 10^{-16}$ \\
Fixed $\{\mr{Right}, \mr{Down}\}$, steered at phase 3 start & 104 & $+0.131$ $[+0.108, +0.154]$ & $0.119$ & $11.20$ & $1.6 \times 10^{-19}$ \\
Fixed $\{\mr{Right}, \mr{Down}\}$, steered at phase 3 end & 104 & $+0.205$ $[+0.180, +0.230]$ & $0.133$ & $15.73$ & $3.8 \times 10^{-29}$ \\
\bottomrule
\end{tabular}%

}
\end{table}

All seven conditions reject the null hypothesis $\bb E(X^{\mr{FF 2}}) = 0$ with $p < 10^{-9}$.  The fixed $\{\mr{Right}, \mr{Down}\}$ partition produces the largest $t$-statistics because it uses all 104 matched seeds, but the susceptibility-identified-directions conditions remain significant on smaller samples.

\subsection{Rank Correlation Between Susceptibility and Steering Orderings} \label{rank_correlation_appendix}

The activation-steering comparison presented in \cref{activation_steering_section} relies on a binary classification of directions into distinguished and non-distinguished clusters.  We present here an alternative based on rank-ordering by studying the Spearman rank correlation between susceptibility ranking and steering-threshold ranking across the four cardinal directions.

For each seed, we rank the four cardinal directions by their mean FF 2 susceptibility (lowest to highest) and by their mean activation steering threshold (highest to lowest) and compute the per-seed Spearman rank correlation $\rho$ (using medians in place of means in either or both rankings gives extremely similar values).  In \cref{tab:rank_correlation_full} we report the mean $\rho$ across seeds.

\begin{table}[h]
\centering
\caption{Mean per-seed Spearman rank correlation $\ol \rho$ between FF 2 susceptibility ranking and activation-steering-threshold ranking across the four cardinal directions, with both quantities summarized by their mean over the states in each direction.  Conditions are labeled by the training distribution, training checkpoint, and susceptibility evaluation distribution $\alpha'$.}
\label{tab:rank_correlation_full}
\begin{tabular}{lcc}
\toprule
Condition & $n$ & $\ol \rho$ \\
\midrule
$\alpha = 0.5$, phase 3 start, $\alpha' = 0.5$ & 104 & $+0.454$ $[+0.373, +0.533]$ \\
$\alpha = 0.5$, phase 3 start, $\alpha' = 1$ & 63 & $+0.394$ $[+0.270, +0.511]$ \\
$\alpha = 0.5$, phase 3 end, $\alpha' = 0.5$ & 104 & $+0.490$ $[+0.400, +0.579]$ \\
$\alpha = 0.5$, phase 3 end, $\alpha' = 1$ & 104 & $+0.158$ $[+0.050, +0.265]$ \\
$\alpha = 1$, phase 3 start, $\alpha' = 1$ & 100 & $+0.406$ $[+0.302, +0.502]$ \\
$\alpha = 1$, phase 3 end, $\alpha' = 1$ & 101 & $+0.053$ $[-0.069, +0.174]$ \\
$\alpha = 0.6$, phase 3 start, $\alpha' = 0.6$ & 67 & $+0.343$ $[+0.242, +0.439]$ \\
$\alpha = 0.6$, phase 3 start, $\alpha' = 1$ & 67 & $+0.415$ $[+0.307, +0.516]$ \\
$\alpha = 0.6$, phase 3 end, $\alpha' = 0.6$ & 67 & $+0.149$ $[+0.033, +0.263]$ \\
$\alpha = 0.6$, phase 3 end, $\alpha' = 1$ & 67 & $+0.084$ $[-0.048, +0.212]$ \\
\bottomrule
\end{tabular}
\end{table}

The mean rank correlation $\ol \rho$ is positive at the first phase 3 checkpoint across all populations and evaluation distributions, with values in the interval $[0.3, 0.45]$.  At end of training this remains the case for $\alpha = 0.5$ on-distribution but weakens substantially for $\alpha = 1$ and $\alpha = 0.6$ ($\ol \rho \le +0.15$), with 0 in the confidence interval for $\alpha'=1$, consistent with the within-phase-3 weakening of the susceptibility signal documented in \cref{within_optimal_section}.  

\subsection{Explanatory Power of Susceptibilities} \label{regression_appendix}

In the $\alpha < 1$ setting the result of \cref{activation_steering_section} that susceptibility-distinguished directions are harder to steer is consistent with an alternative explanation: that susceptibilities are just rediscovering that $\{\mr{Right}, \mr{Down}\}$ states are systematically harder to steer than $\{\mr{Left}, \mr{Up}\}$ states regardless of whether they are distinguished by susceptibilities (\cref{activation_steering_section}).  To separate these readings we fit linear regressions predicting per-direction steering resistance from both the fixed $\{\mr{Right}, \mr{Down}\}$ indicator and the susceptibility predictors, and test whether the susceptibility predictors contribute information beyond that encoded in the fixed partition.

We set up our analysis as follows.  We will define three predictors as a function of the seed $i$ and cardinal direction $D_c \in \{\mr{Right}, \mr{Down}, \mr{Left}, \mr{Up}\}$.  We will then fit the minimum activation steering strengths $s_{i,D_c}$ for the FF2 layer to affine linear combinations of these three predictors or subsets thereof.  The linear predictors we will use are 
\begin{enumerate}
    \item The indicator $f_{D_c}$ defined by $f_{D_c} = 1$ if $D_c \in \{\mr{Right}, \mr{Down}\}$ and $f_{D_c} = 0$ if $D_c \in \{\mr{Left}, \mr{Up}\}$.
    \item The Z-scored mean FF 2 susceptibility over states in direction class $D_c$ computed on distribution, denoted $\ol \chi^{\alpha}_{i,D_c}$.
    \item The Z-scored mean FF 2 susceptibility over states in direction class $D_c$ computed off distribution with $\alpha'=1$, denoted $\ol \chi^{1}_{i,D_c}$.
\end{enumerate}

We fit the affine linear model
\[s_{i,D_c} = a_1 f_{D_c} + a_2 \ol \chi^{\alpha}_{i,D_c} + a_3 \ol \chi^{1}_{i,D_c} + b\]
 by ordinary least squares.  We also fit restricted variants for which one or more of the $a_j$ coefficients are constrained to equal zero.  For each variant we report two statistics, both reflecting how much of the variance in $s$ the predictors explain after charging the model for its complexity: the \emph{adjusted coefficient of determination} $R^2_{\mr{adj}} = 1 - (1 - R^2)\tfrac{n - 1}{n - k - 1}$ and the \emph{Bayesian information criterion} $\mr{BIC} = n \log(\mr{SSR}/n) + (k+1) \log n$, where $\mr{SSR}$ is the residual sum of squares and $k$ is the number of varying predictors.

In \cref{tab:regression_alpha05} we report these two statistics for each restricted model separately at the start and end of phase 3, for a population of 64 models with $\alpha = 0.5$.

\begin{table}[h]
\centering
\caption{Comparison of regression model selection statistics for $\alpha = 0.5$ models.  Here $k$ is the number of varying predictors.  The model selected by the Bayesian information criterion is indicated in bold.}
\resizebox{\textwidth}{!}{%
\begin{tabular}{l rll c rll}
\toprule
Predictors & \multicolumn{3}{c}{Phase 3 start} & & \multicolumn{3}{c}{Phase 3 end} \\
\cmidrule(lr){2-4}\cmidrule(lr){6-8}
 & $k$ & $R^2_{\mr{adj}}$\  & BIC\  & & $k$ & $R^2_{\mr{adj}}$\  & BIC\  \\
\midrule
$f_{D_c}$                                        & $1$ & $0.250\ [0.158, 0.355]$ & $-63.7\ [-102.2, -34.0]$ & & $1$ & $0.420\ [0.308, 0.541]$ & $-129.3\ [-189.1, -84.0]$ \\
$\ol \chi^{\alpha}$                              & $1$ & $0.252\ [0.160, 0.354]$ & $-64.1\ [-101.7, -34.7]$ & & $1$ & $0.366\ [0.253, 0.491]$ & $-106.5\ [-162.9, -64.7]$ \\
$\ol \chi^{1}$                                   & $1$ & $0.181\ [0.089, 0.297]$ & $-40.9\ [-80.1, -13.7]$ & & $1$ & $0.040\ [0.002, 0.116]$ & $-0.5\ [-21.4, 9.6]$ \\
$f_{D_c}, \ol \chi^{\alpha}$                     & $2$ & $0.255\ [0.165, 0.363]$ & $-60.9\ [-100.8, -31.6]$ & & $2$ & $0.437\ [0.330, 0.557]$ & $-132.3\ [-193.8, -88.1]$ \\
$f_{D_c}, \ol \chi^{1}$                          & $2$ & $0.280\ [0.189, 0.390]$ & $\mathbf{-69.5\ [-112.1, -39.0]}$ & & $2$ & $0.440\ [0.337, 0.557]$ & $\mathbf{-133.8\ [-193.8, -90.5]}$ \\
$\ol \chi^{\alpha}, \ol \chi^{1}$                & $2$ & $0.271\ [0.181, 0.380]$ & $-66.4\ [-107.8, -36.6]$ & & $2$ & $0.384\ [0.277, 0.509]$ & $-109.5\ [-167.4, -68.2]$ \\
$f_{D_c}, \ol \chi^{\alpha}, \ol \chi^{1}$       & $3$ & $0.278\ [0.189, 0.391]$ & $-64.3\ [-108.0, -34.6]$ & & $3$ & $0.450\ [0.352, 0.568]$ & $-133.8\ [-195.5, -91.9]$ \\
\bottomrule
\end{tabular}
}
\label{tab:regression_alpha05}

\end{table}

When estimated at the start of phase 3 the model selected by the BIC is parameterized by $f_{D_c}$ and $\ol \chi^1$ ($\mr{BIC} = -69.5$), beating $f_D$ alone ($-63.7$) by a substantial margin.  That is to say, the off-distribution susceptibility $\ol \chi^1$ contributes information beyond the fixed partition (a nested $F$-test for the addition produces $F = 11.5$, $p = 8.1 \times 10^{-4}$ to reject the null hypothesis that $\ol \chi^1$ does not contribute to the best linear two-parameter fit).  The on-distribution susceptibility $\ol \chi^{0.5}$ however does not add any additional explanatory power (the corresponding $F$-test fails to reject the null hypothesis at any non-zero power).  When steering at the end of phase 3, we find that adding $\ol \chi^{0.5}$ and $\ol \chi^{1}$ to $f_{D_c}$ produce similar BIC values, though the difference to $f_{D_c}$ alone is smaller than at the start of phase 3, with $F=8.57$ ($p=3.7\times10^{-3}$) and $F=10.20$ ($p=1.6\times10^{-3}$) for adding $\ol \chi^{0.5}$ and $\ol \chi^{1}$, respectively.   %When steering at the end of phase 3 instead the BIC still favors alone: neither susceptibility contributes significantly.  The off-distribution susceptibility carries information beyond the fixed partition early in phase 3 but not at end of training, when the fixed partition subsumes the susceptibility predictors entirely.

The same regression framework applied to a matched $52$-seed subset of the $\alpha = 0.6$ population gives a qualitatively similar conclusion. In this case the optimal model at the start of phase 3 according to the BIC is given by $\ol \chi^1$ alone ($\mr{BIC} = -39.2$).  

\subsection{Phase Transition History: Per-Direction Statistics} \label{phase2_residue_appendix}

Recall that in \cref{phase2_residue_section} we tested the Down-Right discrepancy 
\[\Delta_i = \ol \chi_{\mr{DR},i} - \tfrac{1}{3}(\ol \chi_{\mr{DL},i} + \ol \chi_{\mr{UL},i} + \ol \chi_{\mr{UR},i})\]
and its dependence on the timing of the phase transitions into phase 2 and into phase 3.  In this section we confirm that the significance of this distinction is specific to asymmetry between the Down-Right direction and the other three diagonal directions, and not similarly variable for the other three diagonals not singled out by asymmetry in the initial state distribution.

For each diagonal direction $D_d$ define a corresponding discrepancy
\[\Delta^{D_d}_i = \ol \chi_{D_d, i} - \tfrac{1}{3} \sum_{D_d' \ne D_d} \ol \chi_{D_d', i}.\] 
Of the 104 $\alpha = 0.5$ seeds we classify (\cref{methods_conventions_section}), 25 skipped phase 2 and 79 did not.  In \cref{tab:phase2_perdiag} we report the per-direction $t$ and Spearman statistics of \cref{phase2_residue_section} for each diagonal direction $D_d$.

\begin{table}[h]
\centering
\caption{Per-direction Welch's two-sample $t$-tests (comparing populations that visited and skipped phase 2) and Spearman rank correlations (phase 2 duration vs $\Delta^{D_d}_i$), computed at the end of phase 3 for the $\alpha = 0.5$ population.  The $t$-tests are taken with numbers of visited and skipped models $n_v = 79$, $n_s = 25$.  Spearman tests at $\alpha' = 0.5$ use $n = 104$. }
\label{tab:phase2_perdiag}
\begin{tabular}{l rr c rr c rr}
\toprule
Direction & \multicolumn{2}{c}{Welch $\alpha' = 0.5$} & & \multicolumn{2}{c}{Welch $\alpha' = 1$} & & \multicolumn{2}{c}{Spearman $\alpha' = 0.5$} \\
\cmidrule(lr){2-3}\cmidrule(lr){5-6}\cmidrule(lr){8-9}
 & $t$ & $p$ & & $t$ & $p$ & & $\rho$ & $p$ \\
\midrule
Down-Right & $-3.92$ & $1.6 \times 10^{-4}$ & & $-4.44$ & $2.4 \times 10^{-5}$ & & $-0.367$ & $1.3 \times 10^{-4}$ \\
Down-Left & $+4.09$ & $8.9 \times 10^{-5}$ & & $+2.33$ & $0.022$ & & $+0.360$ & $1.7 \times 10^{-4}$ \\
Up-Left & $+3.28$ & $1.5 \times 10^{-3}$ & & $+2.35$ & $0.021$ & & $+0.392$ & $3.9 \times 10^{-5}$ \\
Up-Right & $+3.53$ & $6.2 \times 10^{-4}$ & & $+2.21$ & $0.030$ & & $+0.352$ & $2.5 \times 10^{-4}$ \\
\bottomrule
\end{tabular}
\end{table}
In all three tests (Spearman and on- and off-distribution $t$-tests) the Down-Right statistic is singled out by its sign.  It is not surprising that all tests are significant since the $\Delta^{D_d}$ statistics are linearly dependent by construction.  The comparable statistic sizes, however, indicate that the asymmetry is distributed approximately uniformly across the non-Down-Right diagonals rather than being concentrated in any one individually.

\FloatBarrier

\section{Susceptibility and Refined LLC Plots} \label{extra_plots_appendix}

In this appendix we present a more comprehensive sampling of susceptibility scatter plots and weight-refined LLC curves across models trained with $\alpha = 0.6$ (see \cref{fig:alpha_0.6_all_susceptibilities}) and with $\alpha = 1$ (see \cref{fig:alpha_1.0_all_susceptibilities}).
\begin{figure}
  \centering
  \includegraphics[width=\textwidth]{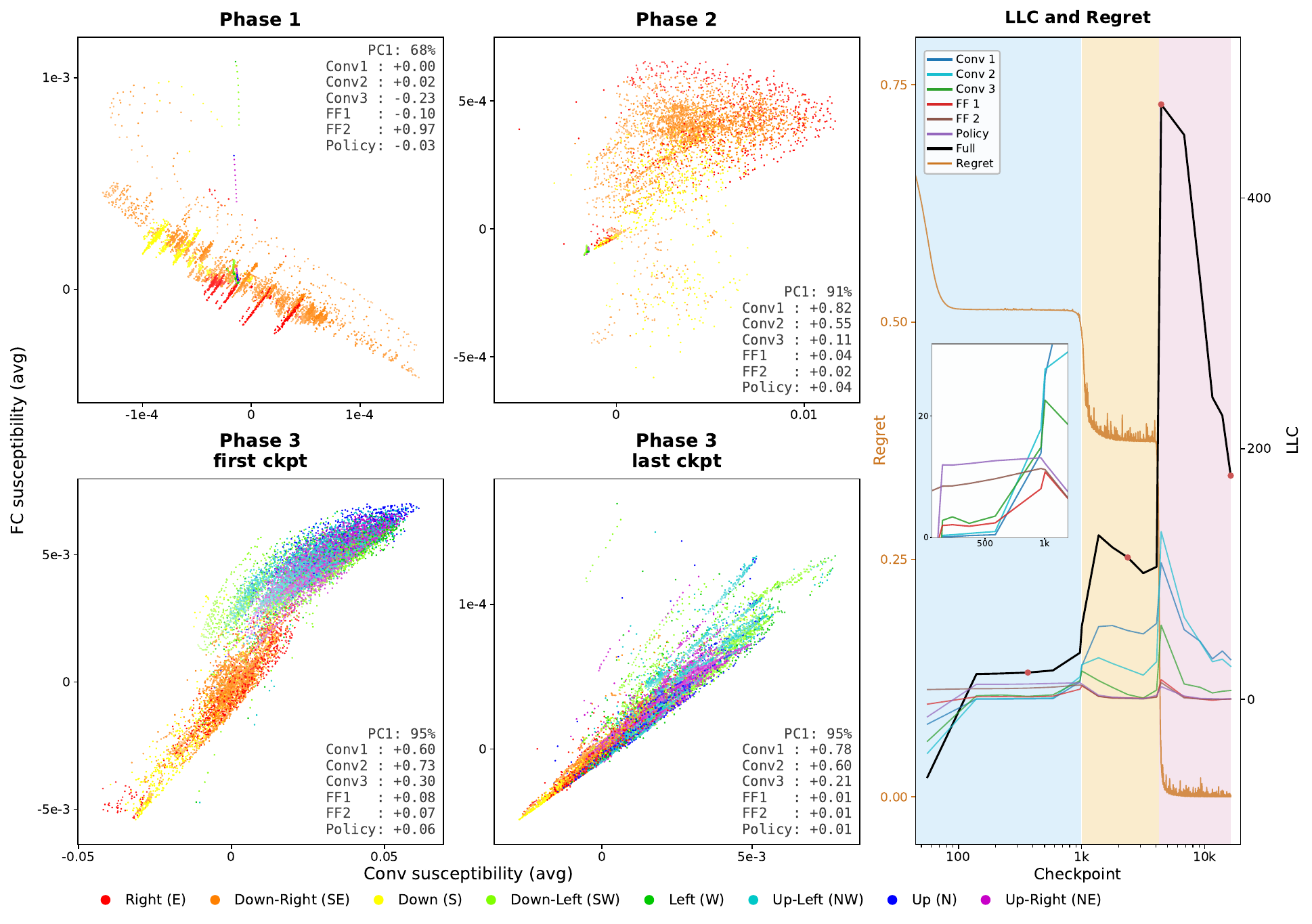}\\[1ex]
  \includegraphics[width=\textwidth]{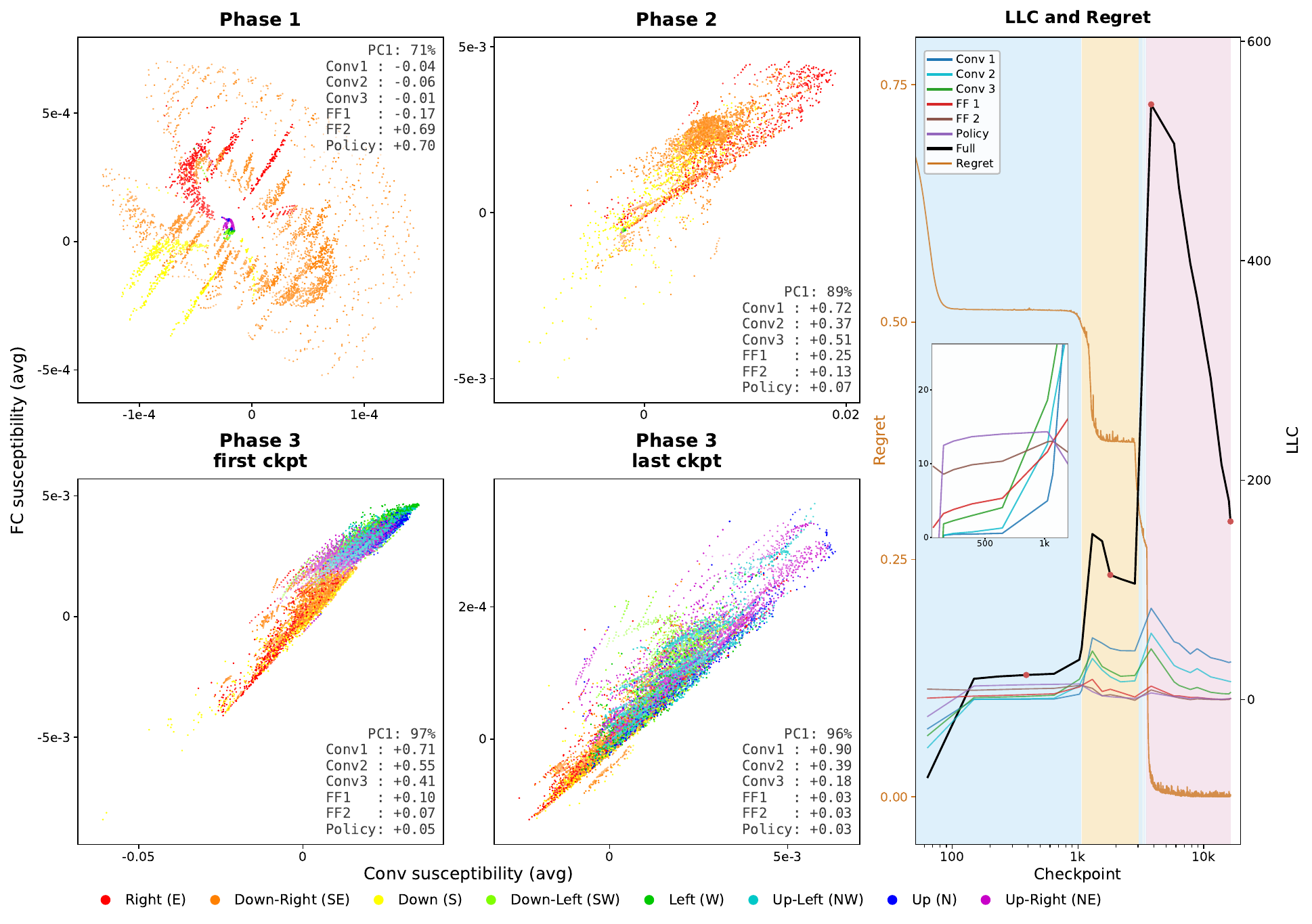}
\caption{Susceptibilities and LLC and regret curves for models trained with $\alpha=0.6$. The red dots in the LLC curves indicate the checkpoints included in the panels to the left of it. Each row is different checkpoints within the same run, different rows are different runs. The variation explained by PC1 and the cosine similarity between the PC1 and each of the weight restriction directions is indicated at the upper left corner of each susceptibility panel. All these runs go through 3 distinct phases, and there is significant susceptibility similarity across different runs.}
    \label{fig:alpha_0.6_all_susceptibilities}
\end{figure}

\begin{figure}\ContinuedFloat
  \centering
  \includegraphics[width=\textwidth]{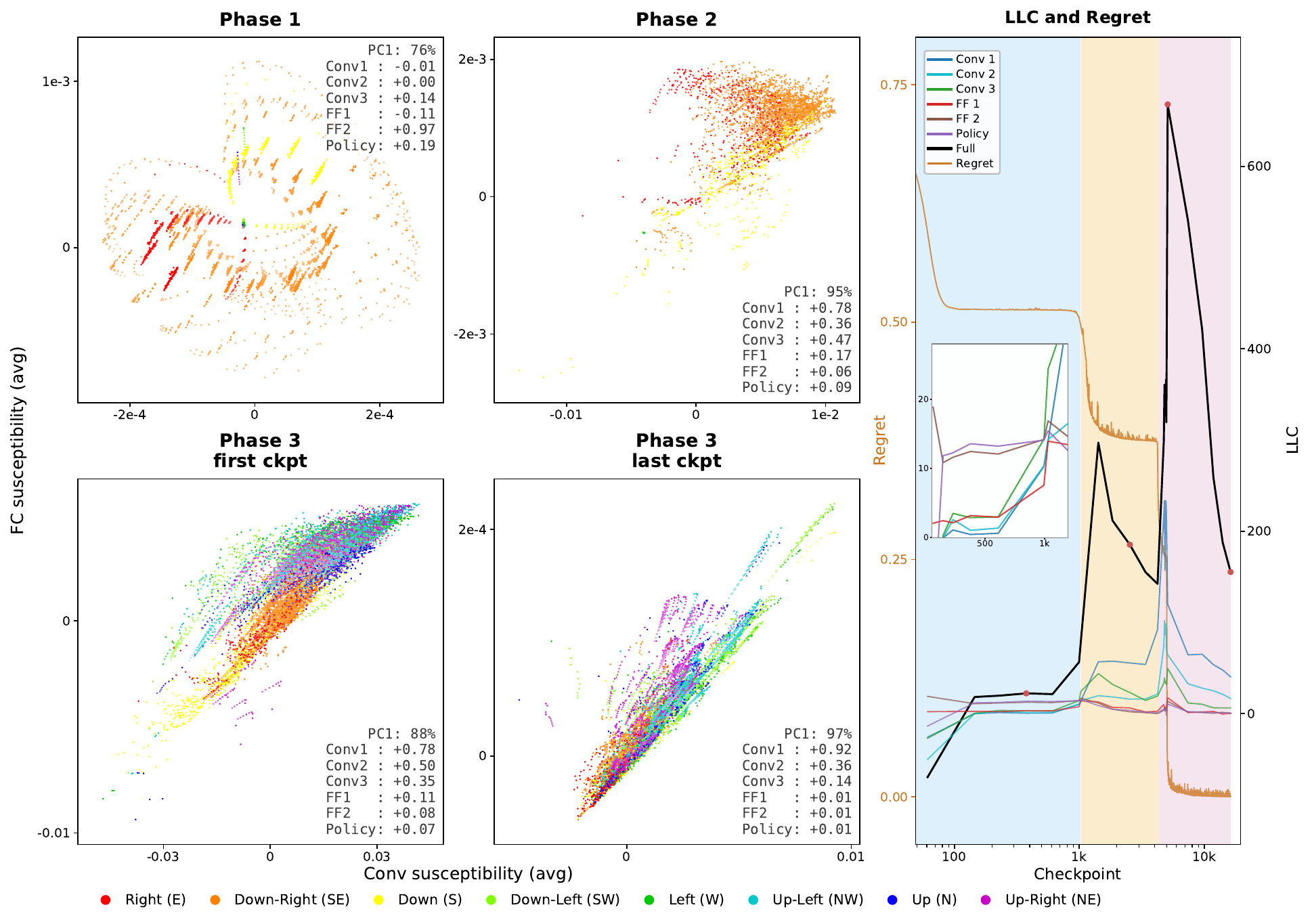}\\[1ex]
  \includegraphics[width=\textwidth]{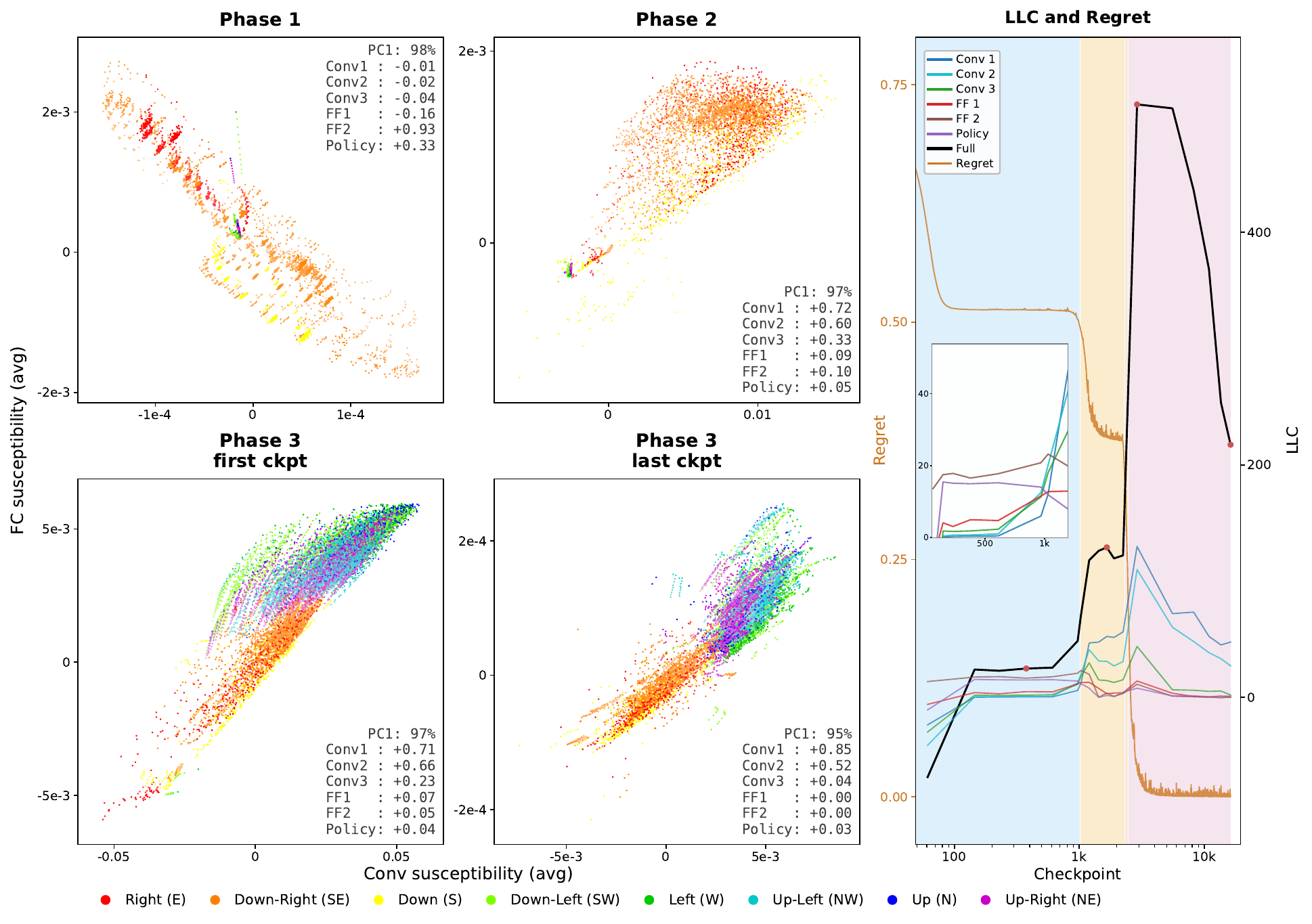}
  \caption{(continued)}
\end{figure}

\begin{figure}\ContinuedFloat
  \centering
  \includegraphics[width=\textwidth]{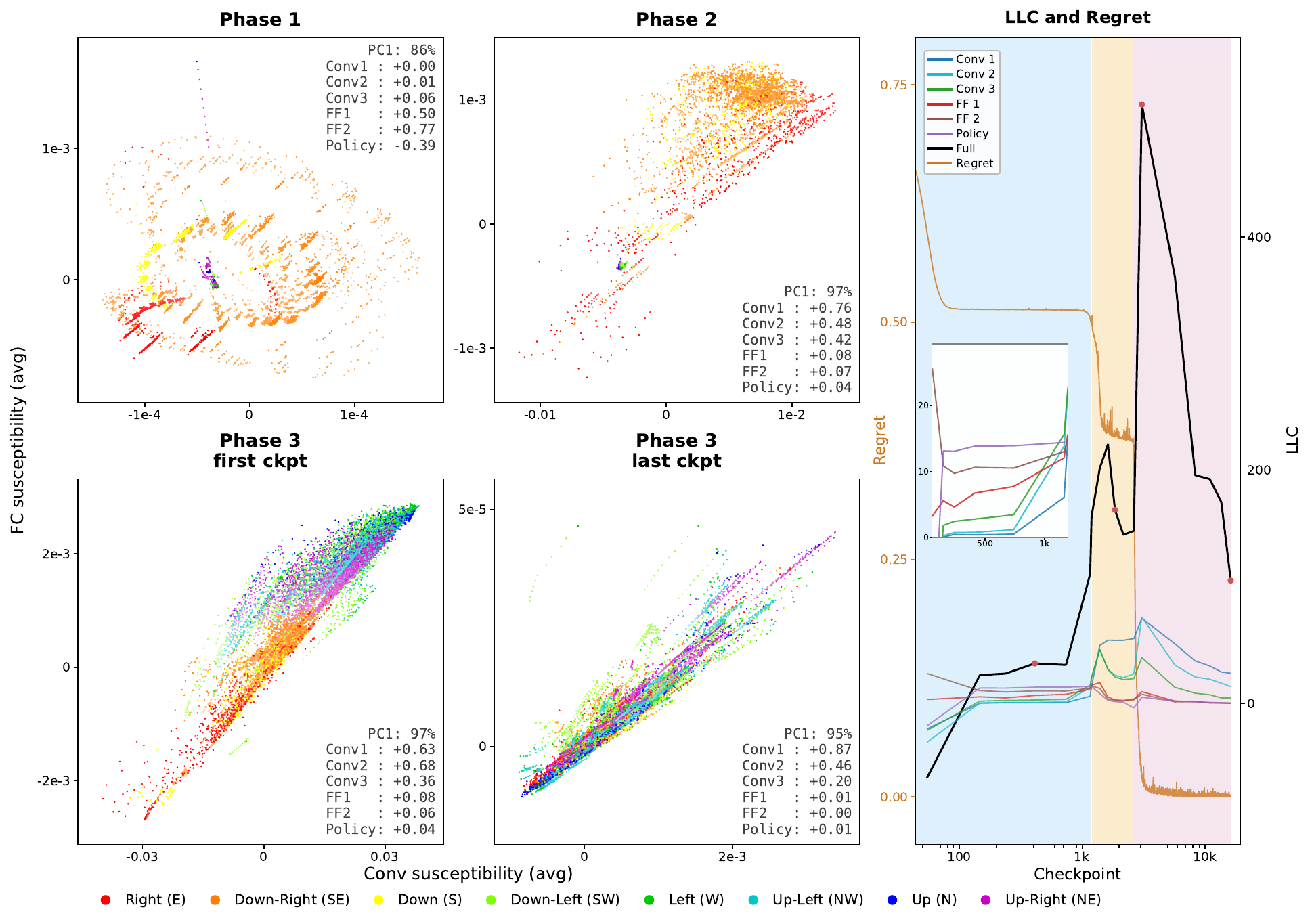}\\[1ex]
  \includegraphics[width=\textwidth]{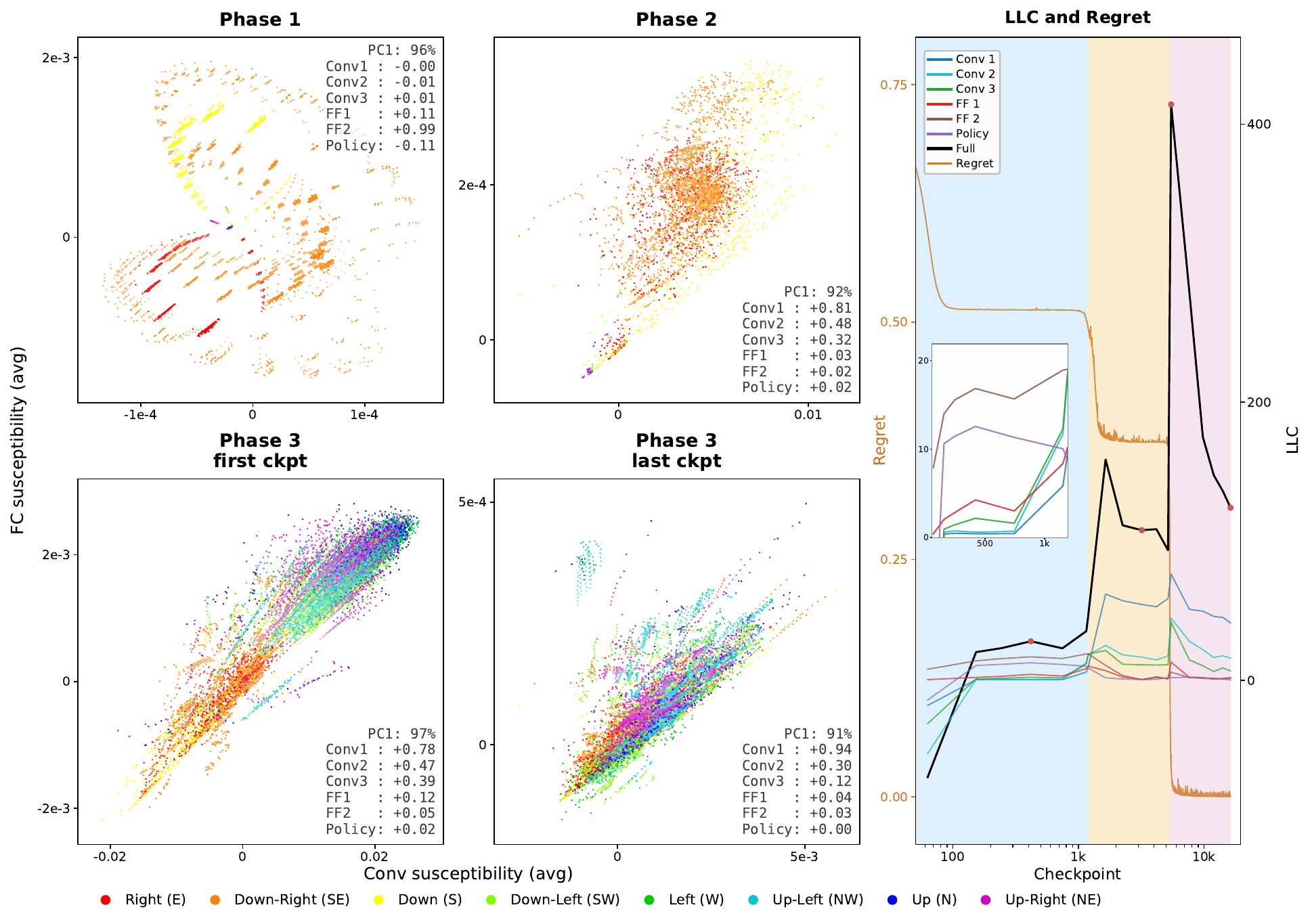}
  \caption{(continued)}
\end{figure}

\begin{figure}\ContinuedFloat
  \centering
  \includegraphics[width=\textwidth]{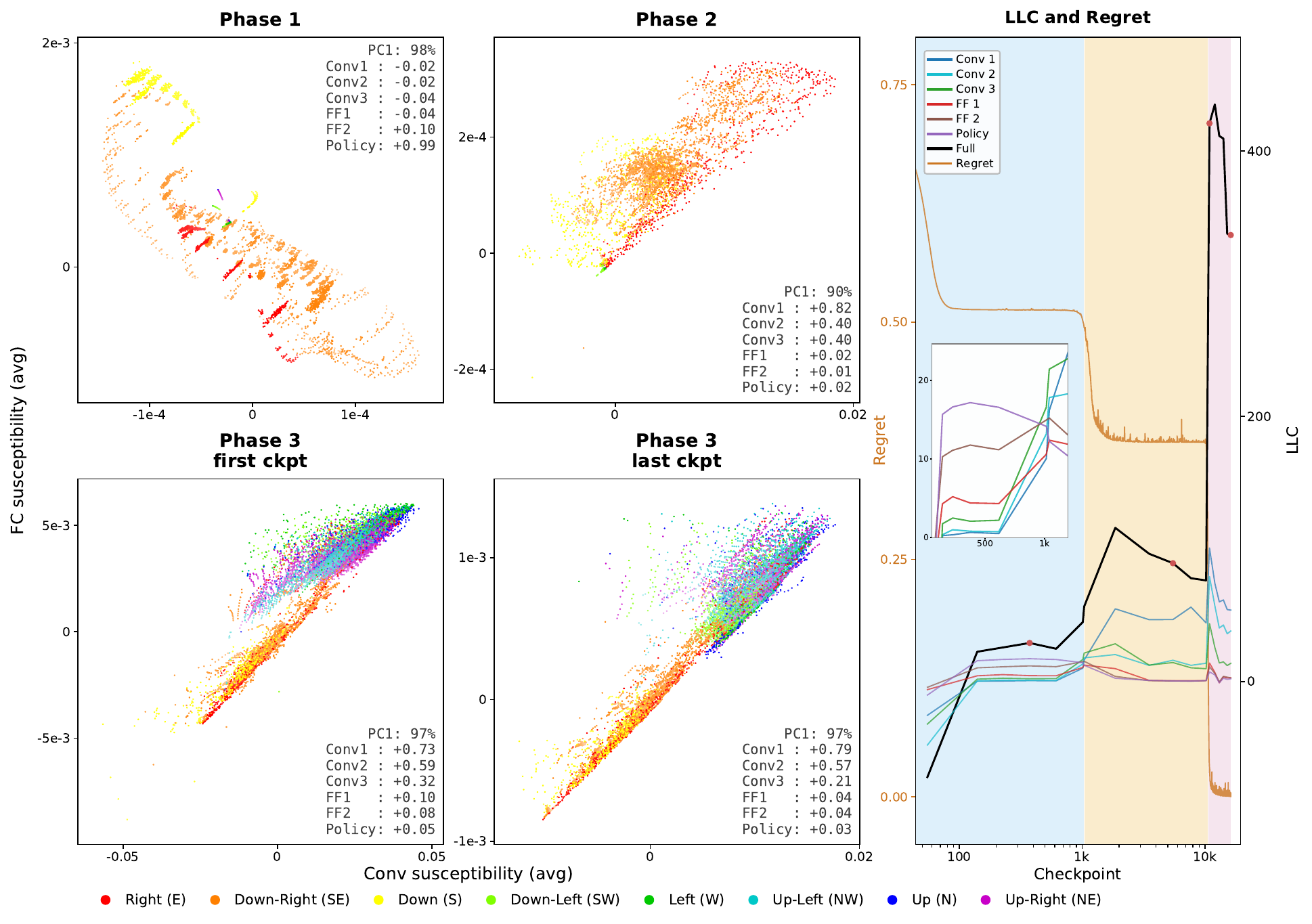}\\[1ex]
  \includegraphics[width=\textwidth]{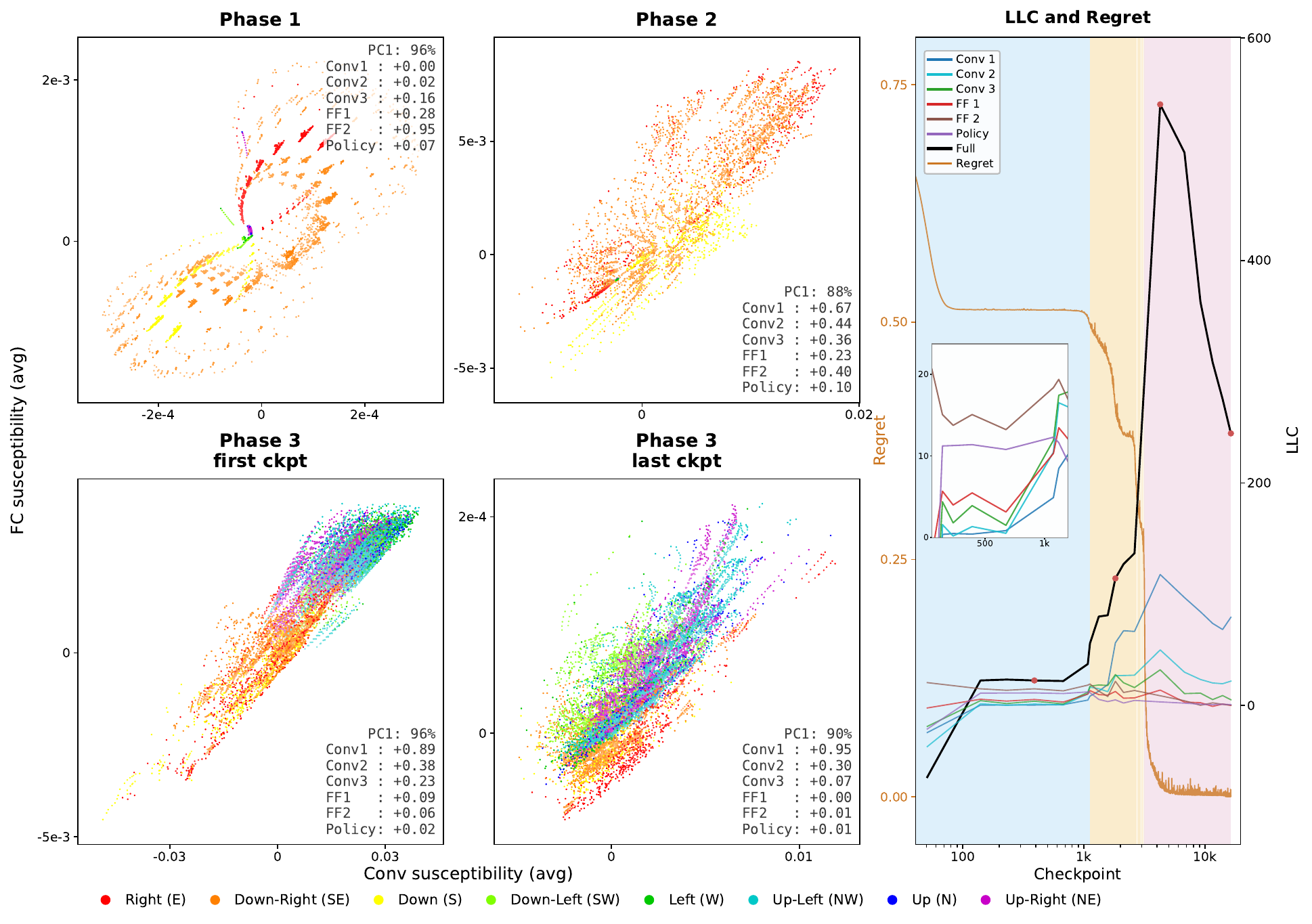}
  \caption{(continued)}
\end{figure}

\begin{figure}\ContinuedFloat
  \centering
  \includegraphics[width=\textwidth]{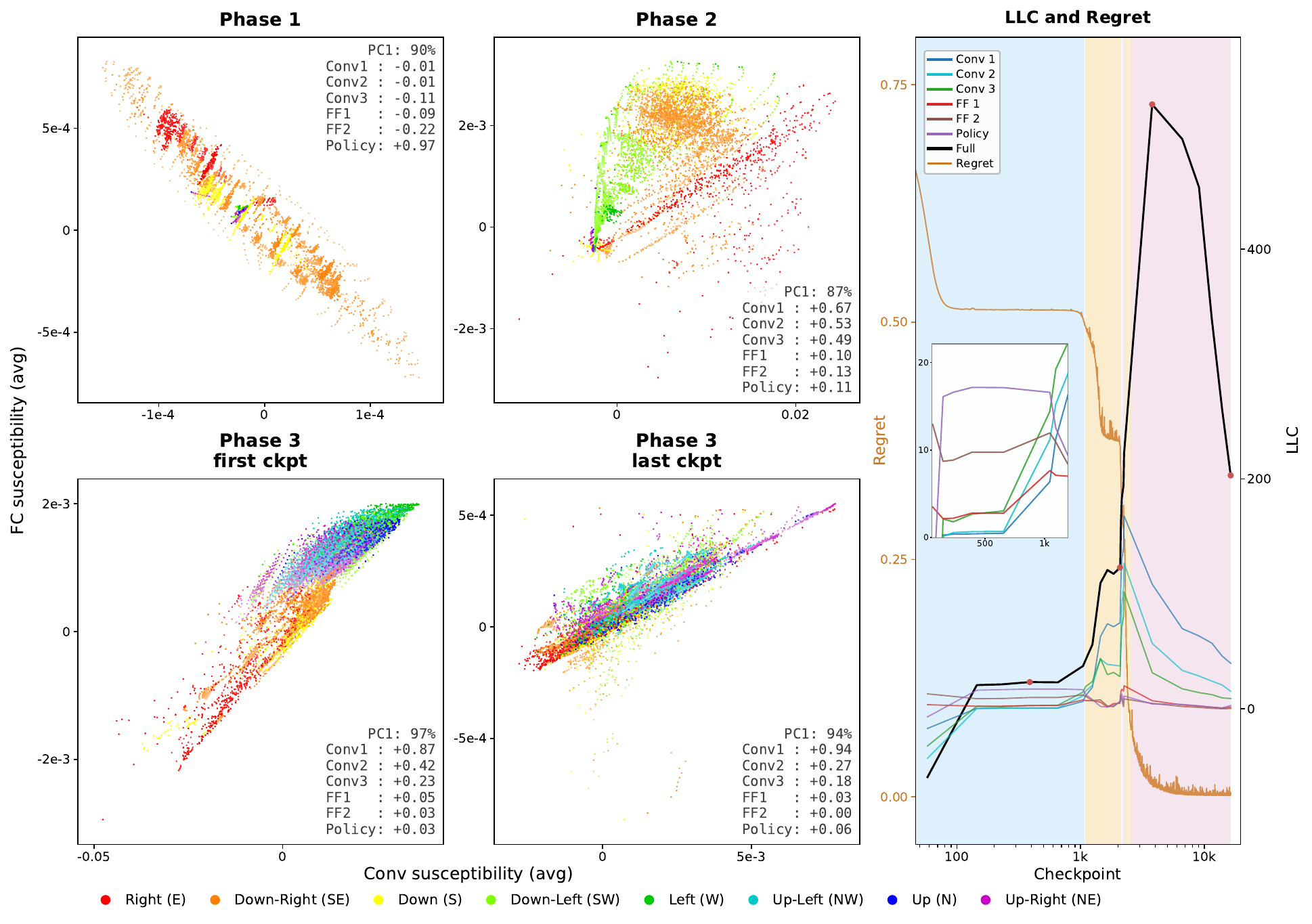}\\[1ex]
  \includegraphics[width=\textwidth]{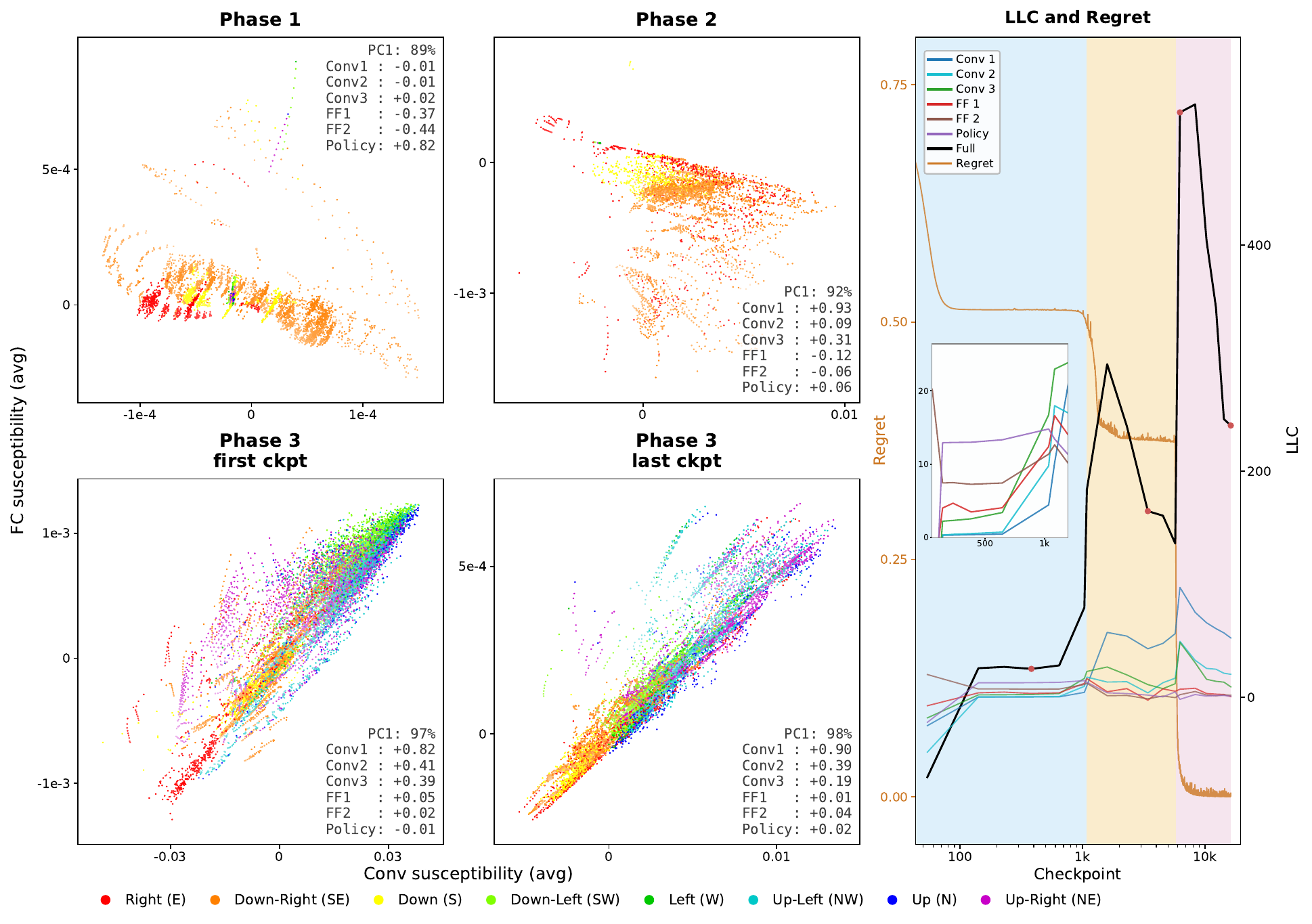}
  \caption{(continued)}
\end{figure}

% =====================================================================
% alpha = 1.0 (12 runs; id 25, 27, 29 skipped — incomplete ckpt2000 zarr)
% Three PDFs per figure environment / page. The three runs with a real
% LLC trajectory (id 15, 19, 30) are grouped onto the first page so the
% remaining pages carry the visually homogeneous "scatter only" panels.
% =====================================================================
\begin{figure}
  \centering
  \includegraphics[width=\textwidth]{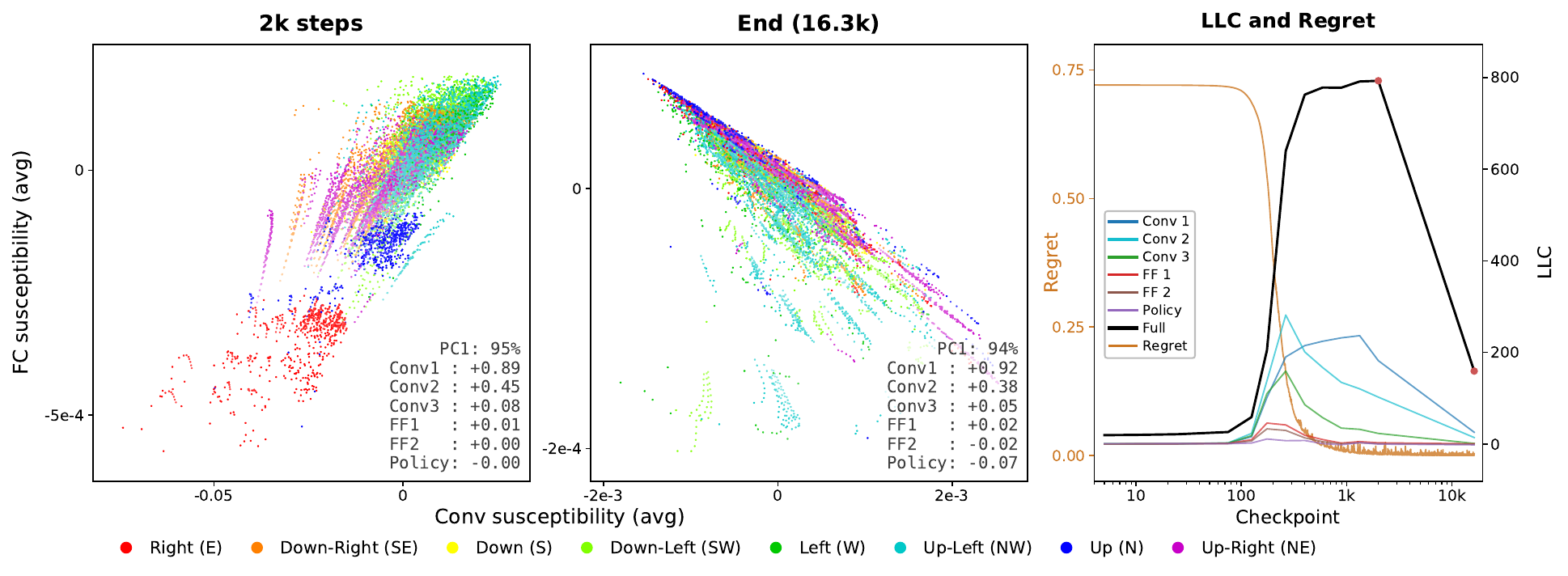}\\[1ex]
  \includegraphics[width=\textwidth]{figures/conv_vs_fc_grid_alpha1.0_id19_single.pdf}\\[1ex]
  \includegraphics[width=\textwidth]{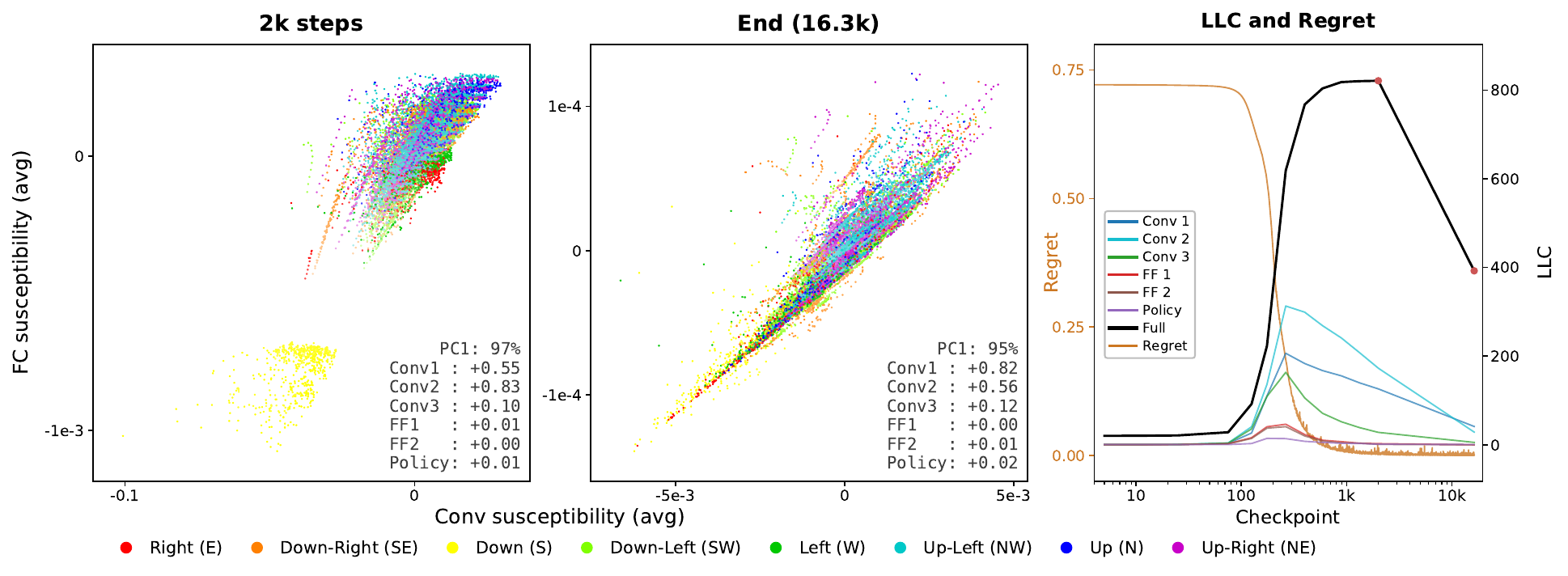}
  \caption{Susceptibilities and LLC and regret curves for models trained with $\alpha=1$. The two panels within the same row of each side of the dividing line are from the same run. The variation explained by PC1 and the cosine similarity between the PC1 and each of the weight restriction directions is indicated at the upper left corner of each susceptibility panel.}
    \label{fig:alpha_1.0_all_susceptibilities}
\end{figure}

\begin{figure}\ContinuedFloat
  \centering
  \includegraphics[width=\textwidth]{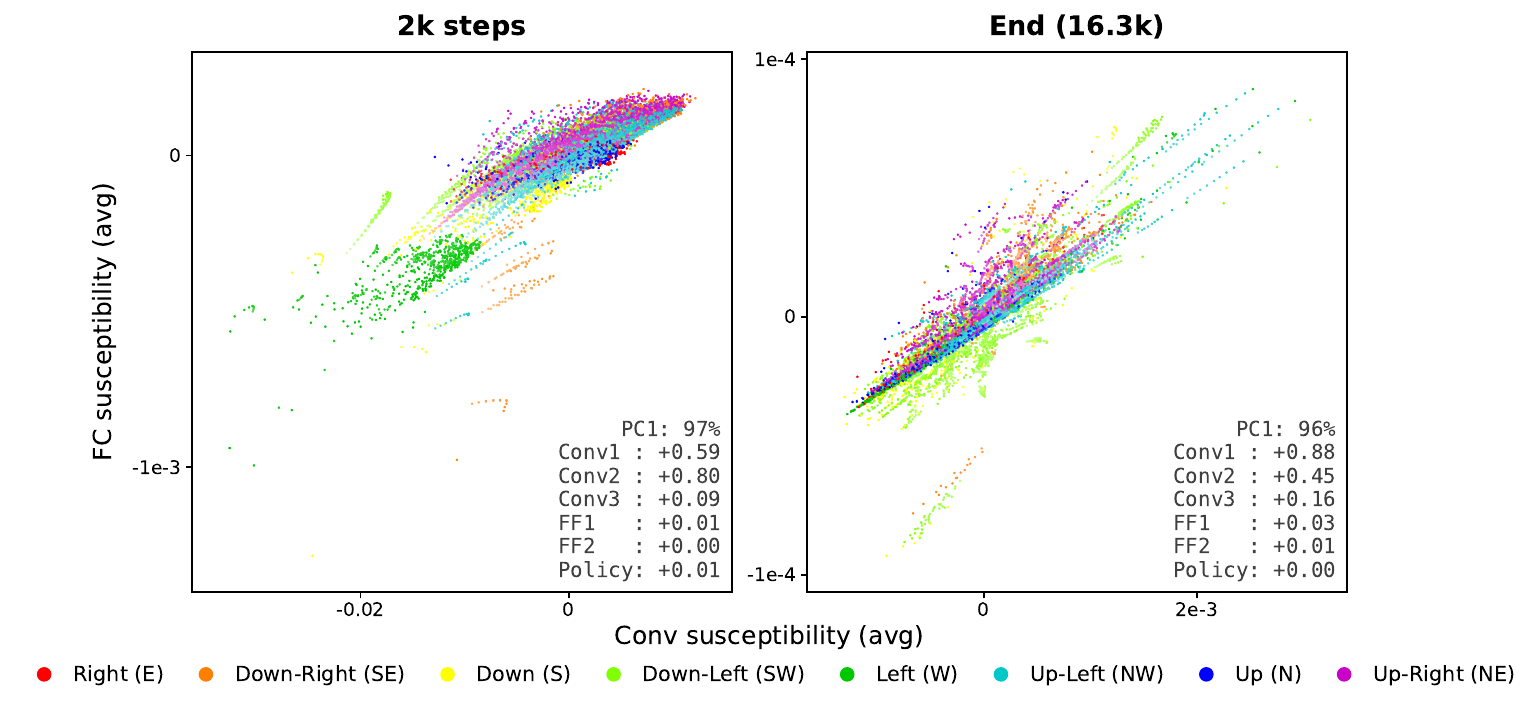}\\[1ex]
  \includegraphics[width=\textwidth]{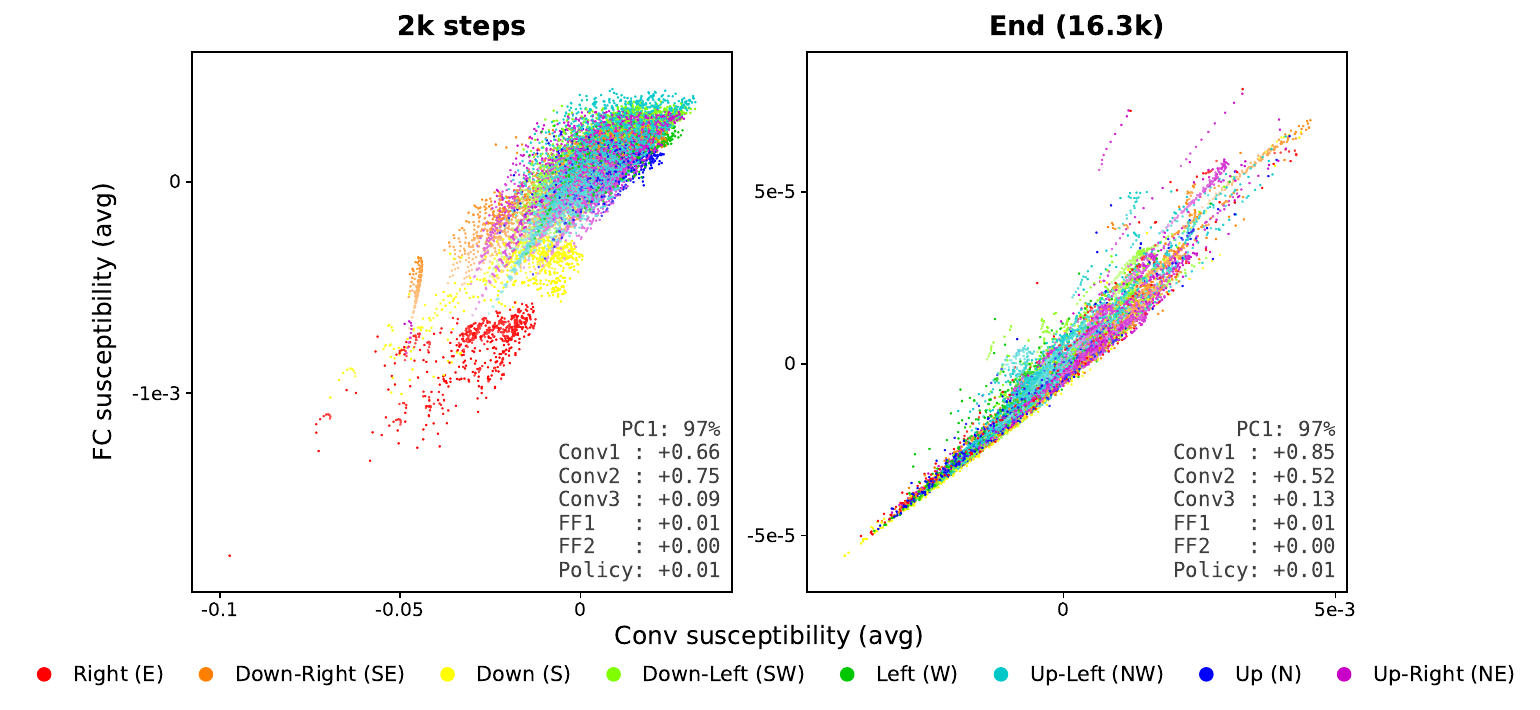}\\[1ex]
  \includegraphics[width=\textwidth]{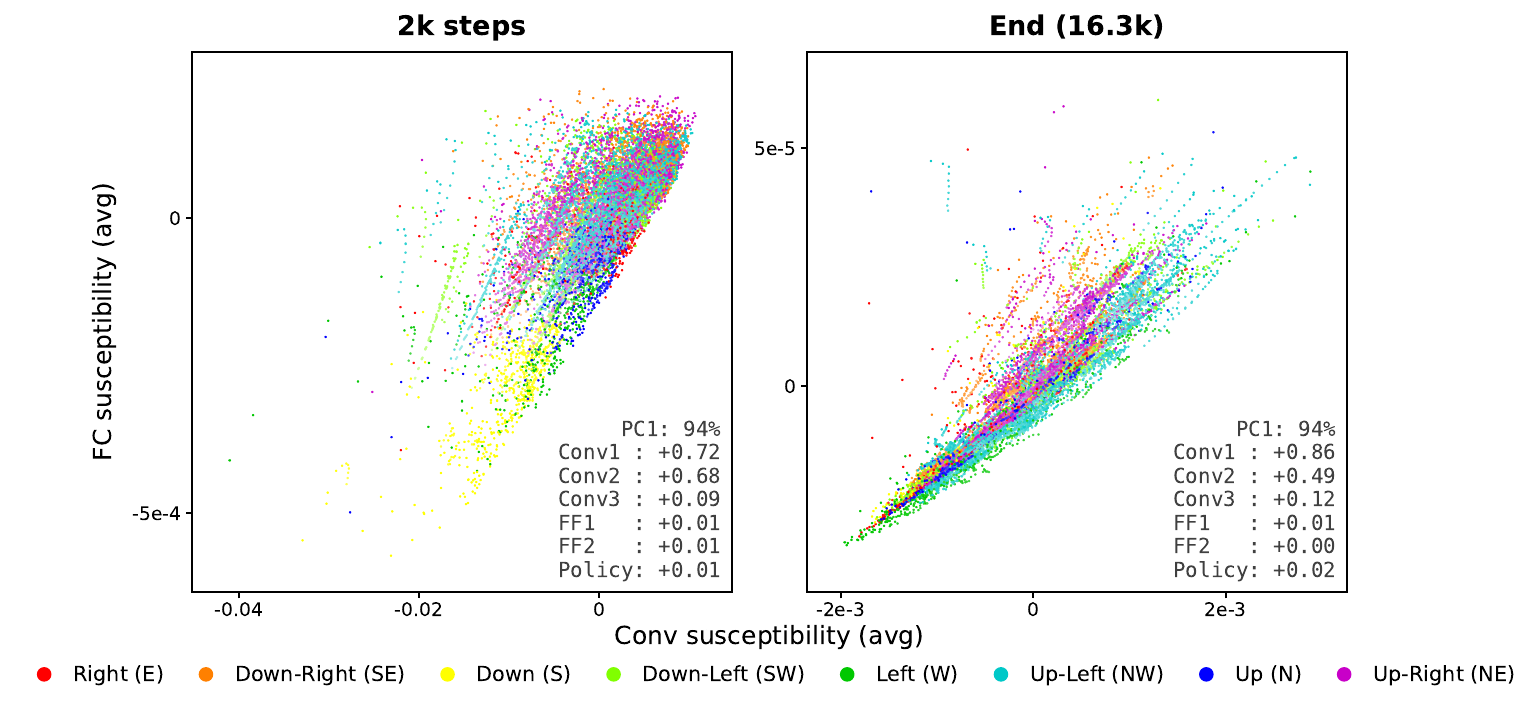}
  \caption{(continued)}
\end{figure}

\begin{figure}\ContinuedFloat
  \centering
  \includegraphics[width=\textwidth]{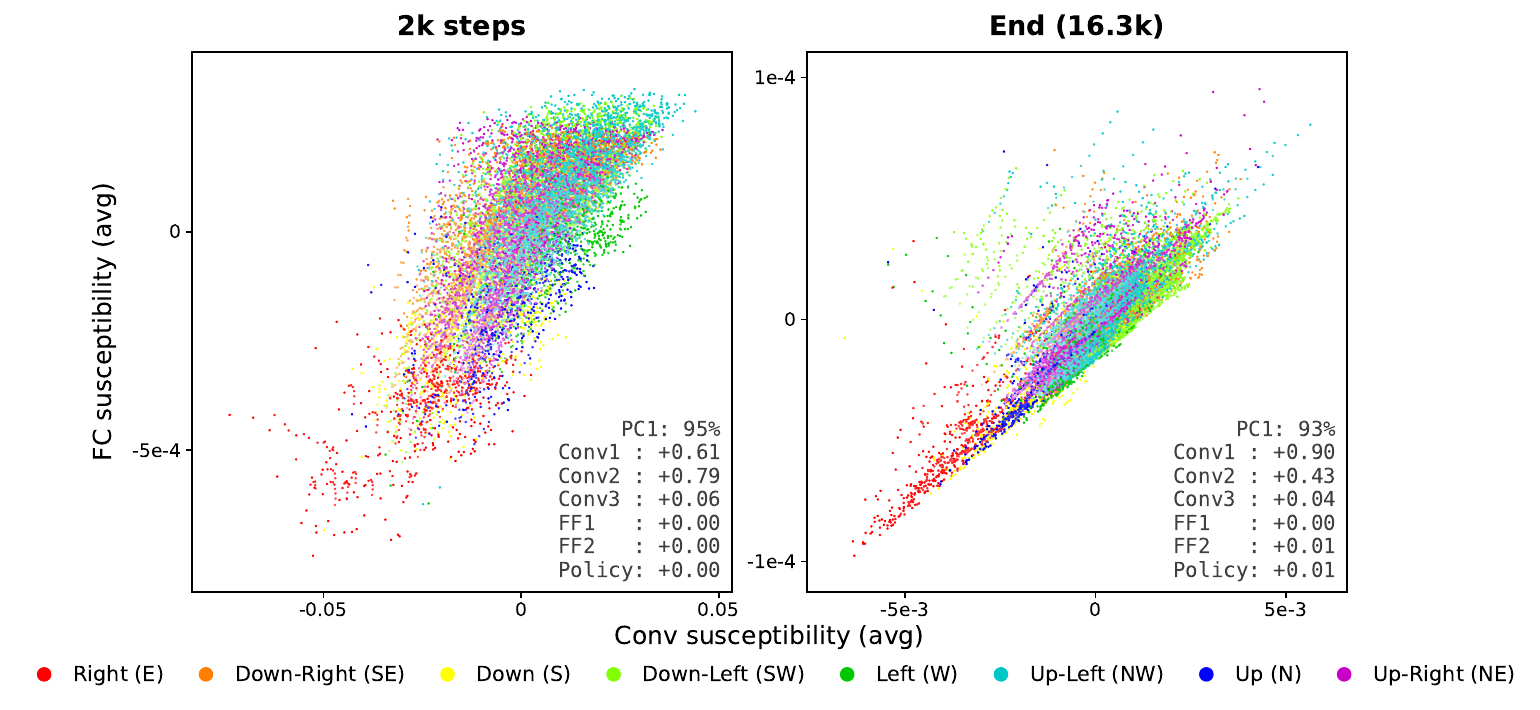}\\[1ex]
  \includegraphics[width=\textwidth]{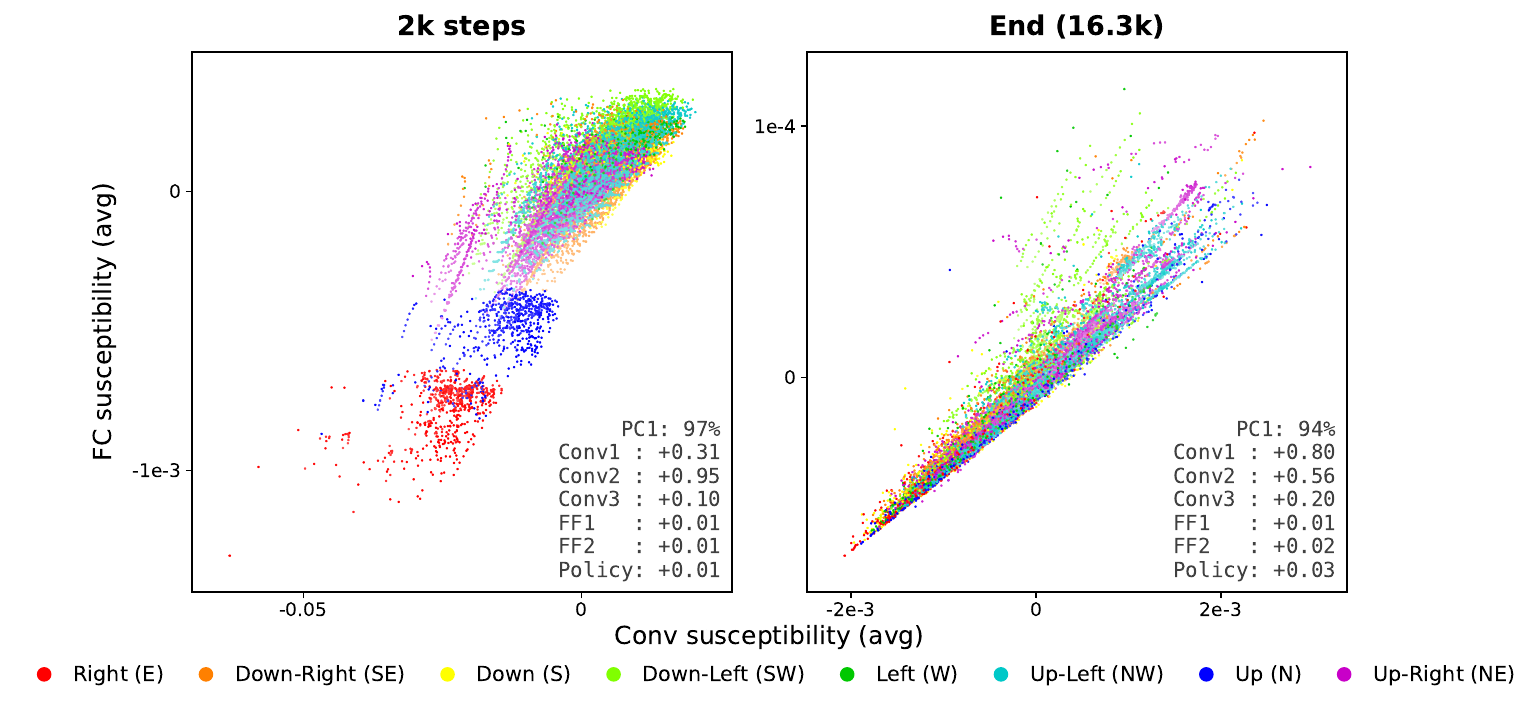}\\[1ex]
  \includegraphics[width=\textwidth]{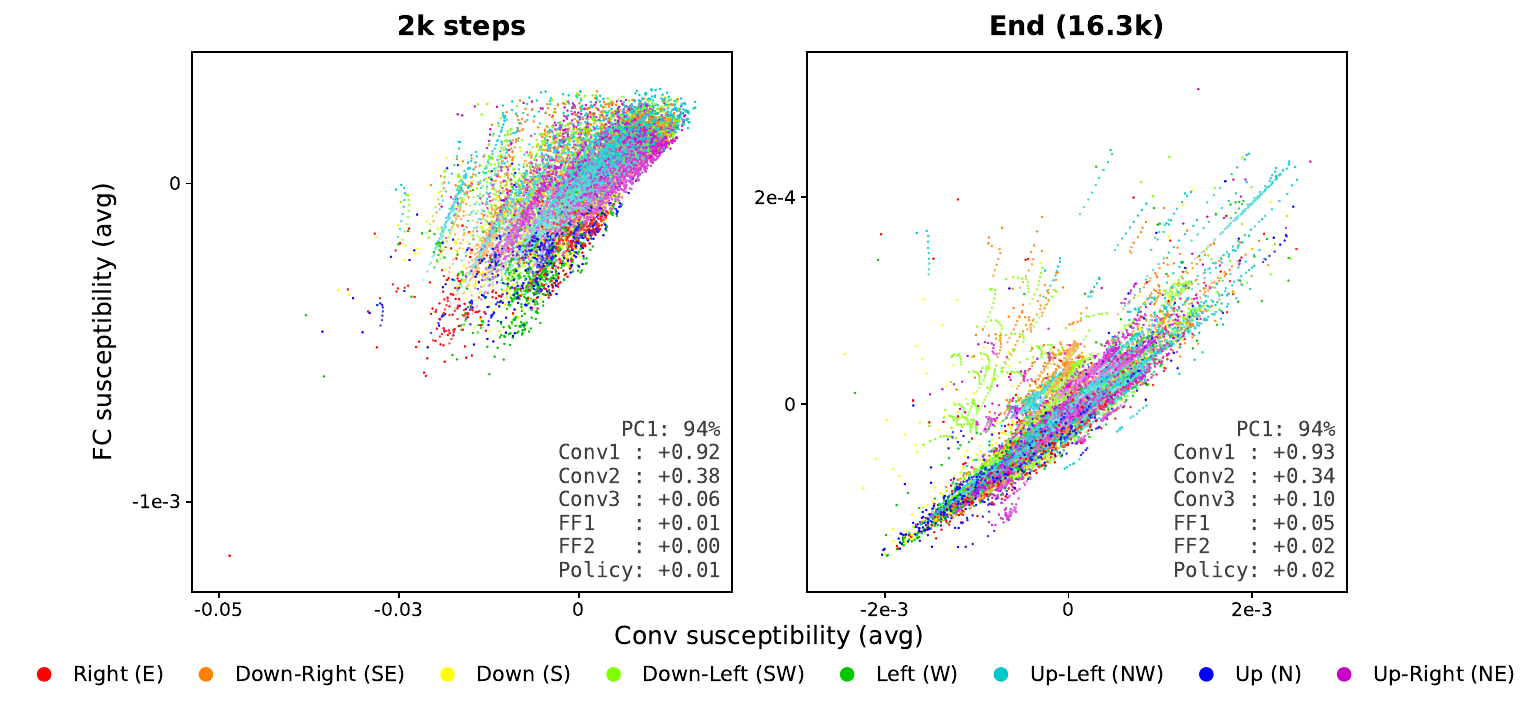}
  \caption{(continued)}
\end{figure}

\begin{figure}\ContinuedFloat
  \centering
  \includegraphics[width=\textwidth]{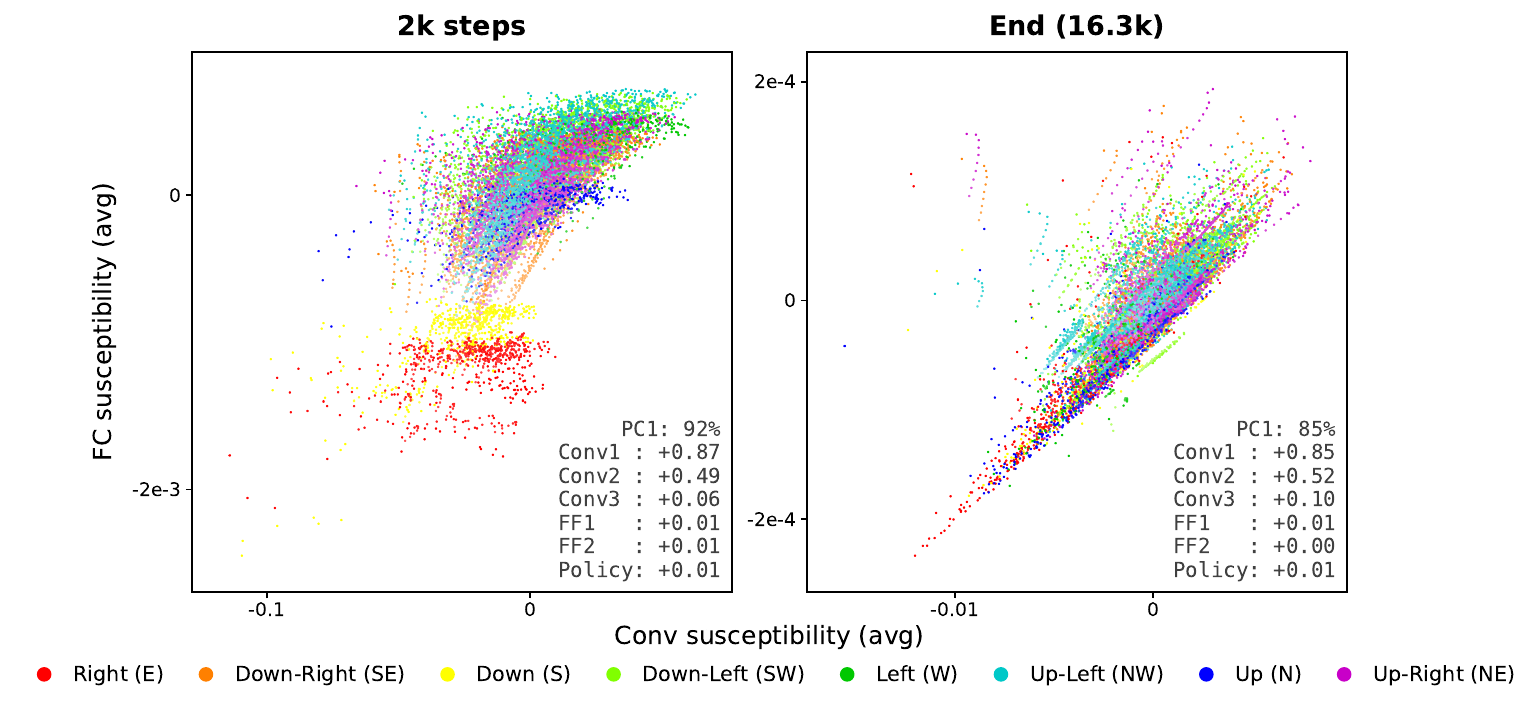}\\[1ex]
  \includegraphics[width=\textwidth]{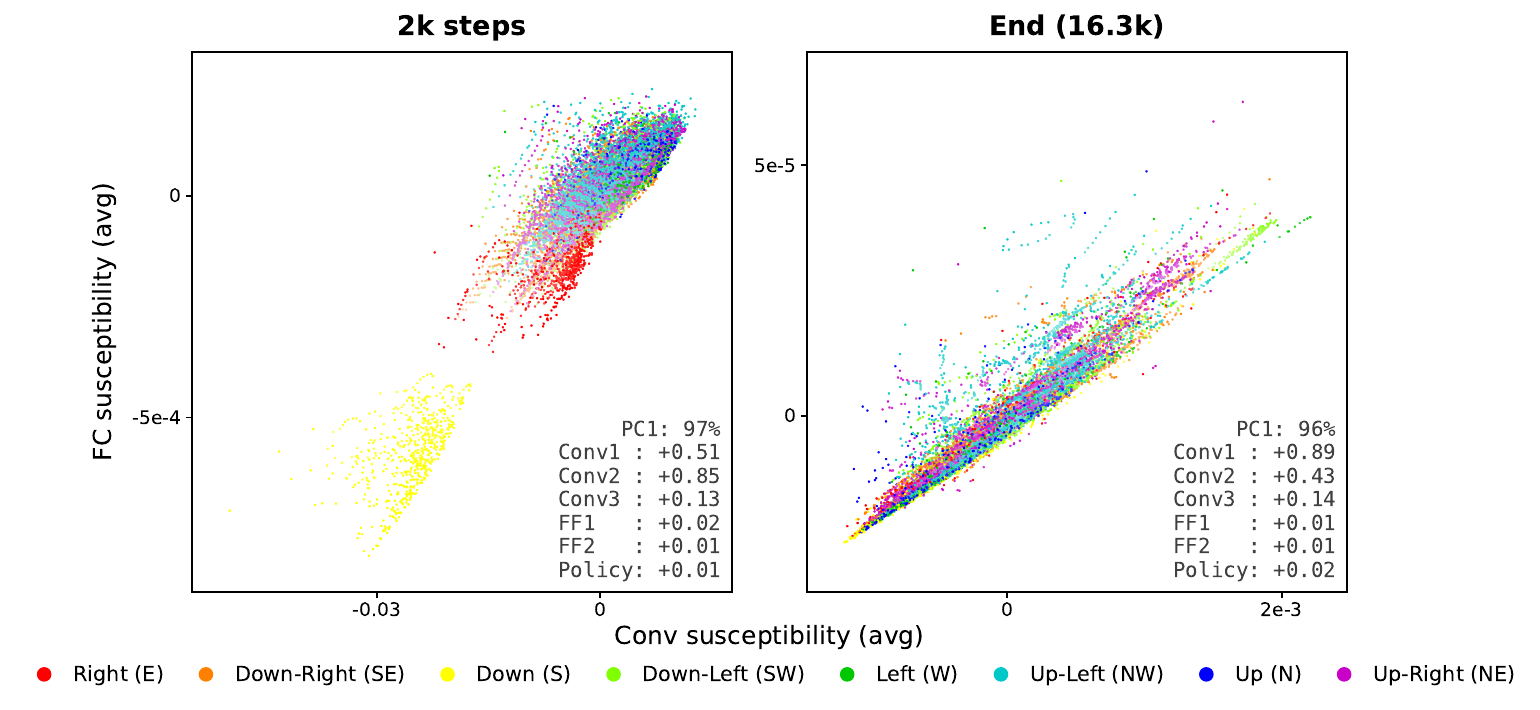}\\[1ex]
  \includegraphics[width=\textwidth]{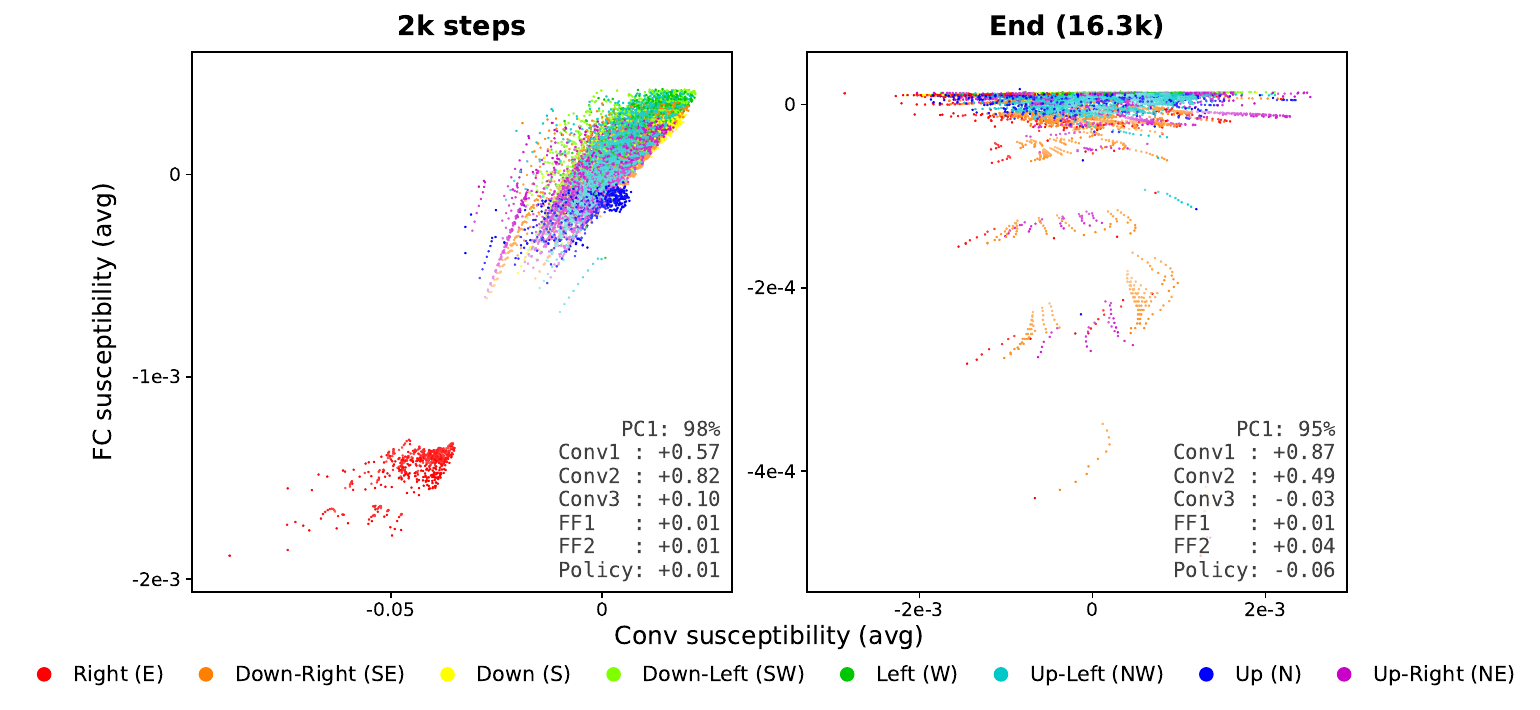}
  \caption{(continued)}
\end{figure}

\end{document}